\documentclass[11pt,journal,onecolumn,compsoc]{IEEEtran}
\ifCLASSOPTIONcompsoc
  \usepackage[nocompress]{cite}
\else
\fi
\ifCLASSINFOpdf
   \usepackage[pdftex,final]{graphicx}
   \DeclareGraphicsExtensions{.pdf,.jpeg,.png}
\else
   \usepackage[dvips,final]{graphicx}
   \DeclareGraphicsExtensions{.eps}
\fi
\usepackage{amssymb}
\usepackage[cmex10]{amsmath}
\usepackage{array}

\ifCLASSOPTIONcompsoc
  \usepackage[font=normalsize,labelfont=sf,textfont=sf]{subfig}
\else
  \usepackage[font=footnotesize]{subfig}
\fi

\usepackage[utf8]{inputenc}

\usepackage{tabularx}
\usepackage{amsfonts}
\usepackage{amsthm}
\usepackage{mathtools}
\usepackage{commath}
\usepackage{enumerate}
\usepackage[usenames]{color}
\definecolor{mygreen}{rgb}{0,0.9,0.1}
\definecolor{myblack}{rgb}{0.2,0.2,0.2}
\definecolor{myred}{rgb}{0.9,0,0.1}
\definecolor{myblue}{rgb}{0,0,0.9}

\usepackage[
 hyperindex=true,%
  bookmarksopen=true,%
  bookmarksnumbered=true,%
 pagebackref=true,%
 colorlinks=true,%
 linkcolor=myred,%
 urlcolor=myred,%
 citecolor=green,%
 anchorcolor=myblue
 ]{hyperref}

\renewcommand*{\refeq}[1]{%
  \begingroup
    \hypersetup{
      linkcolor=cyan,
    }%
    \ref{#1}%
  \endgroup
}

\newtheorem{mylem}{Lemma}[section]
\newtheorem{mythm}{Theorem}[section]

\newtheorem{myprop}{Proposition}[section]

\newtheorem{mynote}{Note}[section]

\theoremstyle{definition}
\newtheorem{mydef}{Definition}[section]
\theoremstyle{definition}
\newtheorem{myexamp}{Example}[section]
\theoremstyle{definition}
\newtheorem{myalg}{Algorithm}[section]

\newcommand{\pga}{W}
\newcommand{\pgb}{b}

\newcommand{\dd}{d}
\newcommand{\vz}{z}
\newcommand{\vx}{x}
\newcommand{\vu}{\mathbf{u}}
\newcommand{\vl}{\Lambda}

\newcommand{\vq}{q}
\newcommand{\vp}{p}
\newcommand{\mK}{M}
\newcommand{\mA}{A}
\newcommand{\mP}{P}
\newcommand{\sLmat}{\mathbb{D}^{N}_{++}}
\newcommand{\sY}{\mathcal{Y}}
\newcommand{\sX}{\mathcal{X}}
\newcommand{\sL}{\mathcal{L}}

\newcommand{\sH}{\mathcal{H}}
\newcommand{\sU}{\mathcal{U}}
\newcommand{\sV}{\mathcal{V}}
\newcommand{\extR}{\overline{\mathbb{R}}}

\newcommand{\nm}{\mathbf{n}}

\newcommand*\rot[1]{\rotatebox{90}{#1}}

\newcommand{\ssB}[1]{\mathcal{B}({#1})}
\renewcommand{\vec}[1]{\mathbf{{#1}}}

\DeclareMathOperator{\proj}{P}
\DeclareMathOperator{\tr}{Tr}
\DeclareMathOperator{\tgv}{TGV}
\DeclareMathOperator{\tv}{TV}

\DeclareMathOperator{\ri}{ri}
\DeclareMathOperator{\zer}{zer}
\DeclareMathOperator{\diag}{diag}

\DeclareMathOperator{\diverg}{div}
\DeclareMathOperator{\sym}{Sym}
\DeclareMathOperator{\aff}{aff}

\DeclareMathOperator{\cl}{cl}

\newcommand{\inner}[2]{\langle {#1},{#2} \rangle}
\newcommand{\cG}[1]{\Gamma_0({#1})}

\begin{document}

%%%%%%%%%%%%%%%%%%% GENERAL INFO PART - NEEDS EDITING !!!%%%%%%%%%%%%%%%

\title{Confidence driven TGV fusion}

\author{Valsamis Ntouskos  \qquad
\quad Fiora Pirri \\ {\small \{ntouskos,pirri\}@diag.uniroma1.it } \\[3pt] 
{\small ALCOR Vision, Perception and Learning Robotics Lab}\\
{\small Department of Computer, Control and Management Engineering}\\
{\small University of Rome ``La Sapienza''} }

\IEEEcompsoctitleabstractindextext{%
\begin{abstract}
We introduce a novel model for spatially varying variational data fusion, driven by point-wise confidence values. The proposed model allows for the joint estimation of the data and the confidence values based on the spatial coherence of the data. We discuss the main properties of the introduced model as well as suitable algorithms for estimating the solution of the corresponding biconvex minimization problem and their convergence. 
%An extension of the primal-dual hybrid gradient algorithm is proposed and we discuss its convergence. 
The performance of the proposed model is evaluated considering the problem of depth image fusion by using both synthetic and real data from publicly available datasets.
\end{abstract}
%% For details on IEEE styles look into ..OldVersion/main.tex

}

\maketitle

\IEEEdisplaynotcompsoctitleabstractindextext

\IEEEpeerreviewmaketitle

\graphicspath{{./figures/}}
\newcommand{\gwidth}{0.48\textwidth}

\section{Introduction}
Variational methods have gained a large popularity advantage over other methods when dealing with  ill-posed problems in computational vision. The reason is that they have shown good computational properties and high flexibility in large scale regularization problems, typically those arising in computational image processing applications, as for example image denoising, inpainting and super-resolution. In these contexts  the original problem is transformed into an energy minimization problem, by introducing a suitable energy functional, which favors some desired characteristics of the optimal solution. More specifically, given a domain  $\sU$, which is a Banach space, and the extended real line $\extR\coloneqq\mathbb{R}\cup\{-\infty,+\infty\}$, the energy minimization problem is driven by an energy functional $E:\sU \rightarrow \extR$ of the following general form:
\begin{equation}\label{eq:basicmodel}
E(u) = F(\mK u) + H(u;d,\lambda).
\end{equation}
Functional $H$ enforces fidelity to the given data $d$, namely the observations. On the other hand, functional $F$ acts as a regularization on a linear transformation of $u$, specified by the linear operator $\mK$, which usually represents a differential operator. 
In case both $F$ and $H$ are convex, lower-semicontinuous functions, efficient algorithms for the minimization of $E$ have been proposed, even when $F$ and possibly also $H$ are not differentiable everywhere. First-order proximal splitting algorithms are amongst the most relevant. A well-known method belonging to this class of algorithms is the primal-dual hybrid gradient method (PDHG) \cite{Chambolle2010,Condat2013,Zhu2008}.

The parameter $\lambda$ in (\refeq{eq:basicmodel})  balances the relative importance of the two terms $F$ and $H$, and it is usually   assigned {\em a-priori} and  applied uniformly on the effective domain of $E$. We consider here  $\lambda$ as a multiplicative parameter applied to the fidelity term $H$. 

When $\lambda$ is applied as a multiplicative parameter in $H$, the effect of the parameter is to act as confidence value of the data fidelity term, which is crucial when data come in a multiplicity, and varying in space.   Actually,  a spatially varying regularization parameter $\lambda$ has been examined in the past as for example  in \cite{Strong2003} for the well-known ROF model \cite{Rudin1992}. However the idea of introducing a spatial prior on the fidelity term to asses confidence on the data accuracy is new, to our knowledge.

Given this background, the main contribution of this work is  a new model, which extends  (\refeq{eq:basicmodel})   to govern the fusion of multiple data observations, with occurring spatial overlaps.  The fusion problem amounts to integrating redundant and complementary information from several data sources, each bringing different degree of accuracy, which can highly vary especially in the case the source data are depth images. A key aspect of the proposed model is to generalize the energy minimization problem to  jointly accommodate  estimation of the data and their confidence values, in the following form:
\begin{equation}\label{eq:genmodelnew}
E(u) = F(\mK u) + H(u,\lambda;d) + G(\lambda).
\end{equation}
This model  induces spatially adaptive regularization effects,  letting the coherence of the available data guide the regularization process.
The corresponding minimization problem is no longer convex, though we show that it is biconvex if $G$ is a convex, lower-semi-continuous functional. On this basis, 
we extend biconvex optimization algorithms for dealing with non-smooth functionals and examine their convergence. 
In summary,  this work contributes to the data fusion problem with a new model which we present in its discrete version so as to focus on the algorithms and the  experiments on different datasets, showing the performance of the model. 

Furthermore we present the algorithms Alternative Convex Search (ACS) and Alternate Minimization  (AMA), adapted to our model, showing that they converge for the biconvex joint estimation problem, and settle the conditions that have to be satisfied to guarantee convergence. 

We consider also the PDHG method for our model, contributing with a convergence analysis  for the case of {\em a-priori} assigned spatially varying confidence values, and provide suitable bounds for the PDHG step parameters. Moreover, we extend the analysis of the PDHG algorithm for the joint estimation problem and discuss its convergence.

The remaining of the work is organized as follows. Section~\ref{sec:related} discusses related work, Section~\ref{sec:prelim} introduces preliminary concepts and definitions. Section~\ref{sec:model}  introduces the confidence driven data fusion  model and its properties. In Section~\ref{sec:algs} we examine the convergence of the alternate minimization (AMA) and the alternate convex search (ACS) algorithms. We also discuss convergence of the PDHG algorithm for spatially varying confidence values. In Section~\ref{sec:results} we examine numerical results and  the performance of the model with respect to state of the art methods for the problem of depth image fusion on real and synthetic data. Finally, in Section~\ref{sec:conc} we provide some conclusions and future work directions.

\section{Related Work}\label{sec:related}

The idea of spatially altering the effects of  regularization, 
to the best of our knowledge, has been first introduced by Strong and Chan  \cite{Strong2003} who provided  analytical solutions for the minimizers of specific classes of signals.  They also considered spatially varying regularization parameters, to locally control the image scale space.
Calvetti and Sommersalo \cite{Calvetti2008}  use a weighting scheme based on the statistics of the edges in natural images, proposing the gamma and the inverse gamma distributions as hyper-priors of the regularization term. Their  Bayesian regularization model  includes  the Perona-Malik \cite{Perona1990} and ROF \cite{Rudin1992} models as special cases. 

We recall that Total Variation (TV) for image denoising has been introduced by Rudin, Osher and Fatemi (ROF) in \cite{Rudin1992}.
Several generalizations of total variation regularization have been proposed   to allow for exact reconstruction of higher-order piece-wise polynomial signals, e.g. piece-wise affine or quadratic signals. Some well known such generalizations are the Infimal Convolution Total Variation (ICTV)  proposed by Chambolle \cite{Chambolle1997} and Total Generalized Variation (TGV), introduced by Bredies and colleagues \cite{Bredies2010}.  
We consider the latter, which further generalizes ICTV. See \cite{Mueller2013} for further details and comparisons between the ICTV and TGV methods.

Going back to spatially varying  regularization effects,  Newcombe and colleagues \cite{Newcombe2011b}  apply  weighting parameters in order to ensure lower regularization near image edges, so as to enforce sharp edges of the computed depth image. In a similar way, \cite{Kuschk2013} proposes anisotropic regularization by considering the Nahel-Enkelmann operator applied to the regularization term. 
In  the mentioned works spatially varying weighting is applied to the regularization term.
Under this respect, the model we propose shows some important novelties. First of all, the weighting scheme is applied to the fidelity term. This brings  a new interpretation for the data fusion problem, in which the different contributions of the data sources are gauged by a map of confidence values. More importantly, the proposed model estimates these confidence values directly from the available data, by solving a biconvex minimization problem. Additionally, the model resorts to a fidelity term based on the $L_1$ norm, which is quite  robust to outliers.

As a result, the proposed method combines the advantages of $L_1$ regularization, namely robustness against impulsive noise, and contrast invariance, which corresponds to purely geometrical effects in the scale space, with the ability to locally control the image scale space, by  varying confidence values at each image region. As will be shown in the following,  the model entails a biconvex minimization problem, which poses some challenges in finding the optimal solution, with respect to convex TGV models. 

As a matter of fact, many interesting problems in image processing can be better modeled with non-convex regularization models. Recently a number of non-convex models have been proposed in order to attack the problems of image inpainting \cite{Esser2014}, depth smoothing \cite{Ochs2013}, and TV regularization on manifolds \cite{Lellmann2013,Ntouskos-2015ICCV,Natola-2016CVPR}. Algorithms for optimizing non-convex functionals have been recently proposed focusing on distinctive properties of the terms involved, we recall here some of them.

The {\em Alternating minimization} methods  
transform a constrained optimization problem into an unconstrained optimization one, by adding a quadratic penalization on the constraint violation. Typically the weight on the penalization term increases as the iterations proceed. 
Examples for this class of algorithms can be found in \cite{Nikolova2010}, and convergence properties  are discussed in \cite{Attouch2010}.

{\em Splitting}  methods are used when the problem can be separated in a  smooth non-convex term and a possibly non-smooth part. A recently proposed forward-backward splitting method for dealing with this class of problems, called {\it iPiano}, was introduced in \cite{Ochs2014}.  

{\em Semi convex regularization} is considered 
when the nonconvex term 
can be made convex, for example by adding an additional $L_2$ norm (see Section~\ref{sec:prelim}), \cite{Artina2013} proposed a method based on the augmented Lagrangian and proved that it converges to critical points. More recently \cite{Mollenhoff2015} proposed a PDHG method for problems with a semiconvex regularization term. The authors prove convergence of the algorithm to critical points when the convexity of the fidelity term compensates the nonconvexity of the regularization term, and they show various examples where the algorithm converges, even when this assumption is violated, indicating the (possibly local) robustness of the PDHG methods applied to nonconvex problems.
Finally,  Valkonen in \cite{Valkonen2014}   provides a proof of local convergence of the PDHG method in the case of {\em non-linear regularization operators} (NL-PDHG), when the non-linear operator satisfies certain smoothness assumptions and the operator of the update steps satisfies the Aubin property \cite{Aubin2009}.

The problem we propose touches, in some sense, all the above mentioned ones. Indeed, we discuss two algorithms for solving the biconvex optimization problem which gives the optimal solution of our model. First, we consider the alternate convex search algotiyhm \cite{Gorski2007} and then we examine the application of alternate minimization methods \cite{Attouch2010}, discussing their convergence to critical points. We consider also the application of the PDHG algorithm on biconvex optimization problems and discuss its convergence. 
 
As an application domain  we consider the problem of variational fusion of depth images, which is recognized to be a crucial aspect in  many surface reconstruction approaches. Campbell {\em et al.} \cite{Campbell2008} employ a Markov Random Field to find a solution for  multiple depth hypotheses. Merrell {\em et al.} in \cite{Merrell2007}  adopted a depth image fusion scheme, based on visibility, considering appropriate confidence measures to asses the stability of each depth estimate. In \cite{Hane2012} the authors use a reduced dictionary of depth patches to regularize and fuse depth images of mostly planar structures. 

Total generalized variation models for the fusion of depth images has been introduced in \cite{Pock2011}. 
In \cite{Ferstl2013}, the authors fuse low-resolution high-fidelity depth images, from Time-of-Flight sensors, with high-resolution and low-fidelity depth images, generated from stereo matching, using a primal-dual optimization algorithm on a model based on anisotropic diffusion. As mentioned above, in \cite{Ochs2013} the authors consider non-convex regularizers and propose an iterative algorithm for the optimization of the corresponding problems, evaluating their method with a number of image processing applications, including depth image fusion.

In a different line of work, \cite{Zach2007} proposes a volumetric fusion of the depth images based on Total Variation, to regularize the resulting signed distance function (SDF). In \cite{Fuhrmann2011} the authors propose a hierarchical SDF, which allows the fusion of depth images with very different scales. \cite{Newcombe2011a} proposes a method to both estimate  the pose of the RGB-D camera and to integrate new depth images with the reconstructed 3D model. Fusion is performed by taking the weighted average of individual truncated SDFs. Recently, \cite{Ummenhofer2015} has proposed a surface reconstruction approach from depth images by globally optimizing a signed distance function, defined on an octree grid, which scales very well with the number of input data.

Finally, we mention that image fusion is also treated in other application domains, like medical \cite{James2014} and hyper-spectral imaging \cite{Lanaras2015}, which we do not treat in this work.

\section{Preliminaries}\label{sec:prelim}
We provide here some definitions which the reader might find useful. If not otherwise stated, $\sU$ denotes a Banach space and $\sH$ a Hilbert space. Additionally, 
$\mathbb{S}_{++}^{N}$ denotes the space of $N\times N$ symmetric positive definite matrices, and $\mathbb{D}_{++}^{N}$ the space of $N\times N$ diagonal positive definite matrices. Although our analysis will mainly focus on Hilbert spaces, the following definitions are provided in a more general form considering Banach spaces.

\begin{mydef}[Affine set]
We recall that a set $X\subseteq\sU$ is affine if it contains all the linear combinations of pairs of points $x,y\in M$. 
\end{mydef}

\begin{mydef}[Affine hull]
The affine hull of $X\subseteq\sU$, denoted as $\aff X$ is the intersection of all affine sets that contain $X$.
\end{mydef}

\begin{mydef}[Relative Interior]
Let $C$ be a non-empty convex set. A point $x\in C$ belongs to the relative interior of $C$, namely $x\in \ri C$, if there exists an open sphere $S$ centered at $x$, such $S\cap\aff C\subseteq C$.
\end{mydef}

\begin{mydef}[Convex set]
A set $C\subseteq \sU$ is convex, if
\begin{equation*}
\gamma u + (1 - \gamma)v \in C \quad \forall\,u,v\in C,\, \forall\gamma\in [0,1].
\end{equation*}
\end{mydef}

\begin{mydef}[Proper functional] A functional $F:\sU\mapsto\extR $ is called proper if $F(u) \neq -\infty$ for all $u \in \sU$ and there exists at least one $u \in \sU$ with $F(u) \neq +\infty$. 
\end{mydef}

\begin{mydef}[Semicontinuity]
A functional $F:\sU \mapsto \extR$ is {\em lower semicontinuous} if
\begin{equation*}
\underset{v\rightarrow u}{\lim\inf}\, F(v)\geq F(u), \quad \forall u \in \sU  .
\end{equation*}
$F$ is upper semicontinuous if $-F$ is lower semicontinuous. $F$ is continuous at $u$ if and only if it is both upper and lower semicontinuous at $u$.
\end{mydef}

\begin{mydef}[Convex functional]
Let $C \subseteq \sU$ be a convex set. Then a function
$F: C \mapsto \extR$ is convex, if for all $u,v\in C$ and for all $\gamma\in [0,1]$ it holds
\begin{IEEEeqnarray}{rCl}
F(\gamma u + (1 - \gamma)v) &\leq & \gamma F(u) + (1 - \gamma)F(v).
\end{IEEEeqnarray}
$F$ is called strictly convex if this inequality holds strictly except for $u = v$ or $\gamma \in \{0, 1\}$.
\end{mydef}

\begin{mydef}[Semiconvexity \cite{Artina2013}]
A lower semi-continuous functional $F: \sU \mapsto \extR$ is called {\em $\omega$-semiconvex} if $F+\frac{\omega}{2}\norm{\cdot}^2$ is convex.\\
\end{mydef}

\begin{mydef}[Strong convexity \cite{Artina2013}]
A lower semicontinuous functional $F: \sU \mapsto \extR$ is called {\em c-strongly convex} if for all $u_1,u_2\in \sU$, $q_1\in\partial F(u_1)$, $q_2 \in \partial F(u_2)$, the following holds
\begin{equation*}
\inner{u_1-u_2}{q_1-q_2}\geq c\norm{u_1-u_2}^2.
\end{equation*}
\end{mydef}

The following propositions are useful when operations with convex functionals are involved (the corresponding proofs are provided in \cite{Bertsekas2009}).

\begin{myprop}
Let $F: X\subseteq{\sH} \mapsto \extR$ be a proper convex functional, and an operator $A\in\ssB{X}$, with $\ssB{X}$ the space of bounded linear operators from $X$ to $X$ with domain defined on $X$. 
Then the functional $G:X\mapsto\extR$ defined as
\begin{equation}
G(x) = F(Ax),\quad \forall x\in X,
\end{equation}
is convex.
\end{myprop}

The previous results leads also to the next proposition.
\begin{myprop}
Let $F_{i}:X\subseteq{\sH} \mapsto \extR$, $i=1,\ldots,m$, be  proper convex functionals on $X$, and let $\gamma_1,\ldots,\gamma_m>0$. Then the functional $G:X\mapsto\extR$ defined as
\begin{equation}
G(x) = \gamma_1 F_{1}(x)+\cdots+\gamma_m F_{m}(x),\quad \forall x\in X,
\end{equation}
is convex.
\end{myprop}

\begin{myprop}\label{prop:pointinf}
Let $F_{i}:X\subseteq{\sU} \mapsto \extR$ be proper convex functionals for $i\in I \subset\mathbb{N}$. Then the functional $G:X\mapsto \extR$ defined as 
\begin{equation}
G(x)=\underset{i\in I}{\inf}\,F_{i}(x),
\end{equation}
is convex.
\end{myprop}

\begin{mydef}[Subdifferential]
Let $\sU$ be a Banach space, $\sU^*$ its corresponding dual space, and $F: \sU \mapsto \extR$ a proper convex functional. Then the subdifferential $\partial F(u)$ at a point $u$ is defined as
\begin{IEEEeqnarray}{rCl}
\partial F(u) &\coloneqq &\{ p\in\sU^* \mid F(v)\geq F(u) + \inner{p}{v-u}\, \forall v \in \sU \}.
\end{IEEEeqnarray}
If the set $\partial F(u)$ is not empty, $F$ is called subdifferentiable at $u$. An element $p \in \partial F(u)$ is then called a subgradient of $F$ at $u$.
\end{mydef}

We now give the definition of the convex conjugate of a functional, called also Legendre-Fenchel transform, which is typically used to obtain the primal-dual form of a convex optimization problem.
\begin{mydef}[Convex conjugate]
Let $F:\sU\mapsto \extR$ be a general extended real-valued functional (not necessarily convex). Its {\em convex conjugate} $F^{\ast}:\sU^{\ast}\mapsto\extR$ is defined as
\begin{equation}
F^{\ast}(p) = \underset{u\in\sU}{\sup}\{\inner{u}{p}-F(u)\}.
\end{equation}
The convex conjugate is always convex, as it corresponds to the point-wise supremum of a collection of affine functions (see also Proposition \ref{prop:pointinf}).
The {\em double conjugate} functional is denoted by $F^{\ast\ast}$ and it is given by
\begin{equation}
F^{\ast\ast}(u) = \underset{p\in\sU^{\ast}}{\sup}\{\inner{u}{p}-F^{\ast}(p)\}.
\end{equation}
\end{mydef}
\begin{mynote}\label{note:biconjugate}
In general $F^{\ast\ast}(u)=(\check\cl) F(u)$ holds, where $\check{\cl}\,F$ denotes the convex closure of $F$. If $F$ additionally is a proper convex function then $F^{\ast\ast}(u)=F(u)$. 
\end{mynote}

\begin{mydef}[Biconvex set \cite{Gorski2007}]\label{def:biconvset}
Let $\sU,\sV$ be Banach spaces. The set $B\subseteq \sU \times \sV$ is called biconvex on $\sU \times \sV$ or biconvex for short, if $B_u\coloneqq\{v\in\sV\mid(u,v)\in B\}$ is convex for all $u \in \sU$ and $B_v\coloneqq\{u\in\sU\mid(u,v)\in B\}$ is convex for all $v\in \sV$.
\end{mydef}

\begin{mydef}[Biconvex functional \cite{Gorski2007}]
A functional $F: B \mapsto \extR$ on a biconvex set $B \subseteq \sU \times \sV$ is  biconvex, if for every fixed $u \in \sU$
\begin{subequations}
\begin{IEEEeqnarray}{rCl}
F_u(\cdot) &\coloneqq & F(u,\cdot): B_u \mapsto \extR
\end{IEEEeqnarray}
is a convex function on $B_u$ and for every fixed $v \in \sV$
\begin{IEEEeqnarray}{rCl}
F_v(\cdot) &\coloneqq & F(\cdot,v): B_v \mapsto \extR
\end{IEEEeqnarray}
\end{subequations}
is a convex function on $B_v$.
\end{mydef}

\begin{mydef}[Partial optimum]
Let $F: B\mapsto R$ be a biconvex functional. Then $(u^{\ast}, v^{\ast})\in B$ is called a {\em partial optimum} of $F$ on $B$, if for all $ u \in B_{v^{\ast}}$
\begin{subequations}
\begin{IEEEeqnarray}{rCl}
F(u^{\ast}, v^{\ast}) &\leq & F(u, v^{\ast})
\end{IEEEeqnarray}
and for all $v \in B_{u^{\ast}}$
\begin{IEEEeqnarray}{rCl}
F(u^{\ast}, v^{\ast}) &\leq & F(u^{\ast}, v).
\end{IEEEeqnarray}
\end{subequations}
\end{mydef}

The following theorem extends Theorems 4.1 and 4.2 of \cite{Gorski2007} to the case of non-smooth functions.
\begin{mythm}\label{thm:partopt}
Let $B$ be a biconvex set and let $F:B \mapsto \mathbb{R}$ be a biconvex functional. Then a point $z\coloneqq(x,y)\in \ri(B)$ is a stationary point of $F$ if and only if it is a partial minimum.
\end{mythm}

\begin{proof}
The forward direction is easily shown by using the definition of partial optimum. In particular, considering a partial minimum $\zeta\in \ri(B)$, then the optimality condition $0\in\partial F$ holds. Hence, $\zeta$ is a stationary point. 

The reverse direction is shown as follows. Let $\hat{z}=(\hat{x},\hat{y})\in\ri(B)$ be a stationary point of $F$. For $y=\hat{y}$, the functional $F_{\hat{y}}:B_{\hat{y}}\mapsto \mathbb{R}$ is convex. Since $\hat{x}$ is a stationary point of $F_{\hat{y}}$, then $0\in\partial F_{\hat{y}}(\hat{x})$. From the definition of subgradient then we have
\begin{subequations}
\begin{IEEEeqnarray}{rCl}
F_{\hat{y}}(x) &\geq &F_{\hat{y}}(\hat{x}) + \inner{0}{x-\hat{x}} = F_{\hat{y}}(\hat{x}),\ \forall x\in B_{\hat{y}}.
\end{IEEEeqnarray}
Analogously, for $x=\hat{x}$ we obtain
\begin{IEEEeqnarray}{rCl}
F_{\hat{x}}(y) &\geq &F_{\hat{x}}(\hat{y}),\ \forall y\in B_{\hat{x}}.
\end{IEEEeqnarray}
\end{subequations}
Hence, $\hat{z}$ is a partial minimum.
\end{proof}

\begin{mydef}[Total Variation]
Let $\diverg(\cdot)$ denote the divergence operator and $C_{0}^{\infty}(\Omega,\mathbb{R}^{N})$ the class of infinitely differentiable functions with compact support, with domain $\Omega$ and range $\mathbb{R}^{N}$. Given a function $u\in L^{1}(\Omega)$, its Total Variation is
\begin{IEEEeqnarray}{rCl}
\tv(u) &\coloneqq & \underset{\begin{subarray}{c}
q\in C_{0}^{\infty}(\Omega;\mathbb{R}^{N})\\
\|q\|_{\infty}\leq 1
\end{subarray}}{\sup} \int_{\Omega} u\diverg(q)\,dx.
\end{IEEEeqnarray}
\end{mydef}

In \cite{Bredies2010} the authors generalize the previous definition by means of symmetric tensors of a given order $j$, denoted as $\sym^{j}(\mathbb{R}^{N})$.
\begin{mydef}[Total Generalized Variation \cite{Bredies2010}]
Let  $\sym^{j}(\mathbb{R}^{N})$ be a symmetric tensor of order $j$. Given a function $u\in L^{1}(\Omega)$, its Total Generalized Variation is
\begin{IEEEeqnarray}{rCl}\label{eq:tgvprim}
\tgv_{\alpha}^{l}(u)  &=&  \underset{\begin{subarray}{c}
q\in C_{0}^{\infty}(\Omega;\sym^{j}(\mathbb{R}^{N}))\\
\|\diverg^{j}(q)\|_{\infty}\leq\alpha_j,j=0,\ldots,l-1
\end{subarray}}{\sup} \int_{\Omega} u\diverg^{l}(q)\,dx.
\end{IEEEeqnarray}
\end{mydef}

The $\tgv_{\alpha}^{l}$ functional is convex and can be seen as a combination of higher order $\tv$ terms, determined by the positive weights $\alpha\in\lbrace \alpha_1,\ldots,\alpha_l\rbrace$, with $l$ the maximum $\tv$ order.  

By taking the Legendre-Fenchel transform of (\refeq{eq:tgvprim}), an alternative definition can be given, namely
\begin{IEEEeqnarray}{rCl}
\tgv_{\alpha}^{l}&=&\underset{v\in V}{\inf} \sum_{j=1}^{l}\int_{\Omega} |\mK_{j} v|\, dx,
\end{IEEEeqnarray}
for $V=\{(u_0,\ldots,u_{l-1})\mid u_{j}\in C_{c}^{l-j}(\Omega;\sym^{j}(\mathbb{R}^{d}))\}$, and $M_{j}$ a suitable linear operator defined on the bases of the symmetrized gradient operator $\mathcal{T}=(\nabla u+\nabla u^{\top})/2$ and the weights vector $\alpha$.

\section{Fusion Model}\label{sec:model}
In this section we introduce the confidence driven fusion model and state some of its main properties. Let $\sX\subseteq\mathbb{R}^{N}$ be a finite-dimensional Hilbert space equipped with inner-product $\inner{\cdot}{\cdot}$ and norm $\norm{\cdot}_{2}=\sqrt{\inner{\cdot}{\cdot}}$, and let  $\sL\subseteq \sLmat$ be a finite-dimensional Hilbert space equipped with the Frobenious inner product %$\inner{\cdot}{\cdot}\coloneqq\tr(\adj(\cdot)\cdot)$ 
and the associated  norm. %induced norm $\norm{\vl} =\lbrace \norm{\vl\vx}: \vx\in\sX \mbox{ with } \norm{\vx}\leq1 \rbrace < \infty$. In particular in $\sL$ we consider the Frobenius inner product and the associated norm. 
The proposed fusion model, making precise the general model (\refeq{eq:genmodelnew}), is the following:
%We formulate (\refeq{eq:fusionmodelinf}) in finite dimensions as
\begin{IEEEeqnarray}{rCl}\label{eq:fusionmodel}
E(\vx,\vl) &\coloneqq &\tgv_{\alpha}^{l}(\vx) + \sum_{k=1}^{K}\norm{\vl(\vx-\dd_{k})}_1
	+\frac{1}{2}\tr(\pga^{-1}\vl)-\pgb\log{\det\vl},
\end{IEEEeqnarray}
with  $(\vx,\vl)\in\sX\times\sL $, $\pga\in\sL$, $\pgb>0$.

For $N\rightarrow\infty$ an infinite dimensional version of (\refeq{eq:fusionmodel}) is obtained. We focus though on the finite dimensional case and show the main properties of the proposed model for confidence driven fusion, which demonstrate the regularization behavior of the model and are essential for the convergence analysis of the algorithms considered in Section \ref{sec:algs}.

%%%%%%%%%%%%%%%%%%%%%%%%%%%%% Biconvexity %%%%%%%%%%%%%%%%%%%%%%%%%%%%%%%%%%%
\subsection{Convexity}
\begin{myprop}
The model (\refeq{eq:fusionmodel}) is biconvex on $\mathbb{R}^{N}\times\sLmat$.
\end{myprop}
\begin{proof}
Given $B\coloneqq\mathbb{R}^{N}\times\sLmat$, which is a convex set, we show that (\refeq{eq:fusionmodel}) is biconvex. 
Indeed, for fixed $\bar{\vl}\in \sL$,  both the $L_1$ norm and the $\tgv$ functional are convex in $\vx$, hence (\refeq{eq:fusionmodel}) is convex on the convex set $B_{\bar{\vl}}\coloneqq\{\vx\mid(\vx,\bar{\vl})\in B\}$. On the other hand,  for fixed $\bar{\vx}\in\mathbb{R}^{N}$, the $L_1$ norm, the trace and the ${-\log\det}$ operators are convex in $\vl$  \cite{Boyd2004}, hence (\refeq{eq:fusionmodel}) is convex on the convex set $B_{\bar{\vx}}\coloneqq\{\vl\mid(\bar{\vx},\vl)\in B\}$. 
It follows that (\refeq{eq:fusionmodel}) is biconvex on $B$.
\end{proof}

The following example shows that (\refeq{eq:fusionmodel}) is in general not convex in $(\vx,\vl)$.
\begin{myexamp}
Let $K=1$, $\pga^{-1}=2 I$, $\pgb=(e-1)(e+2)^{-1}$, $d_1=0$ and consider the values $\vz_0=(\vec{0},I)$, and  $\vz_1=(2\cdot\vec{1},e^{-1}I)$, for the joint variable $\vz\coloneqq(\vx,\vl)$. We have
\begin{IEEEeqnarray}{rCl?rCl}
E(\vz_0)&=&N,&  E(\vz_1)&=&N(1+3e^{-1}).
\end{IEEEeqnarray}
It follows that 
\begin{IEEEeqnarray}{rCl}
E_{1/2}&\coloneqq &\frac{E(\vz_0)+E(\vz_1)}{2}=N(1+\frac{3}{2}e^{-1}),
\end{IEEEeqnarray}
and
\begin{IEEEeqnarray}{rCl}
E\left(\frac{\vz_0+\vz_1}{2}\right)&= &E_{1/2}+N\left(\frac{e-1}{e+2}\log2\right)>E_{1/2}. 
\end{IEEEeqnarray}
Hence, (\refeq{eq:fusionmodel}) is in general not convex  in $\vz$.
\end{myexamp}

\begin{mynote}
The previous result shows that the model (\refeq{eq:fusionmodel}) is not convex with respect to the joint variable $(\vx,\vl)$. Hence, in general its minima  do not form a compact connected set. 
\end{mynote}

%%%%%%%%%%%%%%%%%%%%%%%%%%%%% Semiconvexity %%%%%%%%%%%%%%%%%%%%%%%%%%%%%%%%%%%
\begin{myprop}\label{prop:semiconv}
The model (\refeq{eq:fusionmodel}) is $\sqrt{N}$-semiconvex.
\end{myprop}
\begin{proof}
First, we show that the fidelity term $\norm{\vl(\vx-\dd)}_1$ is semiconvex, namely that
$D(\vz)\coloneqq\norm{\vl(\vx-\dd)}_1+\frac{\omega}{2}\norm{\vl}_2^{2}+\frac{\omega}{2}\norm{\vx}_2^{2}$ is convex, for $\omega\geq\sqrt{N}> 0$. That is, for $\gamma\in[0,1]$, and $\vz_1=(\vx_1,\vl_1)$ and $\vz_2=(\vx_2,\vl_2)$ we have 
$D(\gamma\vz_1+(1-\gamma)\vz_2)\leq(\gamma D(\vz_1)+(1-\gamma)D(\vz_2))$. 
Indeed, denoting $y_i=\vx_i-\dd$ and $\gamma^{c}=(1-\gamma)$, we have 
\begin{subequations} \label{eq:mixednorm}
\begin{IEEEeqnarray}{rCl}
&&\hspace{-3em}\norm{(\gamma\vl_1+\gamma^{c}\vl_2)(\gamma y_1+\gamma^{c}y_2)}_{1}-\gamma\norm{\vl_1 y_1}_1-\gamma^{c}\norm{\vl_2 y_2}_{1}\\
&{\leq}&\norm{(\gamma\vl_{1}+\gamma^{c}\vl_{2})(\gamma y_{1}+\gamma^{c}y_{2})}_{1}-\norm{\gamma\vl_{1} y_{1}+\gamma^{c}\vl_{2} y_{2}}_{1}\label{seq:leq1} \\ 
&\leq &\norm{(\gamma\vl_1+\gamma^{c}\vl_2)(\gamma y_1+\gamma^{c}y_2)-\gamma\vl_1 y_1-\gamma^{c}\vl_2 y_2}_{1}\label{seq:leq2} \\ 
&= & \norm{\gamma\gamma^{c}(\vl_1 y_2+\vl_2 y_2)-\gamma\gamma^{c}\vl_1 y_1-\gamma\gamma^{c}\vl_2 y_2}_{1}\\
&= & \gamma\gamma^{c}\norm{(\vl_1-\vl_2)(y_1-y_2)}_{1}\\\label{seq:leq3}
&\leq & \gamma\gamma^{c}\sqrt{N}\norm{(\vl_1-\vl_2)(y_1-y_2)}_2\\
&\leq & \gamma\gamma^{c}\sqrt{N}\norm{\vl_1-\vl_2}_2\,\norm{\vx_1-\vx_2}_2, \label{seq:leq4}
\end{IEEEeqnarray}
\end{subequations}
where convexity of the $\norm{\cdot}_p$ operator for $p\geq 1$ is used in (\refeq{seq:leq1}), triangle inequality in (\refeq{seq:leq2}), and Cauchy-Schwarz inequality in (\refeq{seq:leq3}) and (\refeq{seq:leq4}). 
On the other hand for the quadratic terms we have
\begin{equation}\label{eq:quadterms}
\norm{\gamma u_{1}+\gamma^{c} u_{2}}_2^{2}-\gamma\norm{u_{1}}_2^{2}-\gamma^{c}\norm{u_{2}}_2^{2}=-\gamma\gamma^{c}\norm{u_{1}-u_2}_2^2
\end{equation}
Adding (\refeq{eq:mixednorm}a-g) and (\refeq{eq:quadterms}) for $\vx$ and $\vl$, we get
\begin{subequations}
\begin{IEEEeqnarray}{rCl}
&&\hspace{-3em}D(\gamma\vz_1+\gamma^{c}\vz_2)-(\gamma D(\vz_1)+\gamma^{c}D(\vz_2))\\
&\leq & \gamma\gamma^{c}\left(\sqrt{N}\norm{\vl_1-\vl_2}_{2}\,\norm{\vx_1-\vx_2}_2-\frac{\omega}{2}\norm{\vx_1-\vx_2}_2^2-\frac{\omega}{2}\norm{\vl_1-\vl_2}_{2}^2\right).\label{seq:semiconvlast}
\end{IEEEeqnarray}
\end{subequations}
From (\refeq{seq:semiconvlast}) it is immediate that for $D(\cdot)$ to be convex $\omega\geq \sqrt{N}$ must hold. Since all other terms of (\refeq{eq:fusionmodel}) are convex, the statement holds.
\end{proof}

%%%%%%%%%%%%%%%%%%%%%%%%%%% Boundedness %%%%%%%%%%%%%%%%%%%%%%%%%%%%%%%%%%%
\subsection{Boundedness}

\begin{mythm} \label{prop:boundedness}
The model (\refeq{eq:fusionmodel}) is bounded from below.
\end{mythm}
\begin{proof}
We use the fact
\begin{equation}
\inf_{u} \sum_{v} f(u,v) \geq \sum_{v} \inf_{u} f(u,v).
\end{equation}
Hence
\begin{subequations}
\begin{IEEEeqnarray}{rCl}
\inf_{\vx,\vl} E(\vx,\vl) &\geq &\inf_{\vx} \tgv(\vx)+\sum_{k}\inf_{\vx,\vl}\norm{\vl(\vx-\dd)}\\
 &&+\inf_{\vl}\left\lbrace \frac{1}{2}\tr(\pga^{-1}\vl)-\pgb\log\det\vl \right\rbrace\\
&\geq &\inf_{\vl} \left\lbrace \frac{1}{2}\tr(\pga^{-1}\vl)-\pgb\log\det\vl \right\rbrace. \label{seq:lambdaopt}
\end{IEEEeqnarray}
\end{subequations}
The term in (\refeq{seq:lambdaopt}) has a finite infimum for every $\pga\in\sLmat, \pgb\geq 0$. To see this for $\pgb > 0$, we differentiate with respect to $\vl$ obtaining
\begin{IEEEeqnarray}{rCl}
\frac{1}{2}\tr(\pga^{-1})-\pgb\tr(\vl^{-1})&=&\tr\left( \frac{1}{2}\pga^{-1}-\pgb\vl^{-1}\right),
\end{IEEEeqnarray}
which vanishes for $\hat{\vl}=2\pgb\pga$.

Substituting back to (\refeq{seq:lambdaopt}) we get
\begin{IEEEeqnarray}{rClCl}
\inf_{\vx,\vl} E(\vx,\vl) &\geq & N\pgb\left(1-\log(\det 2\pgb\pga)^{1/N}\right)&>&-\infty.
\end{IEEEeqnarray}
For $\pgb = 0$ the infimum is trivially zero.
\end{proof}

\begin{mynote}
The previous proofs do not use the fact that $\sL\subseteq\sLmat$. In fact they are also valid for $\sL\subseteq\mathbb{S}^{N}_{++}$. We consider here  $\sL\subseteq\sLmat$ as it simplifies the convergence analysis of the minimization algorithms and is also computationally feasible.
Indeed, taking $\sL\subseteq\mathbb{S}^{N}_{++}$, then solutions are computationally feasible only for toy problems.
\end{mynote}

%%%%%%%%%%%%%%%%%%%%%%%%%%%% JUSTIFICATION %%%%%%%%%%%%%%%%%%%%%%%%%%%%%%%%%%%

Model (\refeq{eq:fusionmodel}) 
offers a new perspective to the general problem (\refeq{eq:basicmodel}),  focusing on the pair $(\vx,\vl)$.   In fact, a prominent problem in applying models such as (\refeq{eq:basicmodel}) is in the choice of the regularization parameter, especially in the case of non-smooth models. 

In principle, the choice of the regularization parameter is determined by the data coherence with respect to the solution of $\vx$ 
represented by the fidelity term.  Namely, the formulation using the regularization parameter on the penalty term tries to establish a compatibility of this parameter with the noise in the data.  Heuristic rules have been established in this sense,  such as for example the well known Hanke-Rause \cite{Hanke-1996,Engl-1996} rule explicitly linking the regularization parameter to the fidelity term.  This perspective requires some evaluation of the noise level, which turns out to be quite complex when the data comes in a multiplicity, such as in fusion applications. 

The approach we propose here does not require a prior knowledge on the noise level since this is implicitly coded in the scalar field represented by $\vl$, which is estimated by the given data. 
Here $\vl$ effectively balances the noise level, given by the fidelity term,  by spatially adapting the penalization term to the estimated value of $\vx$.   
Since $\vl$  is bounded from above, thanks to the hyperparameters   $\pgb>0$ and $\pga \in \sLmat$ , and given that $\vx$ is bounded too, we can see that in principle, the estimation of $\vl$ cannot add any new information where no information is available from the source data.  On the other hand, 
its values depend 
on the data coherence, adapting to the noise pointwise. These considerations are also illustrated in the optimality conditions of $\vl$ discussed in the next section.

\section{Algorithms} \label{sec:algs}

In this section we examine three different algorithms for finding the critical points of the biconvex model (\refeq{eq:fusionmodel}). 
We present first an adaptation of the ACS algorithm for the case of non-smooth functionals, and then AMA, which is also commonly 
used for the solution of non-convex optimization problems. We discuss its application for minimizing (\refeq{eq:fusionmodel}),  and its relation with ACS. Both these algorithms introduce convex minimization subproblems. We present the PDHG algorithm for spatially varying confidence values which can be used to solve these convex subproblems.  Finally, we discuss the applicability of the PDHG algorithm on the biconvex problem.

\subsection{Alternative Convex Search} 
The ACS algorithm \cite{Wendell1976,Gorski2007} is an algorithm commonly used for solving biconvex problems. We discuss here its convergence for minimizing (\refeq{eq:fusionmodel}). ACS is based on a relaxation of the original problem, by minimizing at each iteration a set of variables which lead to a convex subproblem.

\begin{myalg}[ACS]\label{alg:acs}
Choose an initial estimate $(\vx_0,\vl_0)\in\sX\times\sL$. For every $n\geq 0$ iterate
\begin{description}
\item[Iter 1] $\vl_{n+1}\in \arg\min\left\lbrace E(\vx_{n},\vl): \vl \in B_{\vx_{n}}\right\rbrace$, \label{it:acs1}
\item[Iter 2] $\vx_{n+1}\in \arg\min\left\lbrace E(\vx,\vl_{n+1}): \vx \in B_{\vl_{n+1}}\right\rbrace$. \label{it:acs2}
\end{description}
\end{myalg}

Considering the optimality condition of the optimization problem in Iter~1, 
the updates for the elements $i=1,\ldots,N$ of the diagonal of $\vl$ are given by
\begin{IEEEeqnarray}{rCl}\label{eq:lambdaopt}
(\vl_{n+1})_{i,i} &=& \frac{b}{\sum_{k=1}^{K} |(\vx_n)_{i}-(\dd_{k})_{i}|+\frac{1}{2}(\pga)^{-1}_{i,i}}.
\end{IEEEeqnarray}

As discussed in Section \ref{sec:model}, $\pga$ and $\pgb$ correspond to hyper-parameters of the model (\refeq{eq:fusionmodel}), which result in a regularization of  $\vl$ as will be discussed below. Examples regarding the values that can be assigned to $\pga$ and $\pgb$ and their effect on the  solution are discussed in Section \ref{sec:results}.

Before discussing convergence of ACS for the minimization of (\refeq{eq:fusionmodel}), we review two theorems given in \cite{Gorski2007}.

\begin{mythm}\label{thm:acsfconv}
Let $B \subseteq \sU\times\sV$, $F : B \mapsto \mathbb{R}$ be bounded from below, and let
the optimization problems at each iteration of ACS be solvable. Then the sequence $\{F(u_{n},v_n)\}_{n\in\mathbb{N}}$ generated by ACS converges monotonically.
\end{mythm}

\begin{mythm}\label{thm:acsconv}
Let $U \subseteq \sU$ and $V \subseteq \sV$ be closed sets, $F : U \times V \mapsto \mathbb{R}$ be continuous, and let the optimization problems at each iteration of ACS be solvable.
\begin{enumerate}
\item If the sequence $\{z_n\}_{n\in\mathbb{N}}$ generated by the ACS algorithm is contained in a compact set, then the sequence has at least one accumulation point.
\item In addition suppose that for each accumulation point $z^{\ast} = (u^{\ast},v^{\ast})$ of the sequence $\{z_n\}_{n\in\mathbb{N}}$ the optimal solution of ACS for $v = v^{\ast}$ or the optimal solution for $u = u^{\ast}$ is unique, then all accumulation points are partial optima and have the same functional value.
\item If for each accumulation point $z^{\ast} = (u^{\ast},v^{\ast})$ of the sequence $\{z_n\}_{n\in\mathbb{N}}$ the solution of both iterations are unique then $\norm{z_{n+1}-z_n} \rightarrow 0$, and the accumulation points form a connected, compact set.
\end{enumerate}
\end{mythm}

The following lemma justifies the roles of the terms $\tr(\pga^{-1}\vl)$ and $\pgb\log{\det\vl}$. 

\begin{mylem}\label{lem:bound}
The sequence $\{\vl_n\}_{n\in\mathbb{N}}$, produced by ACS for the model (\refeq{eq:fusionmodel}), is well defined and bounded from above for $b>0$ and $W\in\sL$.
\end{mylem}

\begin{proof}
We note first that for $b>0$, (\refeq{eq:fusionmodel}) has a unique attainable optimum with respect to $\vl\in\sL$ for every $\vx_n$, given by (\refeq{eq:lambdaopt}). Additionally, the denominator of (\refeq{eq:lambdaopt}) is always greater than zero for $\pga\in\sL$, thus the sequence $\{\vl_n\}_{n\in\mathbb{N}}$ is bounded from above by $2b\max_i\{(W)_{i,i}\}$. 
\end{proof}

In the following, we use Theorems \ref{thm:acsfconv} and \ref{thm:acsconv}, and Lemma \ref{lem:bound} to prove weak convergence of the ACS algorithm to the critical points of (\refeq{eq:fusionmodel}).

\begin{myprop}
The sequence $\{(\vx_n,\vl_n)\}_{n\in\mathbb{N}}$ obtained by applying Algorithm \ref{alg:acs} for minimizing (\refeq{eq:fusionmodel}), converges weakly across subsequences to critical points of (\refeq{eq:fusionmodel}).
\end{myprop}

\begin{proof}
The sequence $\{\vl_n\}_{n\in\mathbb{N}}$  is bounded by Lemma \ref{lem:bound}. 
Hence, the sequence $\{(\vx_n,\vl_n)\}_{n\in\mathbb{N}}$ is bounded due to the boundedness of $\{\vx_n\}_{n\in\mathbb{N}}$ and $\{\vl_n\}_{n\in\mathbb{N}}$. By Bolzano-Weirstrass theorem $\{(\vx_n,\vl_n)\}_{n\in\mathbb{N}}$ has at least one accumulation point. 

By Theorem \ref{prop:boundedness} and Theorem \ref{thm:acsfconv}, the sequence $\{E(\vx_{n},\vl_n)\}_{n\in\mathbb{N}}$, generated by Algorithm \ref{alg:acs}, converges monotonically. Then, by Theorem \ref{thm:acsconv} all accumulation points have the same functional value and hence correspond to partial optima of (\refeq{eq:fusionmodel}). Finally, by Theorem \ref{thm:partopt} all partial optima correspond to critical points of (\refeq{eq:fusionmodel}), which proves the statement.
\end{proof}

We note that the optimal solution of $\vl$ at each iteration depends on the current value of $\vx_n$ and more specifically, on the coherence of $\vx_n$ with the data. Following the proof above, the same holds for the optimal solution $(\hat{\vx},\hat{\vl})$.

The solution of Iter~2 can be estimated using the PDHG algorithm, 
as discussed in Section \ref{ssec:pdhg}.

\subsection{Alternate minimization method}
AMA is another algorithm which can be used to solve biconvex problems (see \cite{Attouch2010}). Here we briefly review AMA and discuss its convergence for finding the stationary points of (\refeq{eq:fusionmodel}). 

\begin{myalg}[AMA] \label{alg:ama}
Choose initial estimate $(\vx_0,\vl_0)\in\sX\times\sL$. For every $n\geq 0$ iterate
\begin{description}
\item[Iter 1] $\vl_{n+1}\in \underset{\vl \in B_{\vx_{n}}}{\arg\min}\left\lbrace E(\vx_{n},\vl) + \frac{1}{2\nu_{n}}\norm{\vl-\vl_{n}}^2\right\rbrace$,
\item[Iter 2] $\vx_{n+1}\in \underset{\vx \in B_{\vl_{n+1}}}{\arg\min}\left\lbrace E(\vx,\vl_{n+1}) + \frac{1}{2\mu_{n}}\norm{\vx-\vx_{n}}^2\right\rbrace$,
\end{description}
with $\mu_n,\nu_n>0$ for all $n$. We observe that Algorithms \ref{alg:ama} and \ref{alg:acs} become equivalent for $\mu_{n},\nu_{n}\rightarrow\infty.$ 
\end{myalg}

Regarding the convergence of AMA for minimizing model (\refeq{eq:fusionmodel}), we appeal to the convergence analysis presented in \cite{Attouch2010}. In \cite{Attouch2010} 
Lipschitz continuity of the gradient of $H$ is required with respect to one of the variables. This is satisfied by (\refeq{eq:fusionmodel}) for the variable $\vl$. Hence, the AMA algorithm converges for minimizing model (\refeq{eq:fusionmodel}) \cite[Theorem 3.3]{Attouch2010}, given that the model satisfyies the Kurdyka-\L{}ojasiewicz inequality at the optimal point $(\hat{\vx},\hat{\vl})$.

\begin{mynote}
Model (\refeq{eq:fusionmodel}) is $\sqrt{N}$-semiconvex (see Proposition \ref{prop:semiconv}) hence choosing $\mu_n,\nu_n\leq\sqrt{N}$ makes the optimization problem convex. This fact can be used for selecting initial values $\mu_n,\nu_n$ which make the problem convex at the beginning and progressively increase in order to better approximate the original biconvex optimization problem.
\end{mynote}

Regarding the update of variable $\vl$, each element of its diagonal leads to the following quadratic problem 
\begin{equation}
(\vl)_{i,i}^{2}-{a_n}(\vl)_{i,i}-\pgb\nu_n=0,
\end{equation}
with
\begin{equation}
{a_n}=(\vl_n)_{i,i}  - \nu_{n}\left(\sum_{k=1}^{K}|(\vx_n)_i-(\dd_{k})_i|+\frac{1}{2}(\pga_{i,i})^{-1}\right),
\end{equation}
which has the following closed form solution
\begin{align}
(\vl_{n+1})_{i,i} &= \frac{1}{2}\left(a_n+\sqrt{a_n^2+4\pgb\nu_n} \right).
\end{align}

The updates of the variable $\vx$ can be estimated using the PDHG algorithm. 

\subsection{PDHG for spatially varying confidence values}\label{ssec:pdhg}
In this section we examine the application of the PDHG algorithm for minimizing problems with spatially varying fidelity weights and the conditions under which the series $\{\vx_n\}_{n\in\mathbb{N}}$ converges. Our analysis is based on monotone operator theory. We refer the reader to \cite{Bauschke2011,Condat2013} and the references therein for further details.

For the convenience of the reader we consider here the general formulation (\refeq{eq:genmodelnew}) which is typically used for the PDHG algorithm. Let us consider a Hilbert space $\sH$ and denote $\cG{\sH}$ the set of proper, lower semicontinuous, convex functions from $\sH$ to $\extR$. 
Additionally, let
\begin{equation}\label{eq:modelgeneral}
E(\vx,\vl) = F(\mK\vx)+\sum_{k}H(\vl(S\vx-\dd_k))+G(\vl),
\end{equation}
with:
\begin{itemize}
\item $S$ a selection operator which depends on the order of the $\tgv$ operator $\mK$. E.g. for TV regularization $S=Id$ and $\mK=\nabla$;
\item  $ F\in\cG{\sY}$ and $G\in\cG{\sL}$;
\item $H:\sX\times\sL \mapsto \mathbb{R}$ is proper, lower semicontinuous and biconvex in $(\vx,\vl)$; 
\item  $\mK:\sX\mapsto\sY$ a bounded linear operator with induced norm $
\norm{\mK} =\lbrace \norm{\mK\vx}\mid \vx\in\sX \mbox{ with } \norm{\vx}\leq1 \rbrace < \infty$.
\item All functions have closed-form resolvent operators or they can be solved efficiently with high precision.
\end{itemize}

The model (\refeq{eq:fusionmodel}) is of the general form (\refeq{eq:modelgeneral}), with 
\begin{subequations}
\begin{IEEEeqnarray}{rCl}
F(\mK\vx)&\coloneqq &\tgv^{l}_{\alpha}(\vx),\\
H(\vl(S\vx-\dd_k))&\coloneqq &\norm{\vl(\vx-\dd_{k})}_1,\\
 G(\vl)&\coloneqq &\frac{1}{2}\tr(\pga^{-1}\vl)-\pgb\log{\det\vl}.
\end{IEEEeqnarray}
\end{subequations}

We consider here that $\vl$ is fixed to the value $\bar{\vl}$ throughout the minimization. Applying the Legendre-Fenchel transformation to the functionals $F$ and $H$ and substituting them in (\refeq{eq:fusionmodel}) we obtain the following equivalent formulation 
 \begin{IEEEeqnarray}{rCl}\label{eq:primaldual}
E^{\star}(\vx,\vq,\{\vp_k\})&=
&\inner{\mK\vx}{\vq}+\sum_{k=1}^{K}\inner{\bar{\vl}(S\vx-\dd_k)}{\vp_k}-F^{\ast}(\vq)-\sum_{k=1}^{K} H^{\ast}(\vp_k),
\end{IEEEeqnarray}
with $q$ the dual variable corresponding to $F$ and $p_k$ the dual variables corresponding to $H$.

According to the Karush-Kuhn-Tucker conditions, the saddle points $\hat{\zeta}=(\hat{\vx},\hat{\vq},\hat{\vp}_{1},\ldots,\hat{\vp}_{K})$ of (\refeq{eq:primaldual}) satisfy the following monotone variational inclusion
\begin{equation}\label{eq:varincl}
\begin{pmatrix}
0\\0\\0\\[-3pt]\vdots\\0
\end{pmatrix}
\in
\begin{pmatrix}
\mK^{\top}\hat{\vq}+\sum_{k=1}^{K}S^{\top}\bar{\vl}\hat{\vp}_{k}\\
\partial F^{\ast}(\hat{\vq})-\mK\hat{\vx}\\
\partial H^{\ast}(\hat{\vp}_1)-\bar{\vl}(S\hat{\vx}-\dd_1)\\[-3pt]
\vdots\\
\partial H^{\ast}(\hat{\vp}_K)-\bar{\vl}(S\hat{\vx}-\dd_K)
\end{pmatrix}
,
\end{equation}
where each row corresponds to the optimality condition of each variable involved in the optimization. We assume that the saddle points of (\refeq{eq:primaldual}) form a non empty set. This assumption makes the previous condition also sufficient, hence every point satisfying (\refeq{eq:varincl}) is a saddle point of (\refeq{eq:primaldual}).

\begin{myalg}[PDHG for spatially varying fidelity weights]\label{alg:pdhg}
Choose an initial estimate $\vx_0\in\sX$. For every $n\geq 0$ iterate
\begin{IEEEeqnarray}{rCl}
\vx_{n+1} &=& \textstyle{\vx_{n}-\tau\left(\mK^{\top}\vq_{n}+\sum_{k=1}^{K} S^{\top}\bar{\vl}{\vp_{k}}_{n}\right)},\\
\tilde{\vx}_{n+1} &=& 2\vx_{n+1}-\vx_{n},\\
\vq_{n+1} &\in& (Id+\sigma_{\vq}\partial F^{\ast})^{-1}(q_{n}+\sigma_{\vq}\mK\tilde{\vx}_{n+1}),\\
{\vp_{k}}_{n+1} &\in& (Id+\sigma_{\vp}\partial H^{\ast})^{-1}({\vp_{k}}_{n}+\sigma_{\vp}\bar{\vl}(S\tilde{\vx}_{n+1}-\dd_k)),\ \mbox{for}\ k=\{1,\ldots,K\}.
\end{IEEEeqnarray}
\end{myalg}

The iterations of Algorithm \ref{alg:pdhg} can be rewritten as 
\begin{subequations}\label{eq:pdhhgoper}
\begin{equation}
-\begin{pmatrix}
0\\0\\\bar{\vl}\dd_{1}\\[-3pt] \vdots \\ \bar{\vl}\dd_{K}
\end{pmatrix}
\in
\begin{pmatrix}
\mK^{\!\top}\vq_{n+1}+\sum_{k=1}^{K}S^{\!\top}\bar{\vl}{\vp_{k}}_{n+1}\\[3pt]
-\mK\vx_{n+1}+\partial F^{\ast}(\vq_{n+1})\\
-\bar{\vl}S\vx_{n+1}+\partial H^{\ast}({\vp_{1}}_{n+1})\\[-3pt]
\vdots\\
-\bar{\vl}S\vx_{n+1}+\partial H^{\ast}({\vp_{K}}_{n+1})\\
\end{pmatrix}
+
P(\zeta_{n+1}-\zeta_n)
,
\end{equation}
with $\zeta\coloneqq(\vx,\vq,\vp_{1},\ldots,\vp_{K})$ and
\begin{equation}
P=
\begin{pmatrix}
\frac{1}{\tau}Id&-\mK^{\!\top}&-S^{\!\top}\bar{\vl}&\cdots&-S^{\!\top}\bar{\vl}\\
-\mK&\frac{1}{\sigma_{\vq}}Id&0&\cdots&0\\
-\bar{\vl}S&0&\frac{1}{\sigma_{\vp}}Id&\cdots&0\\[-3pt]
\vdots&\vdots&\vdots&\ddots&0\\
-\bar{\vl}S&0&0&\cdots&\frac{1}{\sigma_{\vp}}Id\\
\end{pmatrix},
\end{equation}
\end{subequations}
which can be represented in the following form
\begin{equation}\label{eq:varinc}
-B(\zeta_n)\in A(\zeta_{n+1})+P(\zeta_{n+1}-\zeta_{n}).
\end{equation}

Let $O_1\circ O_2$ denote the composition of operators $O_1$ and $O_2$. Solving with respect to $\zeta_{n+1}$, we obtain
\begin{equation}\label{eq:pdhgppa}
\zeta_{n+1} = (Id+P^{-1}\circ A)^{-1}\circ(Id-P^{-1}\circ B)(\zeta_{n}).
\end{equation}

\begin{mylem}\label{lem:steps}
Matrix $P$ is be bounded, self-adjoint, and strictly positive, namely $\inner{\zeta}{P\zeta}>0$, for every $\zeta\neq 0$ for
\begin{subequations}\label{eq:sqstep}
\begin{equation}
\sigma_{\vq}\tau \norm{\mK}^2 \leq \frac{1}{K+1},
\end{equation}
and
\begin{equation}
\sigma_{\vp}\tau \norm{\bar{\vl}}^2 \leq \frac{1}{K+1}.
\end{equation}
\end{subequations}
\end{mylem}

\begin{proof}
$P$ is bounded, and self-adjoint by definition. Considering $\inner{\zeta}{P\zeta}$ we have
\begin{subequations}
\begin{IEEEeqnarray}{rCl}
\inner{\zeta}{P\zeta}&=& \frac{\norm{\vx}^{2}}{(K+1)\tau}-2\inner{\mK\vx}{\vq}+\frac{\norm{\vq}^{2}}{\sigma_q}\notag\\
&&+\sum_{k=1}^{K}\left(\frac{\norm{\vp_k}^{2}}{\sigma_p}-2\inner{\vl\vx}{\vp_{k}}+\frac{\norm{\vx}^{2}}{(K+1)\tau}\right)\\
&{\geq}&\frac{\norm{\vx}^{2}}{(K+1)\tau}-2\norm{\mK}\norm{\vx}\norm{\vq}
+\frac{\norm{\vq}^{2}}{\sigma_q}\notag\\
&&+\sum_{k=1}^{K}\left(\frac{\norm{\vp_k}^{2}}{\sigma_p}-2\norm{\vl}\norm{\vx}\norm{\vp_{k}}+\frac{\norm{\vx}^{2}}{(K+1)\tau}\right).\label{eq:stepslast}
\end{IEEEeqnarray}
\end{subequations}
For $\inner{\zeta}{P\zeta}$ to be positive, we require that all the terms in parentheses in (\refeq{eq:stepslast}) are positive. Using Young's inequality we recover (\refeq{eq:sqstep}).
\end{proof}

\begin{myprop}
Let $A$ and $B$ be the operators  defined in (\refeq{eq:pdhhgoper}). If (\refeq{eq:sqstep}) is satisfied, Algorithm \ref{alg:pdhg} converges to the zeros of the $A+B$ operator, namely $\zer(A+B)$.
\end{myprop}
\begin{proof}
Equation (\refeq{eq:pdhgppa}) is an instance of the proximal point algorithm as described in \cite{Condat2013}. Hence, Algorithm \ref{alg:pdhg} converges to $\zer(A+B)$ operators if $A$ is maximally monotone, and matrix $P$ is bounded, self-adjoint, and strictly positive. 
The latter follows from Lemma \ref{lem:steps}. 

To show that $A$ is maximally monotone we follow \cite{Condat2013}. The operator $\zeta\mapsto \emptyset \times \partial F^{\ast}(x) \times \partial H^{\ast}(\vp_1)\times\cdots \times \partial H^{\ast}(\vp_K)$ is maximally monotone by Theorem 20.40, Corollary 16.24, Propositions 20.22 and 20.23 of \cite{Bauschke2011}. Moreover, the skew operator
\begin{equation}
\zeta \mapsto (M^{\!\top}\vq+\sum_{k=1}^{K}S^{\!\top}\bar{\vl}\vp_{k},-M\vx,-\bar{\vl}S\vx,\ldots,-\bar{\vl}S\vx),
\end{equation} 
is maximally monotone by \cite[Example 20.30]{Bauschke2011} and has full domain. Hence, by \cite[Corollary 24.4(i)]{Bauschke2011} $A$ is maximally monotone.
\end{proof}

\subsection{PDHG for biconvex problems}
We consider now the biconvex problem of minimizing (\refeq{eq:fusionmodel}) with respect to the joint variable $(\vx,\vl)$, using an extension of the PDHG algorithm for biconvex problems.

\begin{myalg}[PDHG for biconvex problems]\label{alg:pdhgbiconv}
Choose an initial estimate $(\vx_0,\vl_0)\in\sX\times\sL$. For every $n\geq 0$ iterate
\begin{equation}
\begin{array}{lcl}
\vl_{n+1} &\in& (Id+\tau_{\vl}\partial G)^{-1}\left(\vl_{n}-\tau_{\vl}\sum_{k=1}^{K}\diag((S\vx_{n}-\dd_k){\vp_{k}^{\top}}_n) \right),\\
\vx_{n+1} &=& \vx_{n}-\tau_{\vx}\left(\mK^{\top}\vq_{n}+\sum_{k=1}^{K}S^{\top}\vl_{n+1}{\vp_{k}}_{n}\right),\\
\bar{\vx}_{n+1} &=& 2\vx_{n+1}-\vx_{n},\\
\vq_{n+1} &\in& (Id+\sigma_{\vq}\partial F^{\ast})^{-1}(q_{n}+\sigma_{\vq}\mK\bar{\vx}_{n+1}),\\
{\vp_{k}}_{n+1} &\in& (Id+\sigma_{\vp}\partial H^{\ast})^{-1}({\vp_{k}}_{n}+\sigma_{\vp}\bar{\vl}(S\tilde{\vx}_{n+1}-\dd_k)),\,k=\{1,\ldots,K\}.
\end{array}
\end{equation}
\end{myalg}

The iterations of Algorithm \ref{alg:pdhgbiconv} can be written in the form of (\refeq{eq:varinc}) as before. 

There are two important differences in this case with respect to Algorithm \ref{alg:pdhg}. The first, is that the matrix $P$ is changing at each iteration. It is still possible to guarantee that $P_{n+1}$ is strictly positive at every iteration by considering step sizes that vary in each iteration according to (\refeq{eq:sqstep}). 
The second, and more important difference, is that in this case the operator $A$ is not monotone, and as a result the analysis based on proximal point methods cannot be directly applied to prove that the algorithm converges. We note though that if we find experimentally that the sequences  $\{\vl_{n}\}_{n\in\mathbb{N}}$, $\{\vx_{n}\}_{n\in\mathbb{N}}$, $\{\vq_{n}\}_{n\in\mathbb{N}}$ and $\{{\vp_{k}}_{n}\}_{n\in\mathbb{N}}$ remain bounded and additionally $\norm{\vl_{n+1}-\vl_{n}}\rightarrow 0$, $\norm{\vx_{n+1}-\vx_{n}}\rightarrow 0$, $\norm{\vq_{n+1}-\vq_{n}}\rightarrow 0$ and $\norm{{\vp_{k}}_{n+1}-{\vp_{k}}_{n}}\rightarrow 0$, then the algorithm converges to critical points (see \cite{Mollenhoff2015,Esser2014}).

\section{Results} \label{sec:results}

In this section we present numerical results, demonstrating the performance of the proposed confidence driven TGV regularization model. 
We consider depth image fusion as an application domain for evaluating the confidence driven fusion process. 

First, we demonstrate numerically relevant properties of the proposed model using synthetic data, highlighting the well-foundedness of the point-wise confidence operator. 
Then, we thoroughly evaluate the fusion performance of our model using synthetic datasets comprising  several 3D models of objects and urban landscapes. 
Finally, we evaluate our model on real data using a publicly available dataset. 
The datasets used for the evaluation of the proposed model are provided at \url{www.diag.uniroma1.it/~alcor/site/index.php/software.html}. An implementation of the proposed model for the fusion of depth images in MATLAB and CUDA is available at \url{www.github.com/alcor-vision/confidence-fusion}.

\subsection{Numerical results} \label{ssec:numres}
We illustrate here the main properties of (\refeq{eq:fusionmodel}), via  numerical results. We considering the effects on a single depth image. In the case of uniform confidence values $\vl=cI$ with $c>0$, the model reduces to the $\tgv^{l}$-~L1 model. This model, for $l=0$, has been thoroughly examined in the literature (see \cite{Strong2003,Nikolova2004,Chan2005,Chambolle2010}). Here we are particularly interested in the relation of the confidence values with the scale of the imaged objects.

This relation has been examined in \cite{Strong2003} for the original ROF model \cite{Rudin1992}, and in \cite{Chan2005} for the $\tv$-L1 model. Indeed, Chan and Esedoglu in \cite{Chan2005} argue that the regularization of an image using the $\tv$-L1 model leads small scale objects to suddenly disappear in relation to the value of $c$. In particular, structures are affected independently of their contrast values, as opposed to the original ROF model where they start to lose contrast as  $c$ becomes smaller than a critical value. This observation justifies the use of the L1 fidelity term in (\refeq{eq:fusionmodel}) for the case of depth images, as changes of contrast correspond to distortions of the actual depth values.

\begin{figure}[th!]%
\centering
\subfloat[{ Original}]
{\includegraphics[width=0.17\textwidth]{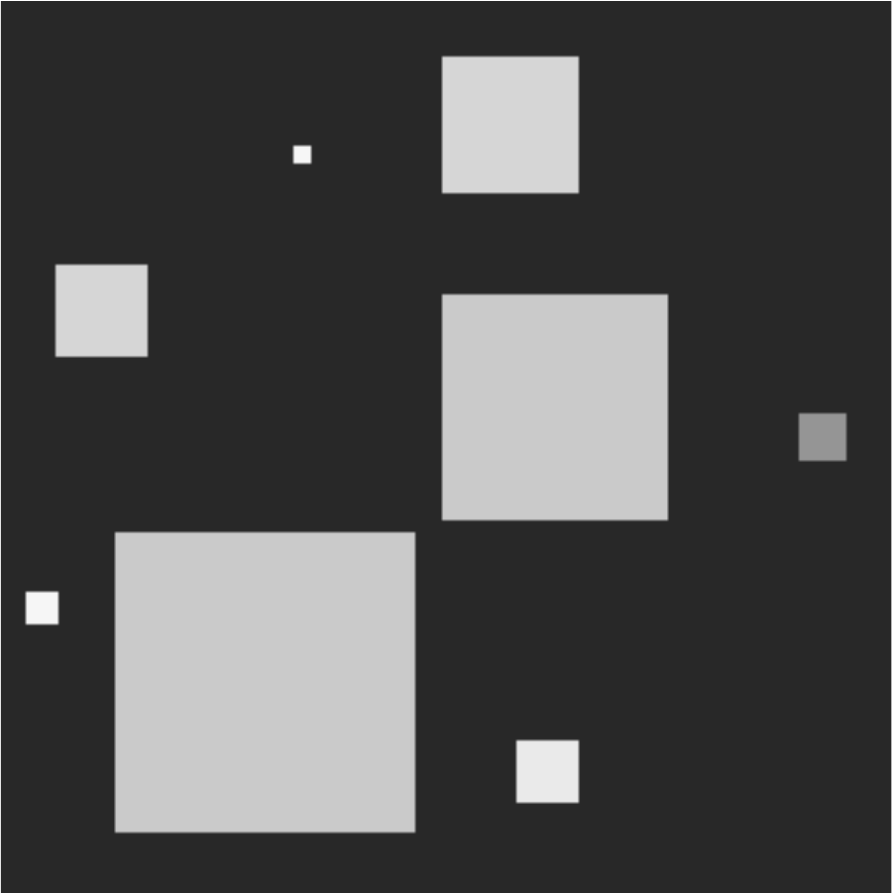}}\hfil
\subfloat[{ $c=1.0$}]
{\includegraphics[width=0.17\textwidth]{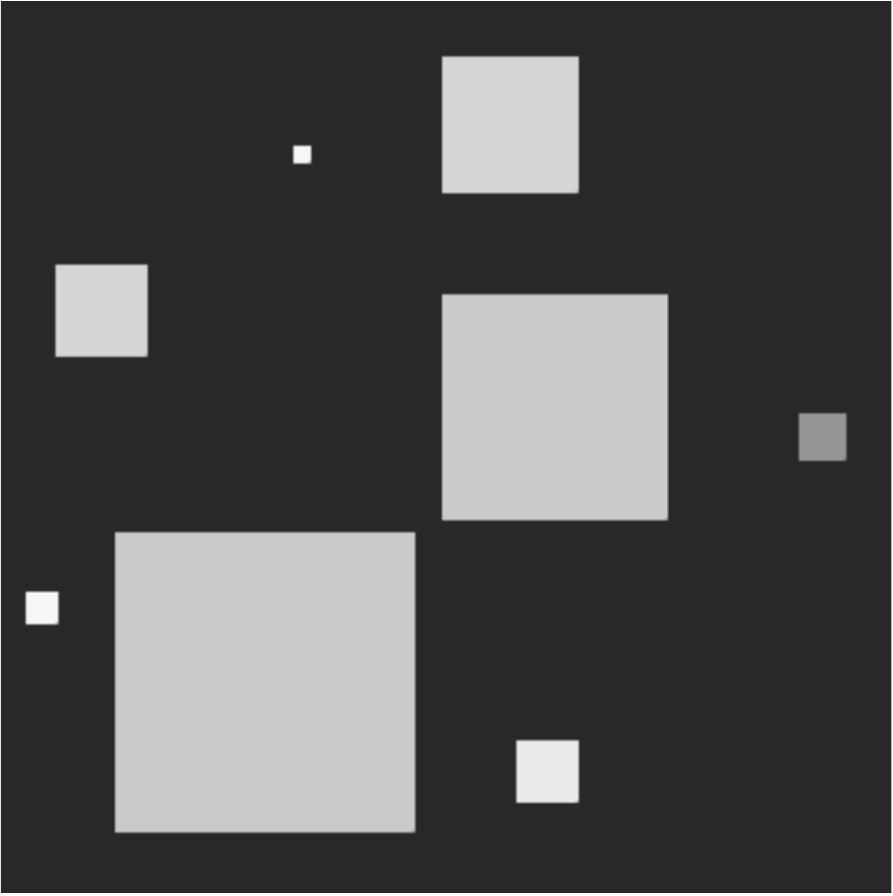}}\hfil
\subfloat[{ $c=0.6$}]
{\includegraphics[width=0.17\textwidth]{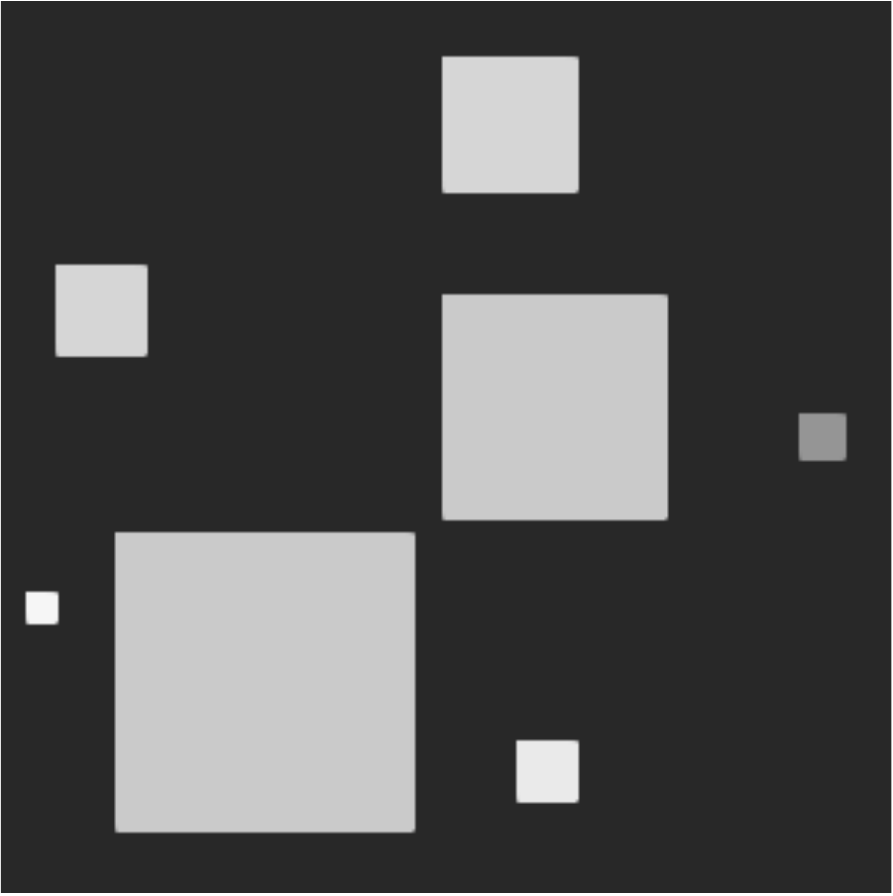}}\hfil
\subfloat[{ $c=0.3$}]
{\includegraphics[width=0.17\textwidth]{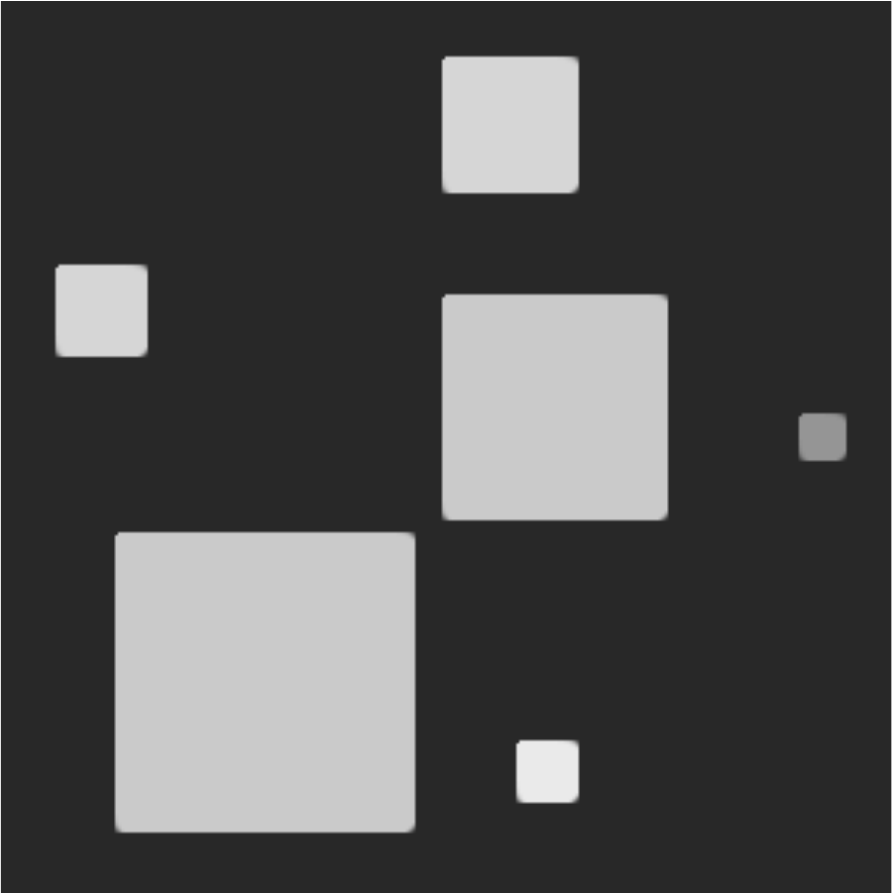}}\hfil
\subfloat[{ $c=0.2$}]
{\includegraphics[width=0.17\textwidth]{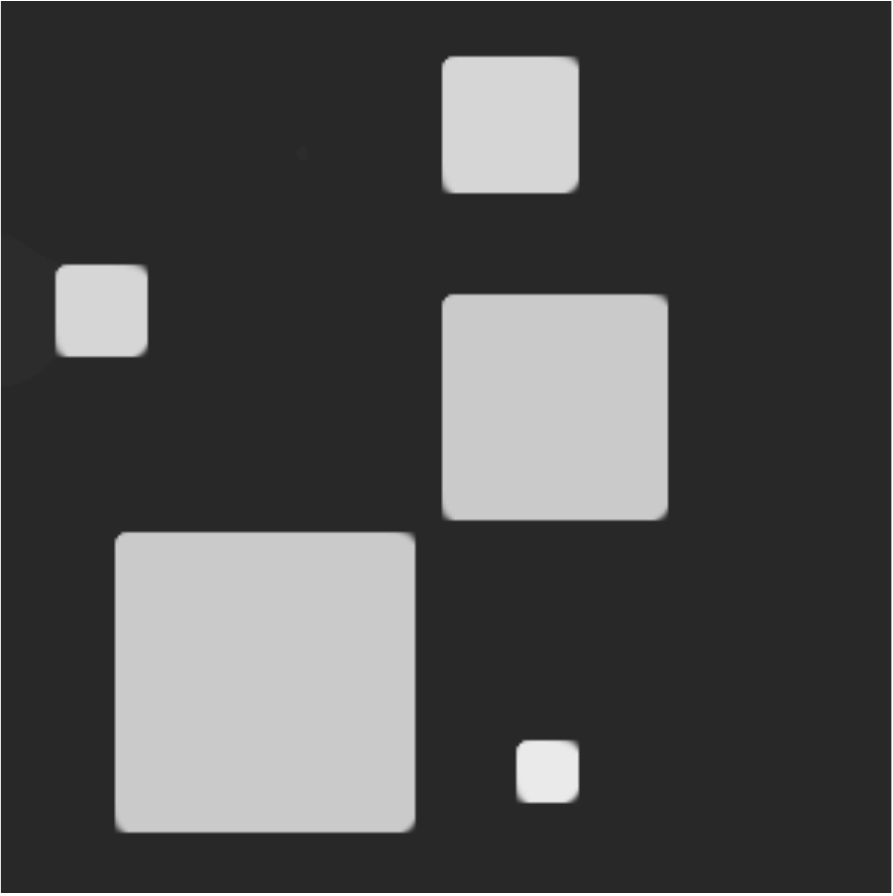}}\\
\subfloat[{ $c_{\Omega/B_1}{=}0.05$}]
{\includegraphics[width=0.25\textwidth]{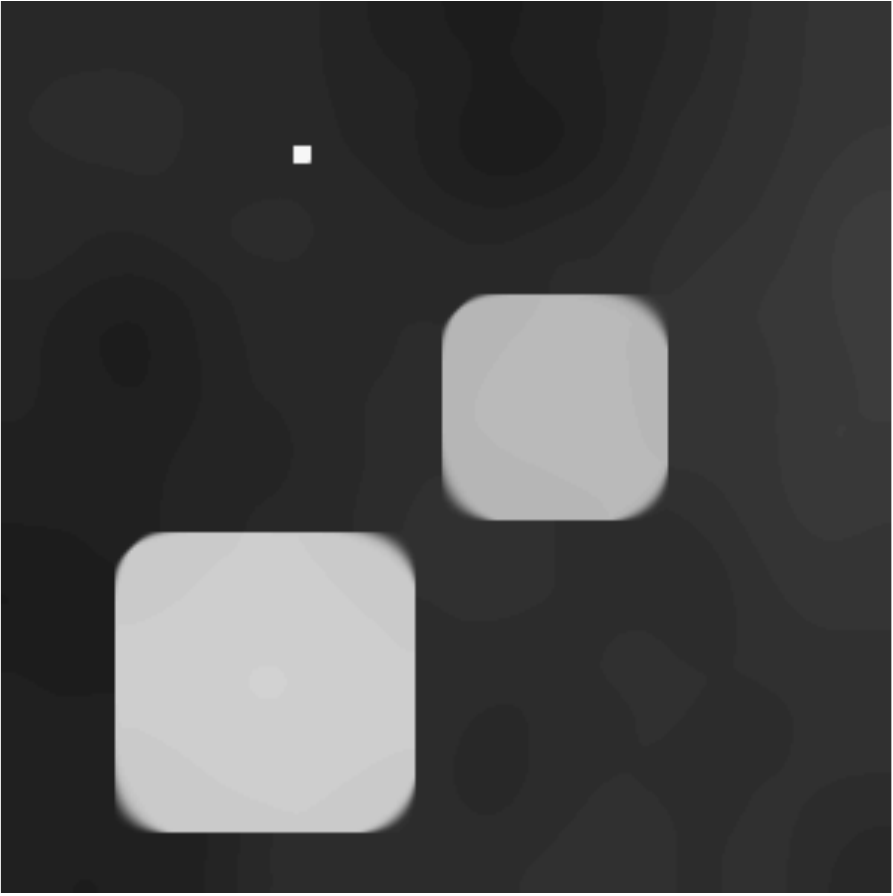}}\hfil
\subfloat[{ $c_{B_3}{=}0.2$}]
{\includegraphics[width=0.25\textwidth]{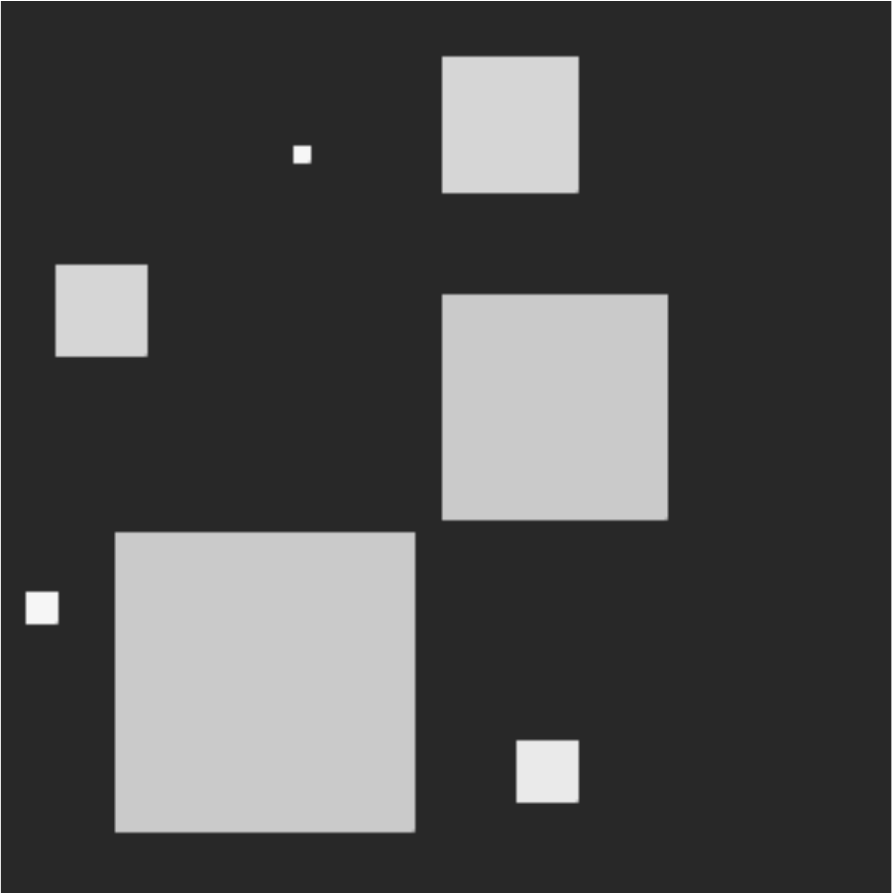}}\\
\caption[Adaptive regularization effects]{Scale space for uniform confidence. (b-e) structures suddenly disappear for different critical values of {$c$} based on their scale; (f) regularization with {$c=1.0$} for the region corresponding to the smallest box {$B_1$} (top-left) and {$c=0.05$} everywhere else; (g) regularization with {$c=0.2$} for the region corresponding to the third smallest box {$B_3$} (middle-right) and {$c=1.0$} everywhere else.} \label{fig:reg_uniform}
\end{figure}

\begin{figure}[th!]%
\centering
\subfloat[{ Original}]
{\includegraphics[width=0.3\textwidth]{orig-crop}}\hfil
\subfloat[{ Noisy Image}]
{\includegraphics[width=0.3\textwidth]{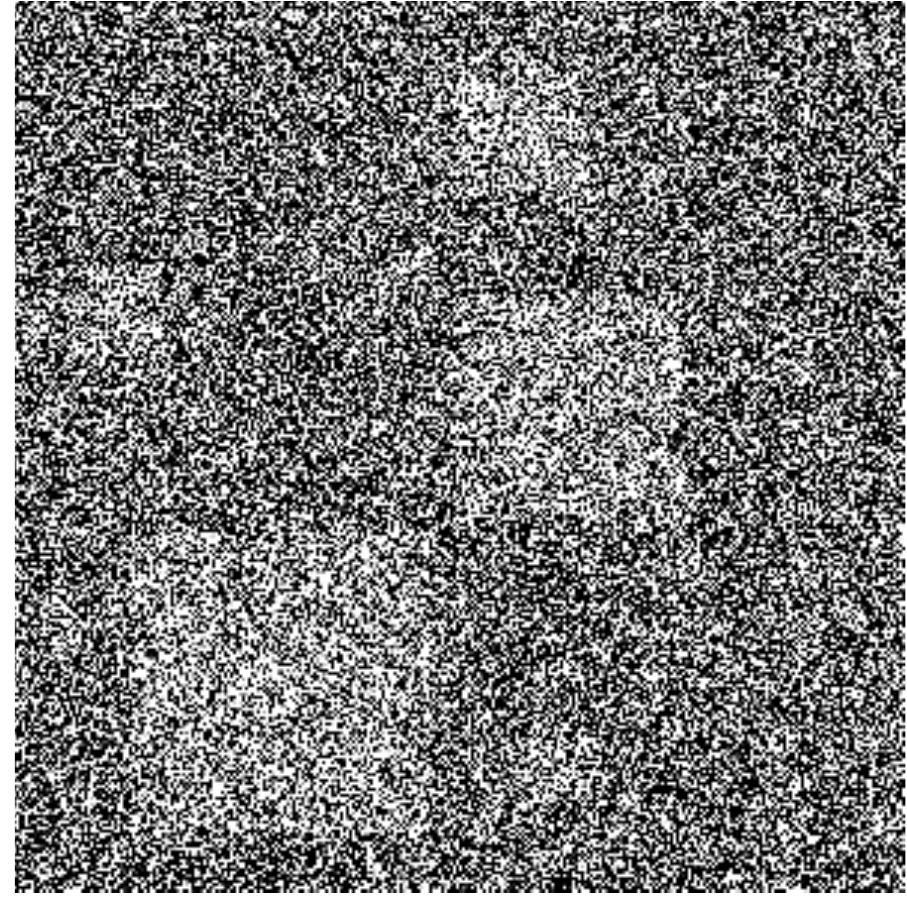}}\\
\subfloat[{ $c=1.0$}]
{\includegraphics[width=0.22\textwidth]{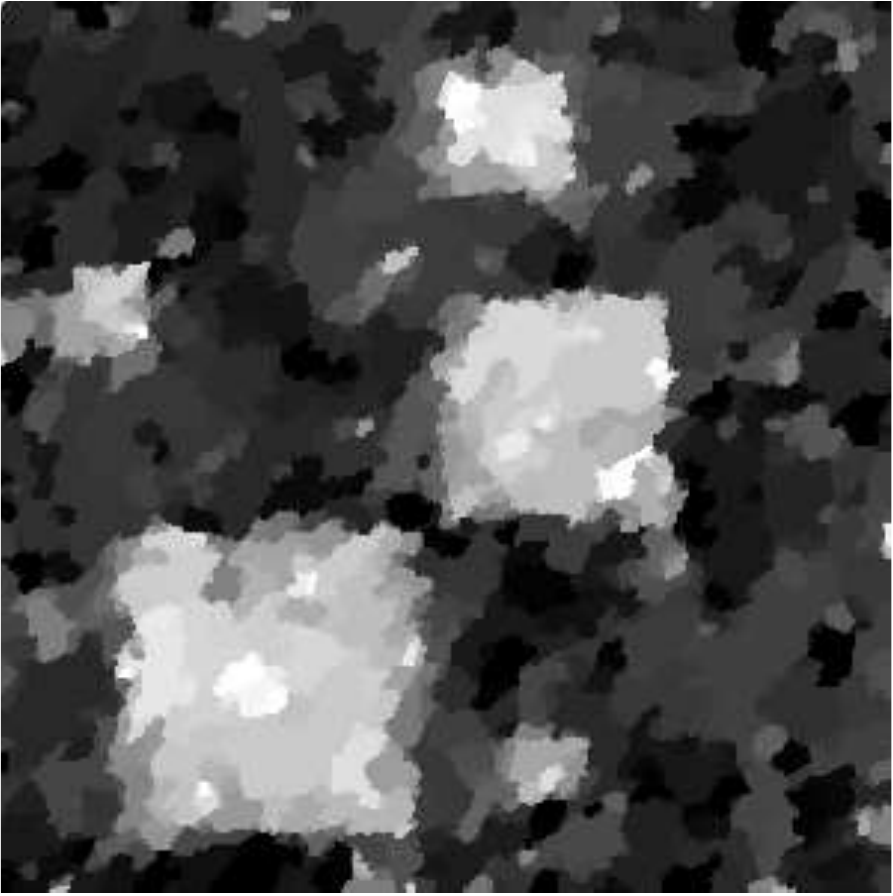}}\hfil
\subfloat[{ $c=0.7$}]
{\includegraphics[width=0.22\textwidth]{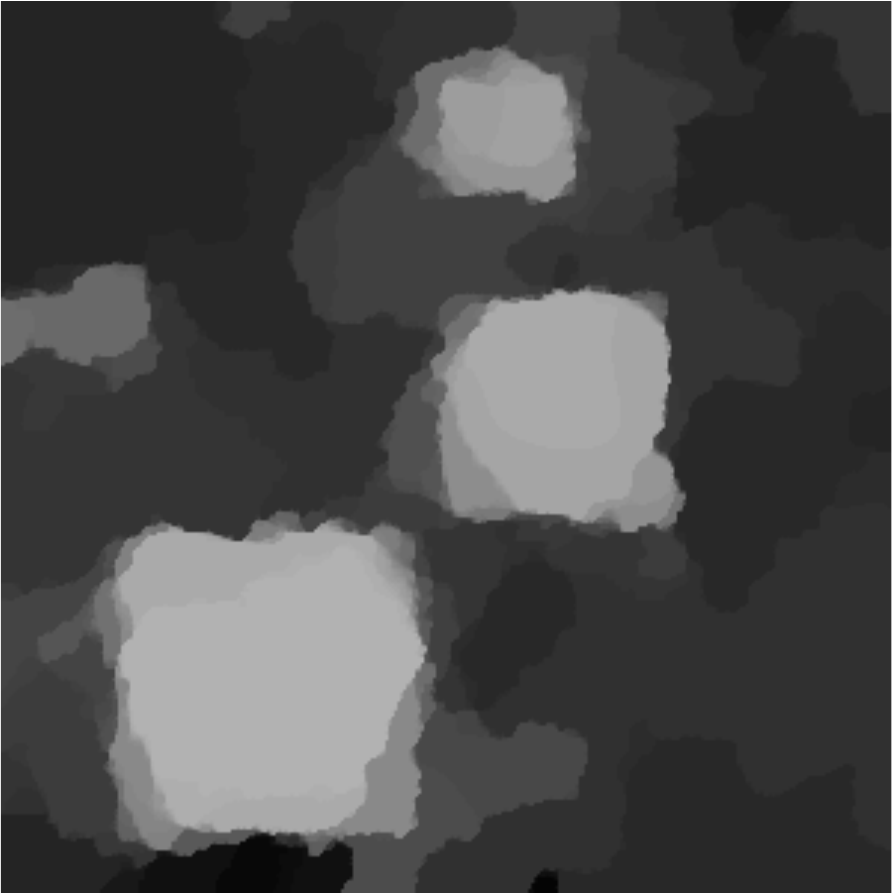}}\hfil
\subfloat[{ $c=0.5$}]
{\includegraphics[width=0.22\textwidth]{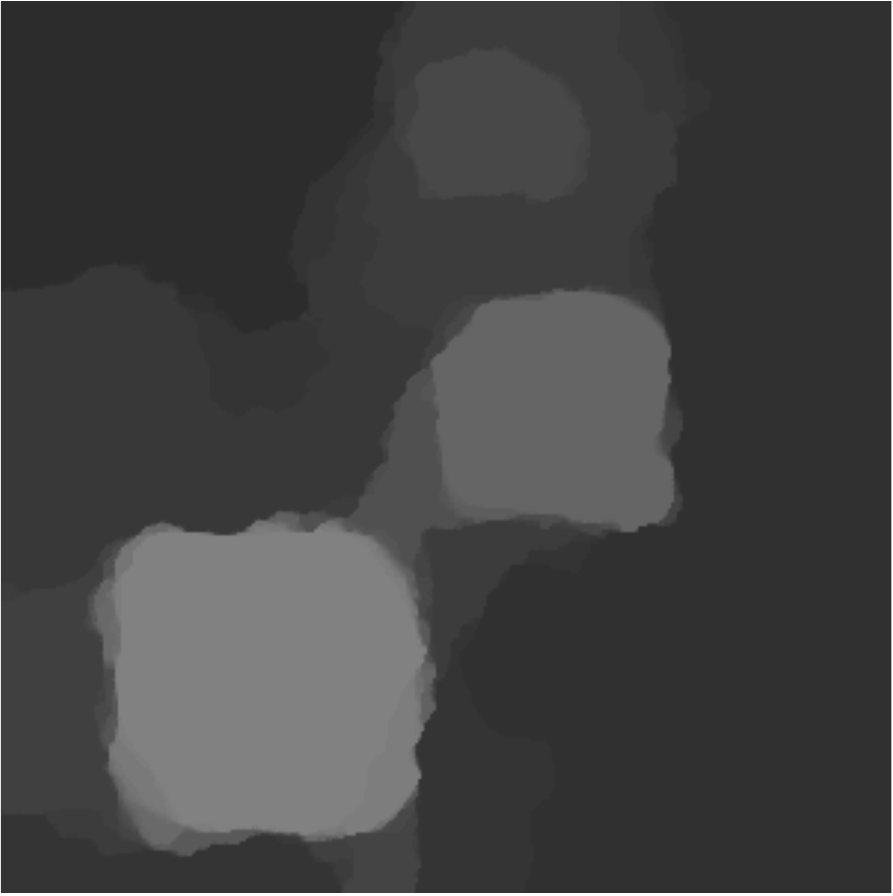}}\\
\subfloat[{ ${\pgb\pga=I}$}]
{\includegraphics[width=0.22\textwidth]{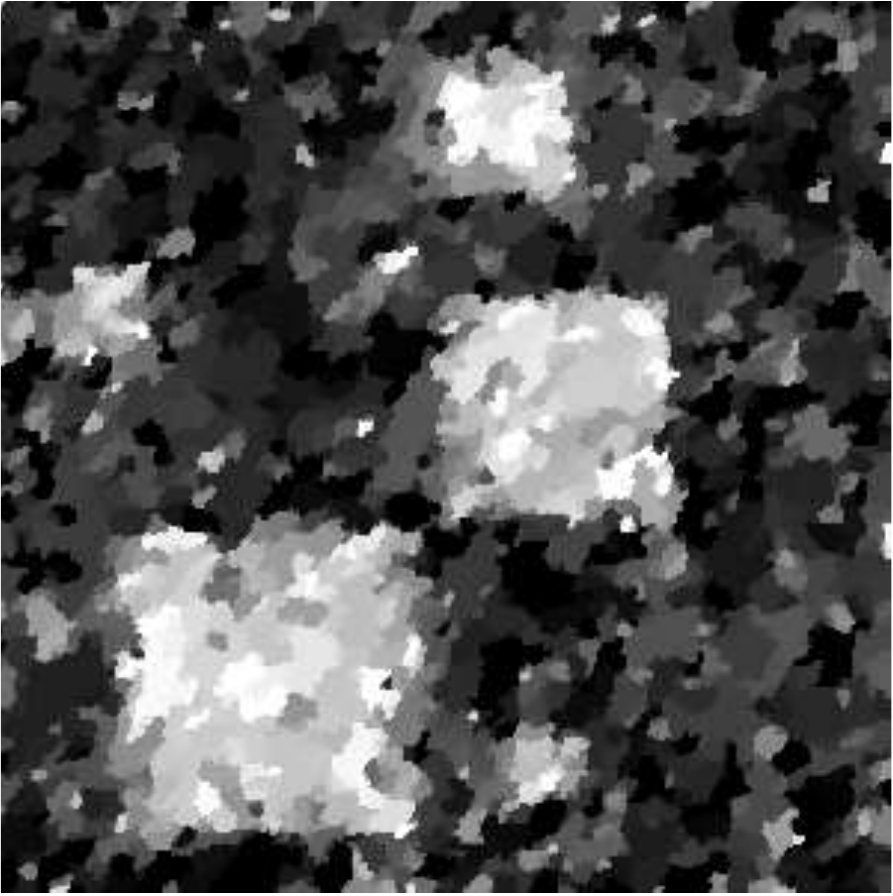}}\hfil
\subfloat[{ ${\pgb\pga=0.7\,I}$}]
{\includegraphics[width=0.22\textwidth]{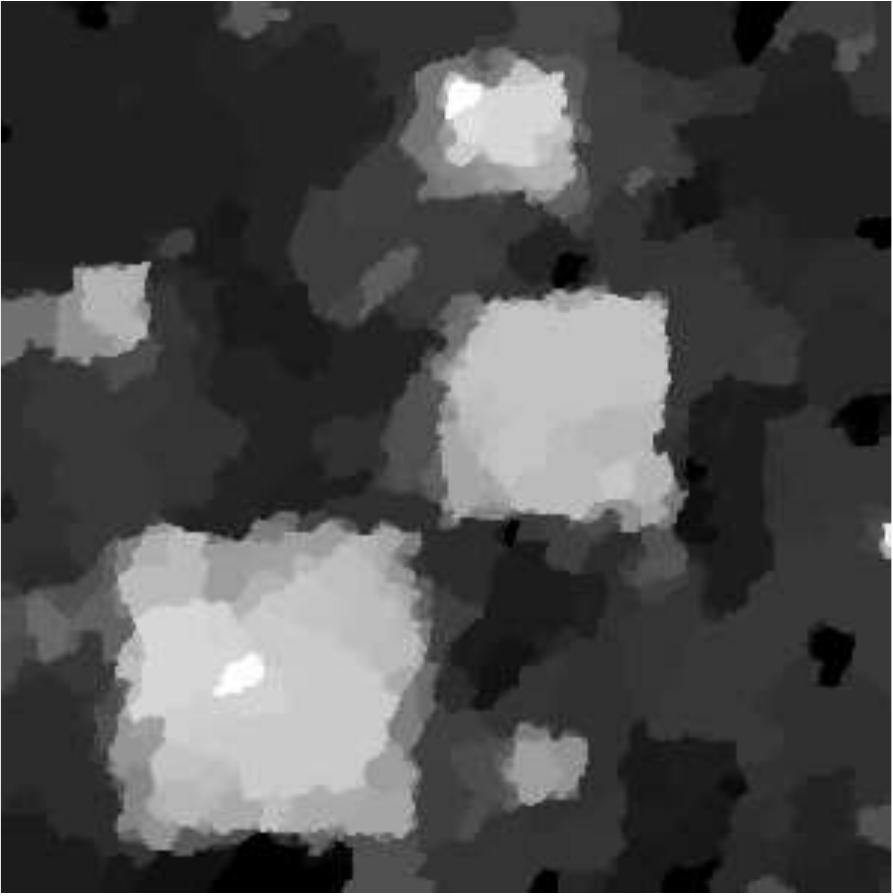}}\hfil
\subfloat[{ ${\pgb\pga=0.5\,I}$}]
{\includegraphics[width=0.22\textwidth]{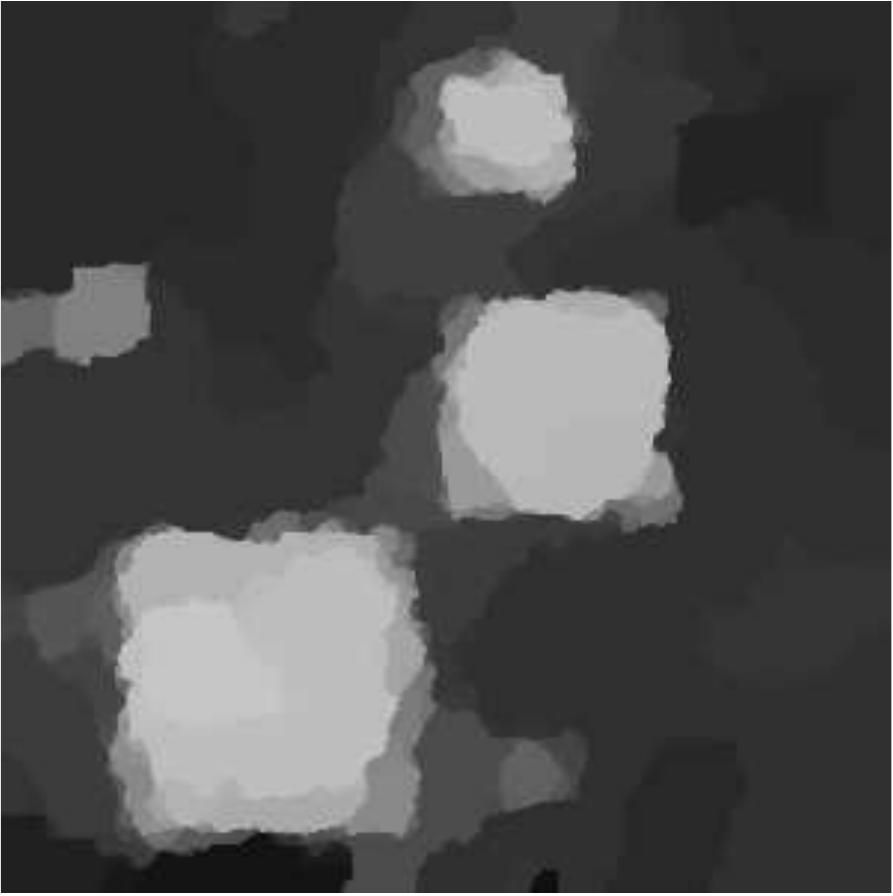}}\\
\caption[Adaptive fusion results for images degraded by point-wise Laplace noise]{Fusion results for images degraded by point-wise Laplace noise; (c-e) for uniform confidence values; (f-h) for estimated confidence values with different hyper-parameters.} \label{fig:reg_noisy}
\end{figure}

The results in Figure~\ref{fig:reg_uniform} show how  confidence values can affect  imaged objects, according to their scale. Let us name $B_1$ the smallest box  on the top-left and $B_2$ the third smallest box on the right. 
One can notice in panels (b)-(e) of Figure~\ref{fig:reg_uniform} that areas suddenly disappear as the uniform confidence value decreases, based on their size and regardless of their actual depth values. Notice in particular that both $B_1$ and $B_2$ disappear for decreasing values of $c$. 
The results in panels (f)-(g) of Figure~\ref{fig:reg_uniform} show the effects of the spatially adaptive regularization. 
Using spatially varying confidence values, the regularization is locally adapted resulting in smaller scale structures with high confidence values to survive excessive regularization, 
and, conversely, large scale structures with low confidence to disappear even when moderate regularization is applied. 
The results in Figure~\ref{fig:reg_noisy} show the difference between uniform and spatially adaptive confidence for depth fusion in the presence of Laplace noise.

The same considerations hold for higher order $\tgv$ regularization, with the only difference that signals of higher order piecewise smoothness (e.g. affine, quadratic etc.) are exactly modeled in this case. This alleviates the well known `stair-casing' effects of $\tv$ regularization. 

Summarizing, we see that the proposed model is very effective and versatile for the fusion of depth maps. In fact, it allows for a point-wise median-like estimation of the depth, while at the same time it ensures adaptive regularization according to confidence values which depend on the data. 

\subsection{Depth Image Fusion}
We consider $K$ cameras. 
Let $R_k$ be the orientation and $\mathbf{t}_k$ the position of the $k$-th camera with respect to a global reference frame, with $k\in\{1,\ldots,K\}$.
Then, each camera pose is represented by the homogeneous transformation
\begin{equation}
T_k=\begin{pmatrix}{R}_k&\mathbf{t}_k\\ 
			0&1
		\end{pmatrix}\in SE(3),\,k=\{1,\ldots, K\}.
\end{equation}

We consider that the scene is projected to the image plane according to the pinhole camera model. Thus,  a camera matrix defined as $\mP_{k} = \mA\mP T_{k}$, with $\mA$ the camera calibration matrix, and $\mP=[I_{3\times 3},0]$ the standard projection matrix,  corresponds to each camera pose. 

Let $\{\mP_{k}\}_{k=1}^{K}$ be a set of camera matrices and $(u,v)^{\top}=\vu\in\Omega\subseteq\mathbb{R}^{2}$ the spatial variable in the image domain. 
We denote the corresponding depth images as $(\dd_{1}, \ldots, \dd_{K})$, with $\dd_k : \Omega \mapsto (0,+\infty)$. 

Considering a reference camera $\mP_{r}$ 
we denote $\{\dd_{k}^{r}\}_{k=1}^{K}$ 
the depth images 
reprojected to the camera $\mP_{r}$. 
The reprojection process from camera $P_k$, $k=1,\ldots,K$ to the reference camera $P_r$ is defined as follows. Note first that back-projecting the depth map we obtain a 2.5D surface. This surface can be subsequently projected in the reference view, while the pointwise distance of the reference camera from the back-projected surface forms the reprojected depth map. Let $ R_{k}^{r} $ and $\mathbf{t}_{k}^{r}$ denote the relative rotation and translation of the frame $k$ to the frame $r$ respectively. 
Each point of the surface, expressed in the frame of the reference view, is given by the linear mapping:
\begin{equation}
\widetilde{X}(\vu) = \dd_{k}(\vu) R_k^{r} \mA^{-1} \begin{bmatrix}\vu\\1\end{bmatrix} + \mathbf{t}_k^{r}.
\end{equation}
These points are imaged in position $\vu' = \mA \widetilde{X}(\vu)$ on the image plane of the reference camera. 
Let $S_{\vu'}=\{ \vu\mid\mA \widetilde{X}(\vu)=\vu' \} $ and $e_3=(0,0,1)^{\top}$. The reprojected depth map is given by $\dd^{r}_k(\vu') = \underset{\vu \in S_{\vu'}}{\min}e_3^{\top}\widetilde{X}(\vu)$.

As a result, at each position of the reference depth image we have up to $K$ depth observations. 
The fusion process combines these depth observations, taking into account corresponding confidence values, in order to produce a more accurate estimation of the real depth values.

\subsubsection{Heuristic Confidence Estimation} \label{ssec:heur}
We discuss here possible confidence measures for the case of depth image fusion. These heuristic confidence values can be used both as baseline methods for comparison with our complete model, as well as to compute the hyper-parameters  $\pga$ and $\pgb$ of our model. The heuristic confidence measures discussed here are based on the structure and the appearance of the scene. 

First, we consider the geometry of the scene. The depth confidence at an image position $\mathbf{u}$ depends on the angle between the viewing ray, given by $\mathbf{r}(\mathbf{u})=\mA^{-1}\mathbf{u}/\|\mA^{-1}\mathbf{u}\|$, and the normal of the surface back-projected at $\mathbf{u}$. 

Letting $ \nm:\Omega \mapsto S^{2}$ the normal map corresponding to depth image $\dd$, with $S^{2}$ the unit sphere embedded in $\mathbb{R}^{3}$, the confidence values are given by
\begin{equation}
	(\vl)_{\vu,\vu} =  \nm(\vu)\cdot\mathbf{r}(\vu)
\end{equation}

Denoting $\proj_{S^2}(\cdot)$  the projection operator on the unit sphere $S^2$ and $\mathrm{D}_{u}^{+},\mathrm{D}_{v}^{+}$ the forward differences with respect to directions $u$ and $v$, the normal map can be estimated as:
\begin{equation}
	\nm(u,v) = \proj_{S^2} \left(\mathrm{D}_{u}^{+}\left(\dd\right)\times \mathrm{D}_{v}^{+}\left(\dd\right)\right).
\end{equation}

The second heuristic confidence is based on the appearance of the scene and it is based on the observation that image edges often correspond to occlusions and thus depth discontinuities. This suggests a weighting scheme which gives higher confidence on the regions around the edges in order to maintain clean edges.

For simplicity we consider here a linear weighting based on the gradient of the intensity image $I$, namely 
\begin{equation}\label{eq:edgesconf}
	\vl = \alpha \|G_{\sigma} \ast \nabla I\|^{\beta},
\end{equation}
with $\alpha,\beta$ parameters suitably shaping the confidence values, $G_{\sigma}$ a Gaussian filter with standard deviation $\sigma$ and $\ast$ the convolution operator. The Gaussian filter is useful to control the width of the affected region around the image edges. 
A similar  weighting scheme has been proposed in \cite{Newcombe2011b}, though the weights were applied, via an exponential function, to the regularization term rather than to the fidelity term. Another related weighting measure based on the Nahel-Enkelmann operator, also applied on the regularization term, was proposed in \cite{Kuschk2013}, which also uses the images of the scene. 

We note that the geometric confidence 
tends to assign low confidence values to regions which are orthogonal to the view direction, which often correspond to regions near the edges. The two approaches can be combined to estimate confidence values with desired properties.

\subsubsection{Synthetic dataset}
We performed an extensive evaluation of the proposed model for the fusion of depth images using synthetic data. We have considered two different classes of 3D models: 1) ordinary small to medium scale objects and 2) models of urban landscapes and buildings. For the objects dataset, we considered the models Bunny, Dragon, Happy Buddha, and Armadillo  from the Stanford 3D scanning repository \cite{Turk94,Curless96,Krishnamurthy96} and the objects Chef, Chicken, Parasaurolophus and T-rex from \cite{Mian2006}. The urban landscapes dataset contains four models taken from the Sketch-up 3D warehouse. 

The two datasets have different characteristics. More specifically, the small and medium scale objects are made by higher-order polynomial terms due to the varying curvature of their surface, while the resulting depth images contain only a small amount of sharp discontinuities. On the other hand, urban landscapes are typically described by lower-order polynomial terms while the resulting depth images contain a significant amount of sharp discontinuities. The motion of the camera also differs (orbiting vs pure translation motion respectively), which affects the occluded regions of the depth images.

\paragraph*{{\bf Objects}}
We compute depth images corresponding to each of the objects by considering a virtual camera with parameters $(f,c_u,c_v)=(576,320,240)\,[px]$ that orbits around the object at a distance of $ 3\,[m.u.]$ (model units). Depth images are generated every $2\pi/72$ rads. The depth images are generated using \cite{Guney2015}. Knowing the exact parameters of the camera the reprojection process produces depth images with correct point-wise correspondences of the depth values, resulting to a fusion problem with perfect data association. 

We consider two sets of metrics, the first based on the depth image and the other on the corresponding disparity image. For the disparity image we use the average error in all the valid pixels (avg-all), and the percentage of pixels with disparity error greater than $n$ (out-$n$) \cite{Geiger2012}. For the depth image evaluation we use the standard root mean square error (RMSE), the mean absolute error (ZMAE), and the mean angular error of the norms (NMAE) \cite{Barron2015}. The average values reported for the synthetic datasets are geometric average values. 
For the objects dataset the disparity image is generated by considering a virtual baseline with length equal to half the distance between successive views ($3\sin\frac{\pi}{72}$).

First, we explore the relation of the fused depth image accuracy with the type and the strength of noise for different versions and ablations of our model considering the minimization algorithms discussed in Section~\ref{sec:algs}. Naturally, noise is added to the original depth images before the reprojection process. The first two columns of Figure~\ref{fig:objgraphs} show the results for Laplace and normally distributed noise with $b,\sigma\in{[0,1]}$ respectively, using $11$ successive depth images and Table~\ref{tab:resobj} shows the error values, for the case of Laplace noise with $b=0.6\,[m.u.]$ (model units). The abbreviations of the different versions of the proposed method are described in Table~\ref{tab:names}.

\begin{figure*}[!t]%
\centering
\includegraphics[width=0.2\textwidth]{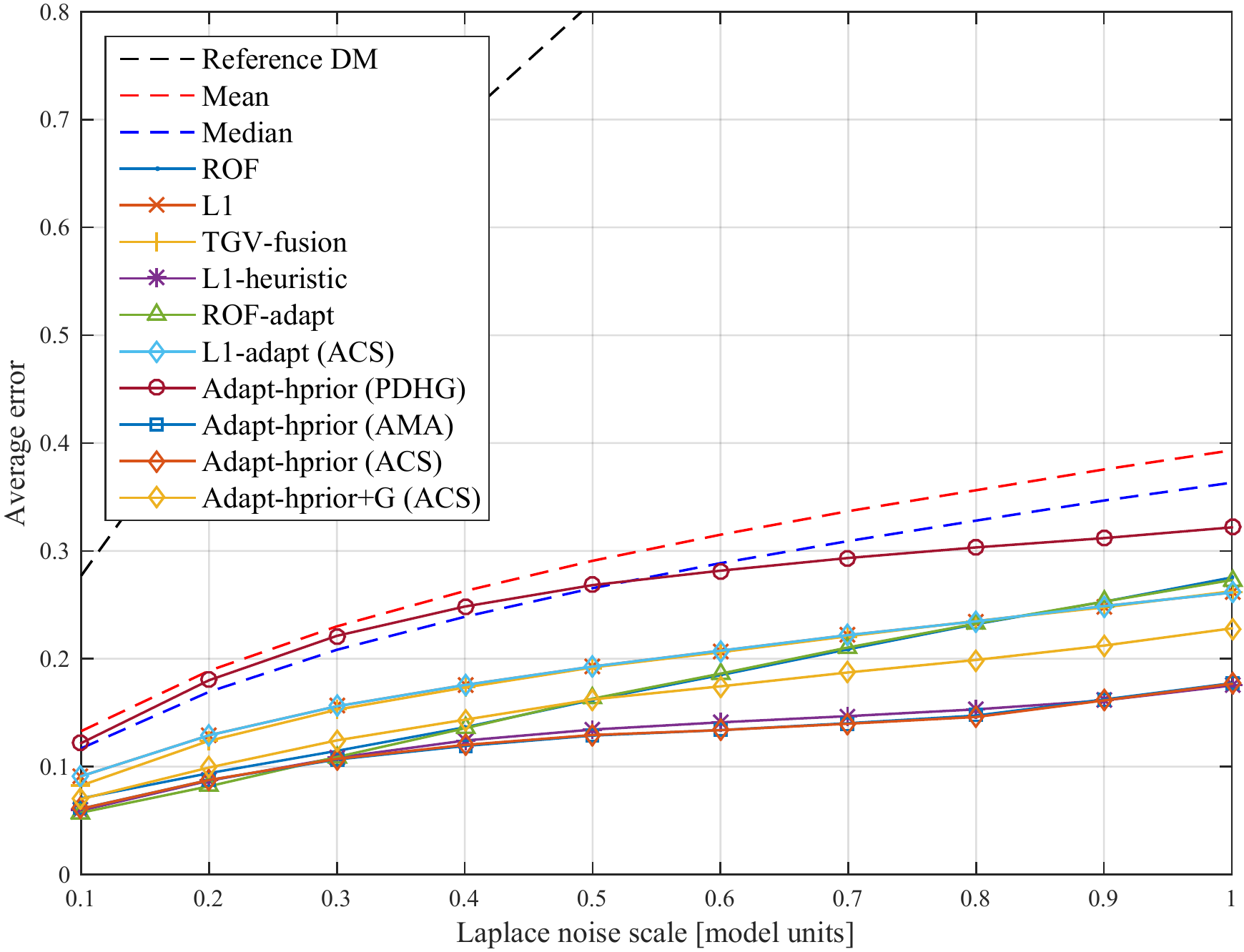}%
\hfil
\includegraphics[width=0.2\textwidth]{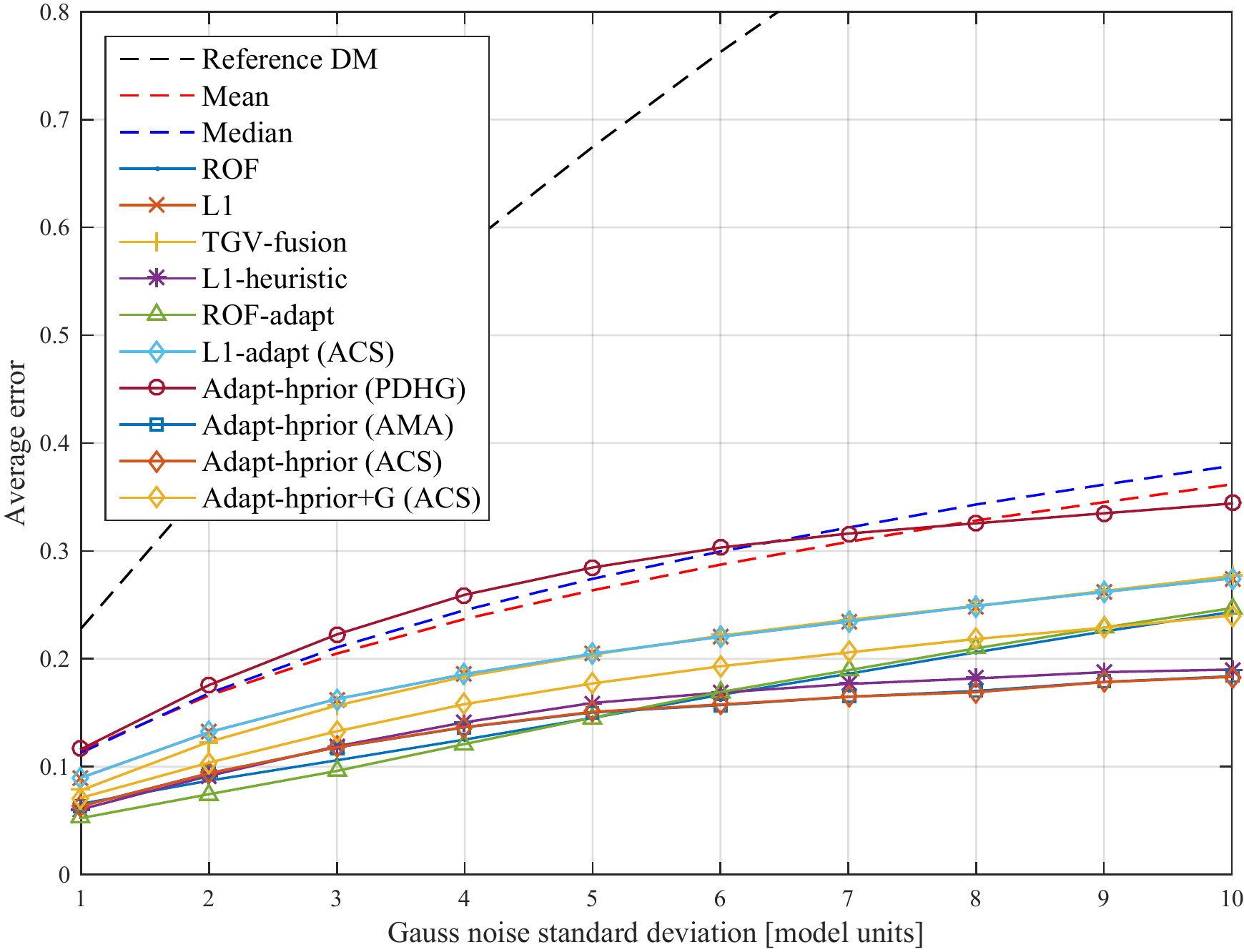}%
\hfil
\includegraphics[width=0.2\textwidth]{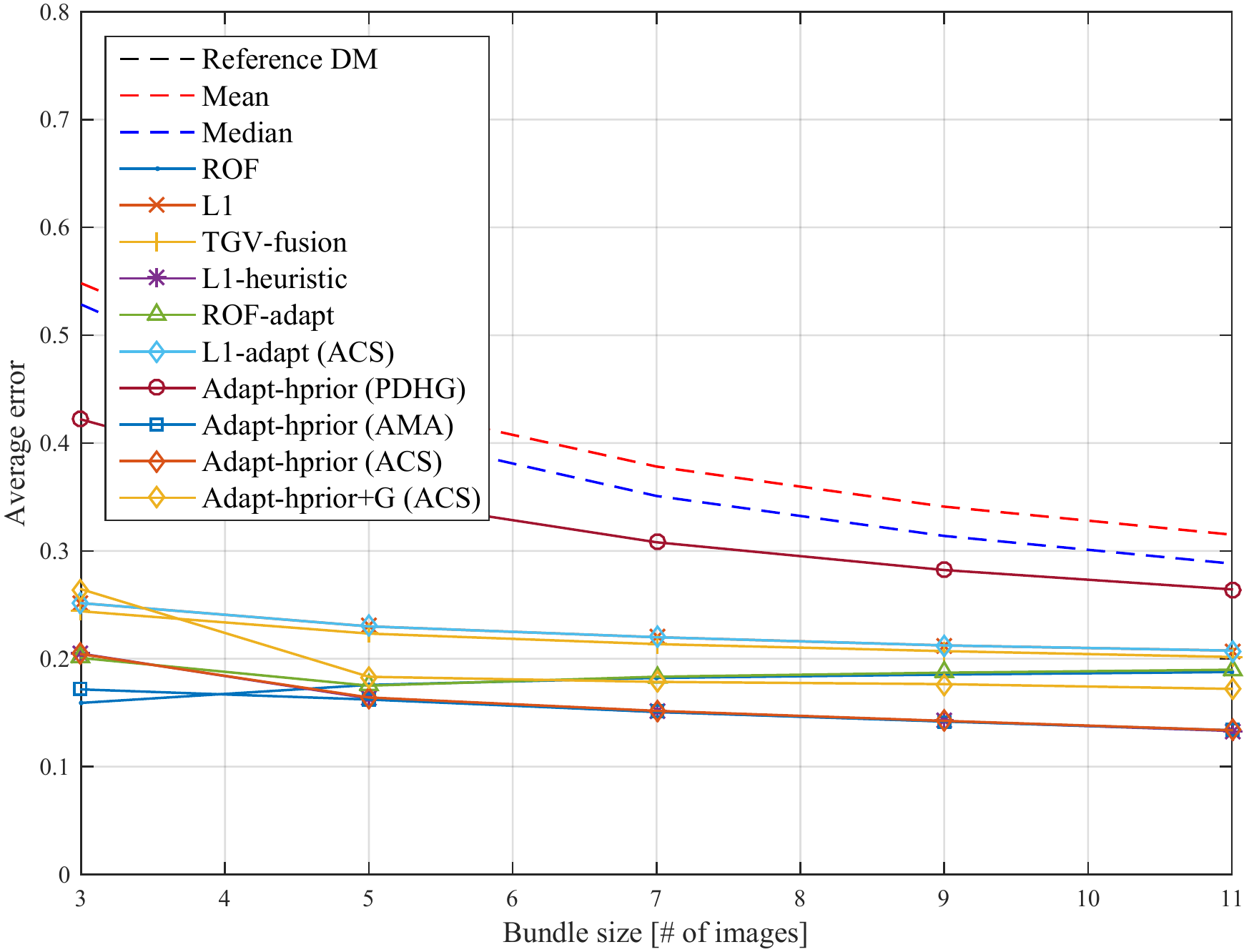}%
\hfil
\includegraphics[width=0.2\textwidth]{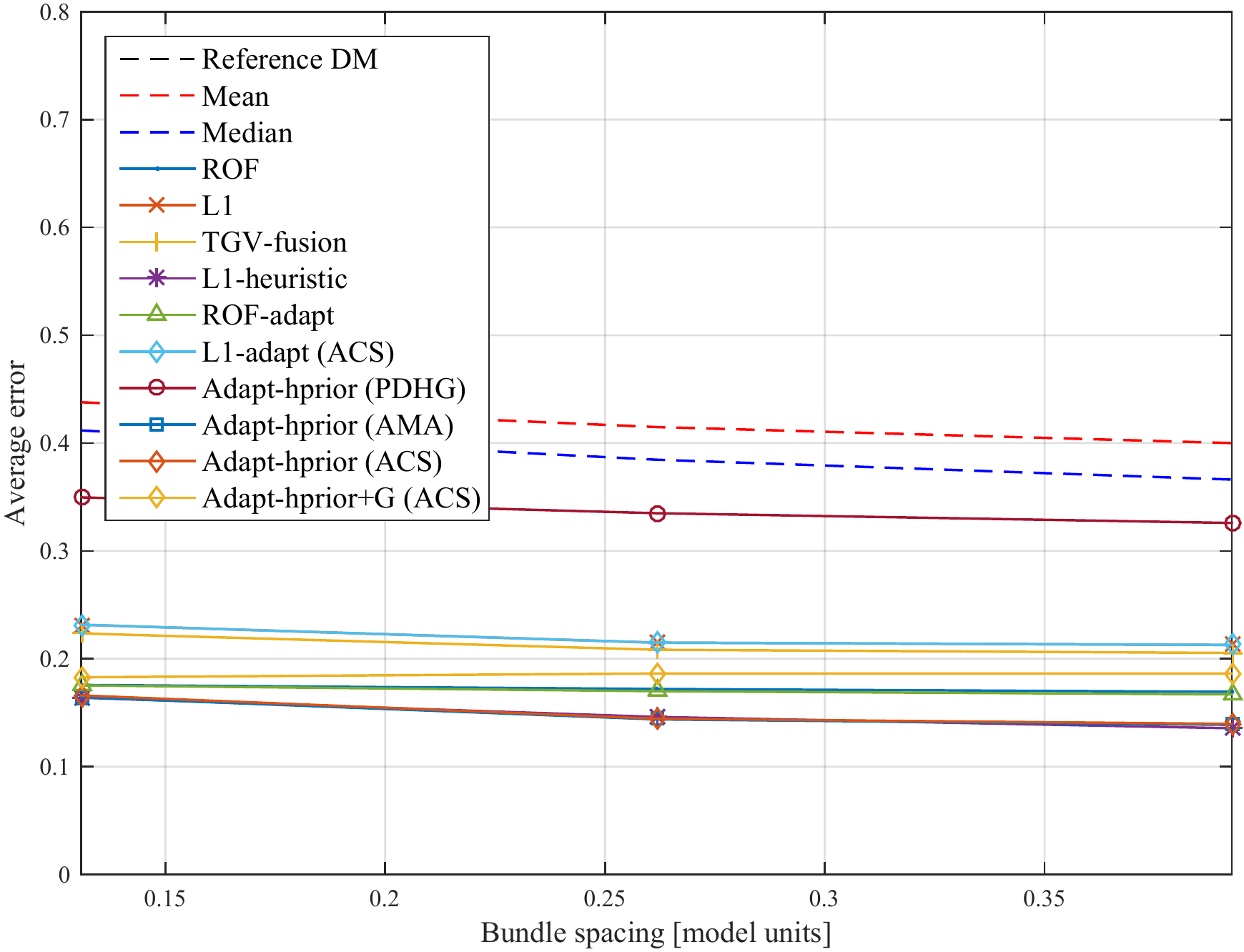}\\%
\includegraphics[width=0.2\textwidth]{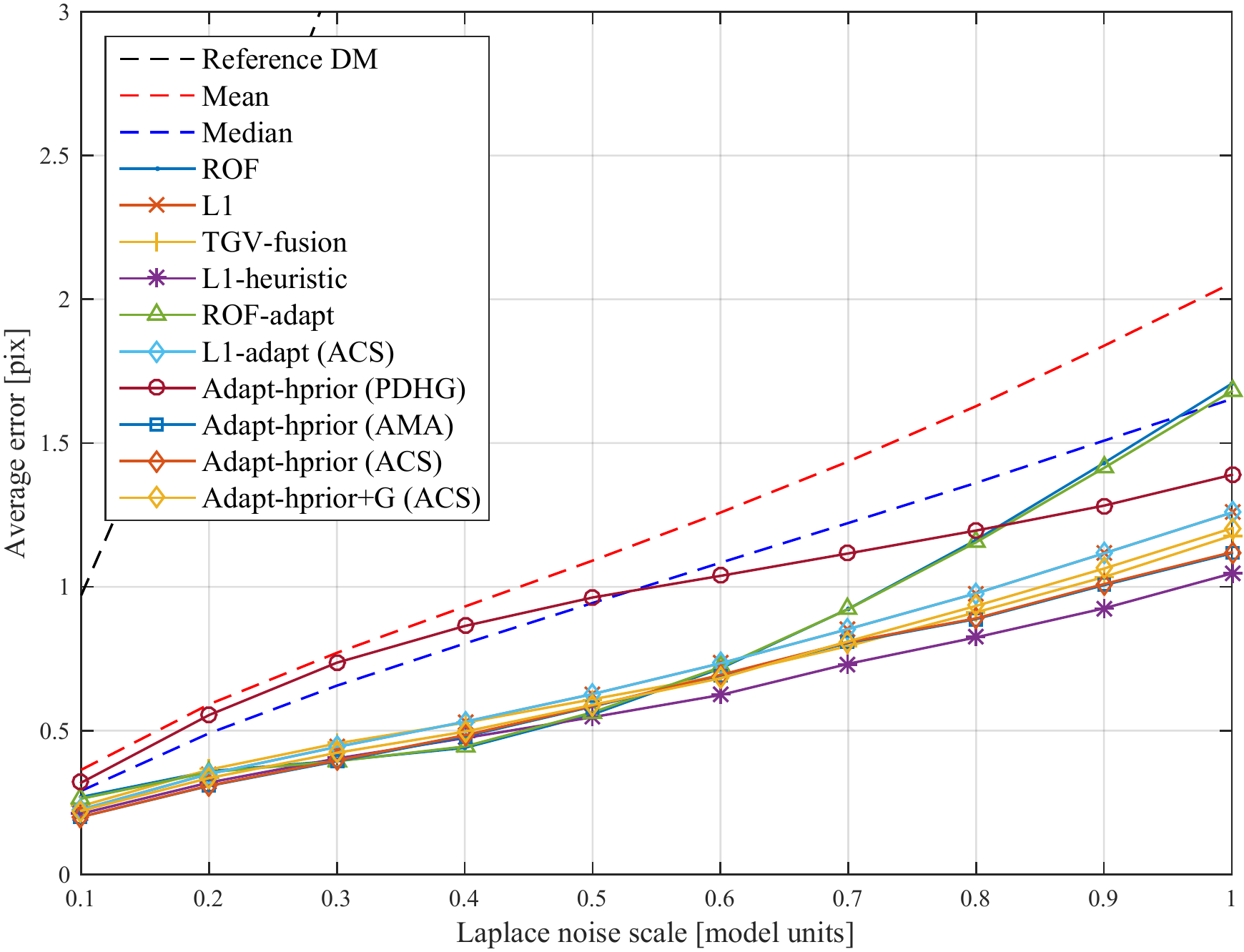}%
\hfil
\includegraphics[width=0.2\textwidth]{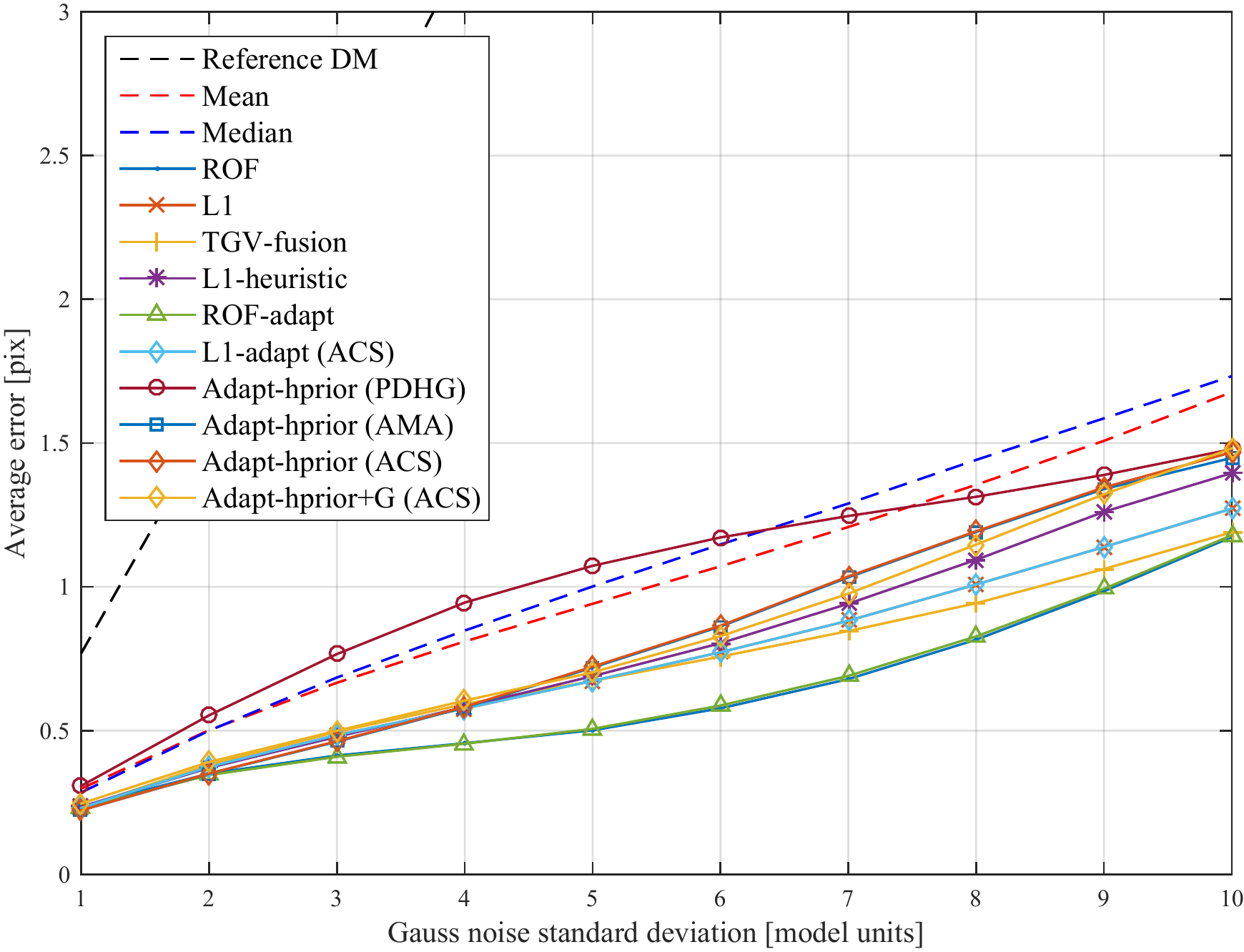}%
\hfil
\includegraphics[width=0.2\textwidth]{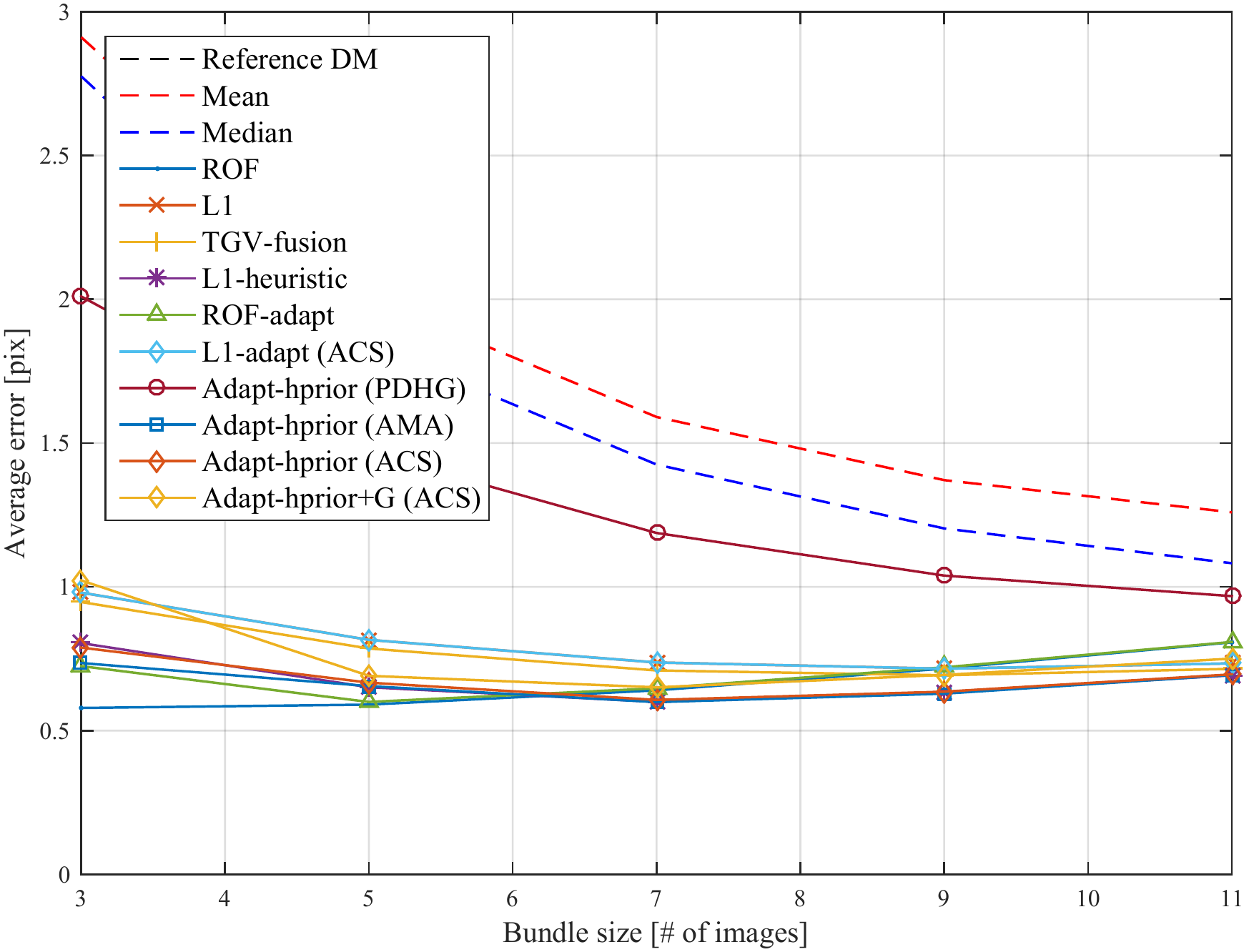}%
\hfil
\includegraphics[width=0.2\textwidth]{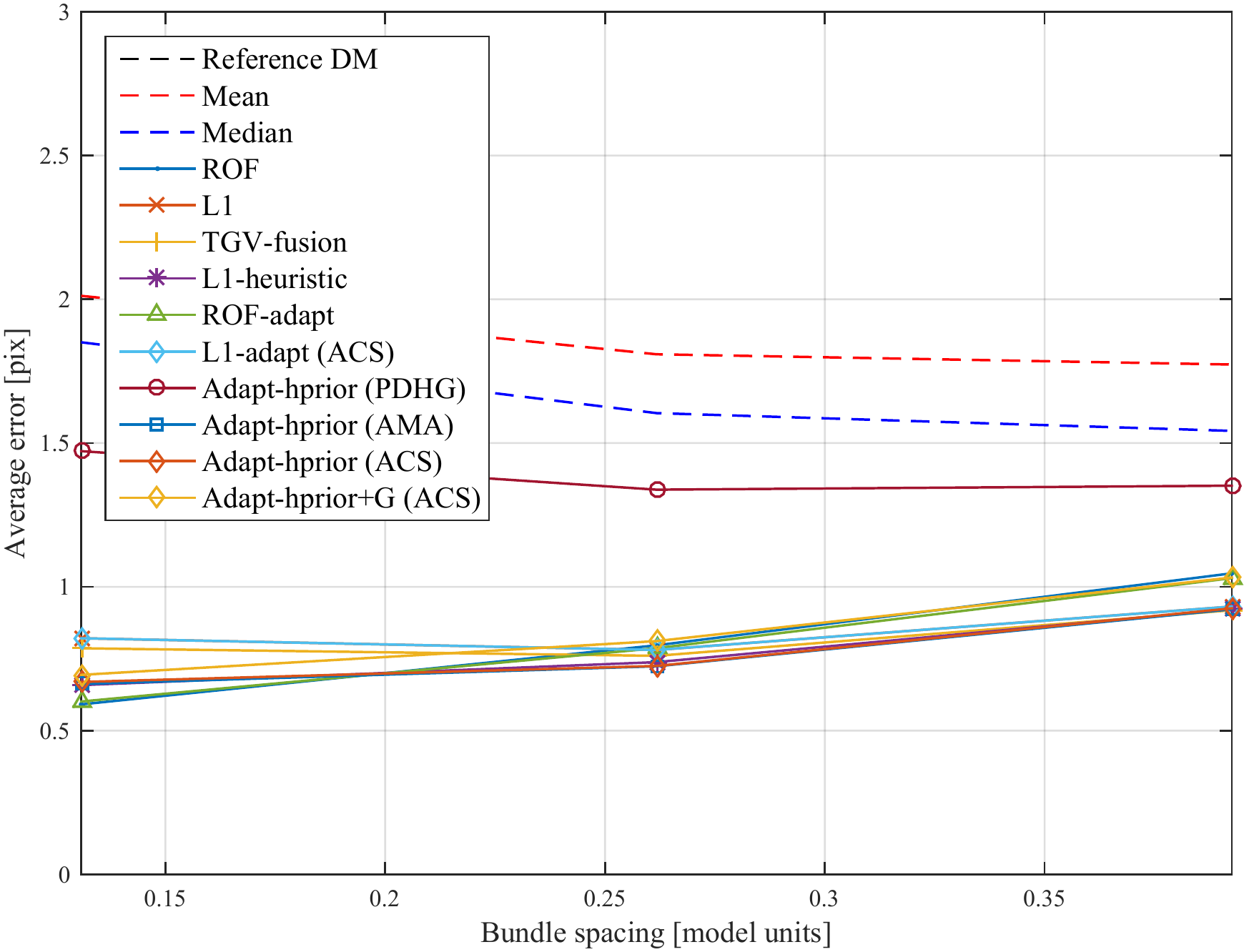}%
\caption{Average depth (top) and disparity (bottom) error in relation to (from left to right): a) Laplace noise scale; b) standard deviation of the Gauss noise; c) number of fused depth images; d) distance between the fused depth images. Laplace noise scale $b=0.6\,[m.u.]$.} \label{fig:objgraphs}
\end{figure*}

\begin{table}[th!]                                                                         
\centering        
\caption{Method names}                                                                  
\label{tab:names}                                                                             
\begin{tabular}{|l|l|}
\hline
L1-heuristic & confidence based on the scene geometry\\
ROF-adapt & $L_2$ fidelity with adaptive confidence values \\
L1-adapt & $L_1$ fidelity with adaptive confidence values \\
Adapt-hprior & L1-adapt with scene geometry based prior \\ 
Adapt-hprior+G & as Adapt-hprior plus appearance prior \\
\hline
\end{tabular}
\end{table}

\begin{table}[th!]                                                                         
\centering               
\caption{Results for the objects dataset for different versions and ablations of the proposed model for Laplace noise of scale $b=0.6\,[m.u.]$, and $11$ fused depth images. (Best values in bold)
}                                                                  
\label{tab:resobj}                                                                          
\begin{tabular}{c|cccc|cc}                                                        
& RMSE & ZMAE & NMAE & Z-avg  & out-3 [\%]& D-avg [px]\\ 
\hline                                                                                    
Reference DM  & 0.8152 & 0.5861 & 1.5464 & 0.9040 & 58.5764 & 7.1851 \\         
Mean  & 0.1631 & 0.1265 & 1.5156 & 0.3150 & 6.2219 & 1.2583 \\                 
Median & 0.1449 & 0.1100 & 1.5081 & 0.2886 & 3.9030 & 1.0839 \\               
\hline                                                                                   
ROF & 0.0872 & 0.0632 & 1.1461 & 0.1848 & 0.4173 & 0.7162 \\                  
L1  & 0.0943 & 0.0689 & 1.3754 & 0.2075 & 0.6818 & 0.7340 \\                   
TGV-fusion \cite{Pock2011}  & 0.0943 & 0.0685 & 1.3531 & 0.2061 & 0.7312 & 0.6912 \\           
L1-heuristic & 0.0831 & 0.0572 & 0.5893 & 0.1410 & 0.2029 & \textbf{0.6238} \\         
\hline                                                                                   
ROF-adapt& 0.0883 & 0.0643 & 1.1357 & 0.1862 & 0.4236 & 0.7204 \\            
L1-adapt (PDHG) & 0.1272 & 0.0970 & 1.4892 & 0.2639 & 2.1541 & 0.9467 \\      
L1-adapt (AMA) & 0.0943 & 0.0688 & 1.3756 & 0.2075 & 0.6807 & 0.7339 \\       
L1-adapt (ACS) & 0.0943 & 0.0689 & 1.3754 & 0.2075 & 0.6818 & 0.7340 \\       
Adapt-hprior (PDHG) & 0.1392 & 0.1068 & 1.5026 & 0.2817 & 3.1713 & 1.0383 \\  
Adapt-hprior (AMA) & \textbf{0.0776} & \textbf{0.0537} & 0.5754 & \textbf{0.1338} & 0.1153 & 0.6914 \\   
Adapt-hprior (ACS) & 0.0778 & 0.0539 & \textbf{0.5719} & \textbf{0.1338} & \textbf{0} & 0.6938 \\   
Adapt-hprior+G (PDHG) & 0.1626 & 0.1235 & 1.5137 & 0.3121 & 5.8530 & 1.1945 \\
Adapt-hprior+G (AMA) & 0.1205 & 0.0629 & 0.6713 & 0.1720 & 1.9409 & 0.6675 \\  
Adapt-hprior+G (ACS) & 0.1243 & 0.0645 & 0.6612 & 0.1744 & 2.2225 & 0.6821 \\  
\end{tabular}                                                       
\end{table}

We observe that in the case of joint depth and confidence (biconvex) estimation problem, ACS and AMA algorithms give equivalent results in practice, with ACS marginally better in average. For this dataset, PDHG algorithm gives results with errors close to the median and average baselines, as it does not converge numerically. This is possibly caused by the high signal to noise ratio (SNR) in the images of this dataset.
We also observe that the {\em L1-heuristic} version of the model provides better results with respect to the  {\em L1-adapt} version, and almost as good as the other two adaptive versions. This is indicative of the scene geometry confidence values effectiveness. 

Additionally, the  {\em Adapt-hprior} version  performs better than the extended  {\em Adapth-hprior+G} version. The reason for this is that lower regularization is applied near the image edges, hence noise is not suppressed in these areas. 
Finally, we see that all the adaptive versions with heuristic priors, as well as the {\em L1-heuristic} version perform better than the TGV-Fusion method \cite{Pock2011}, while {\em L1-adapt} gives similar results.

A visual comparison of the results is presented in Figures~\ref{fig:visresu1} and \ref{fig:visresu2}. For this example it is evident that only the  {\em L1-heurisitic} and  {\em Adapt-hprior} give results which are smooth on one hand and close to the ground truth on the other.  {\em Adapt-hprior} actually is more faithful in terms of shape as the numerical results suggest.  In all the other cases residual high frequency noise can be observed on the surface. This is mainly due to the very low SNR of the original depth images. The  {\em L1-adapt} and  {\em TGV-fusion} methods still are able to capture the shape of the surface, however the reconstructed surface is not smooth.

We examine also the relation of accuracy of the fused depth image with the bundle size and the spacing between the original depth images. The results are presented in the last two columns of Figure~\ref{fig:objgraphs}. 
In general one would expect that more data layers would produce more accurate results. This is confirmed up to a certain point for the disparity error, while for larger bundles the errors increase. This is attributed to the increase of errors in occluded regions resulting by the reprojection of distant depth images. This also evident in the disparity error results in the last column of Figure~\ref{fig:objgraphs}. 
Hence, more depth images are useful for decreasing the error as long as they are close to the reference view point, while more distant images tend to introduce errors as scenes are not consistent any more in the occluded regions. 
Average depth error slightly improves in both these cases instead. A closer examination reveals that the actual depth error increases, while the normal estimation error decreases and this positively affects the average. The decrease in normal errors is reasonable since the scene is captured from a wider view-point range hence their estimation is more robust. These observations can be used to determine the best size of the bundle based on the camera motion, however we will not treat this problem here as it is outside the scope of this work.

\paragraph*{{\bf Urban Landscapes}}
We performed the same set of experiments for the urban landscapes dataset. The intrinsic parameters of the virtual camera are the same, however the camera here follows a purely translational path, facing always the scene from above. The distance of the camera from the zero level of the scene is taken equal to $300\,[m.u.]$, and depth images are generated every $4\,[m.u.]$ forming a bundle of $11$ images. 

The first two columns of Figure~\ref{fig:urbgraphs} show the effect of Laplace and normally distributed noise on the final results, 
while Table~\ref{tab:resurb} shows the actual error values for the case of Laplace noise with $b=6\,[m.u.]$.

\begin{figure*}[!t]%
\centering
\includegraphics[width=0.20\textwidth]{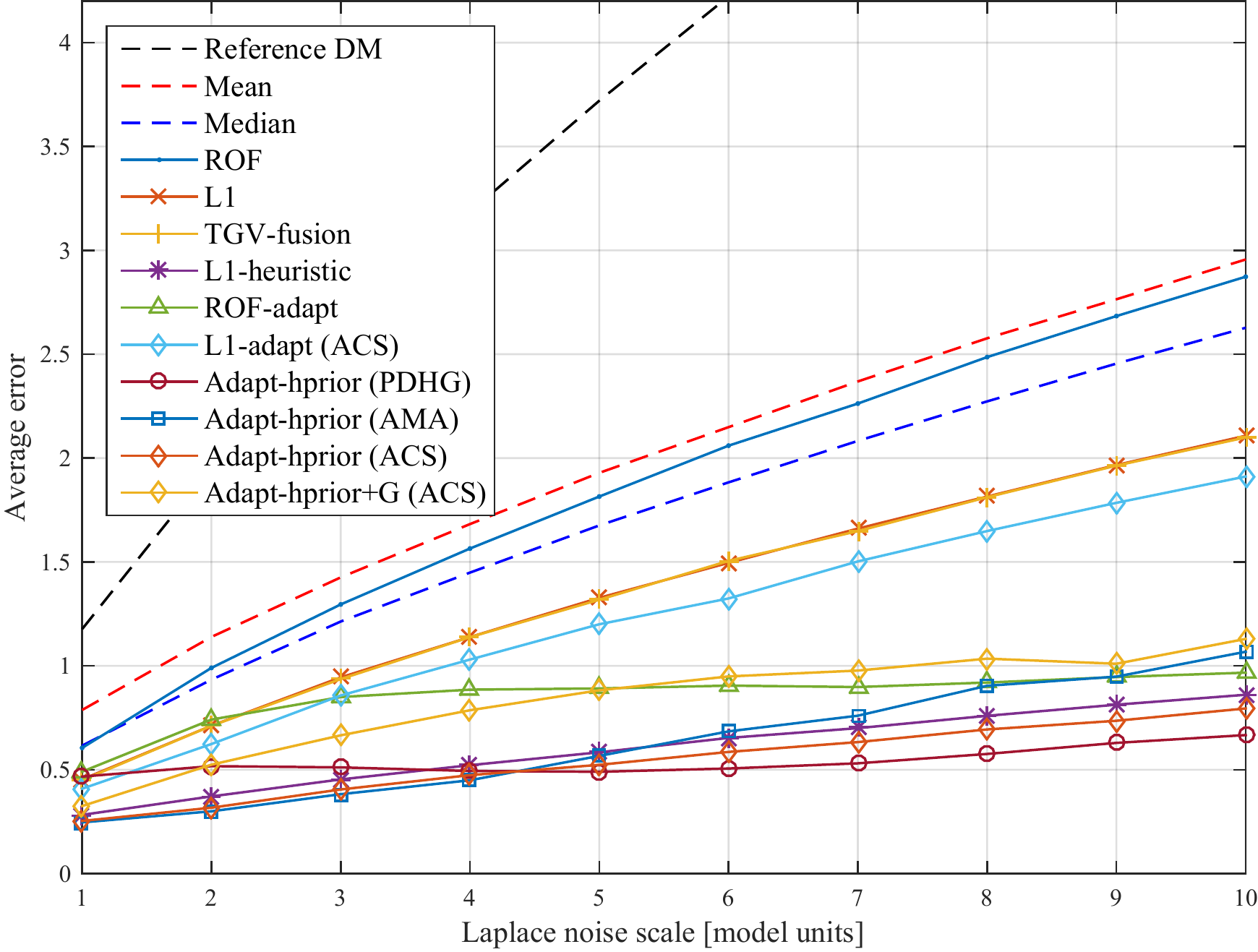}%
\hfil
\includegraphics[width=0.2\textwidth]{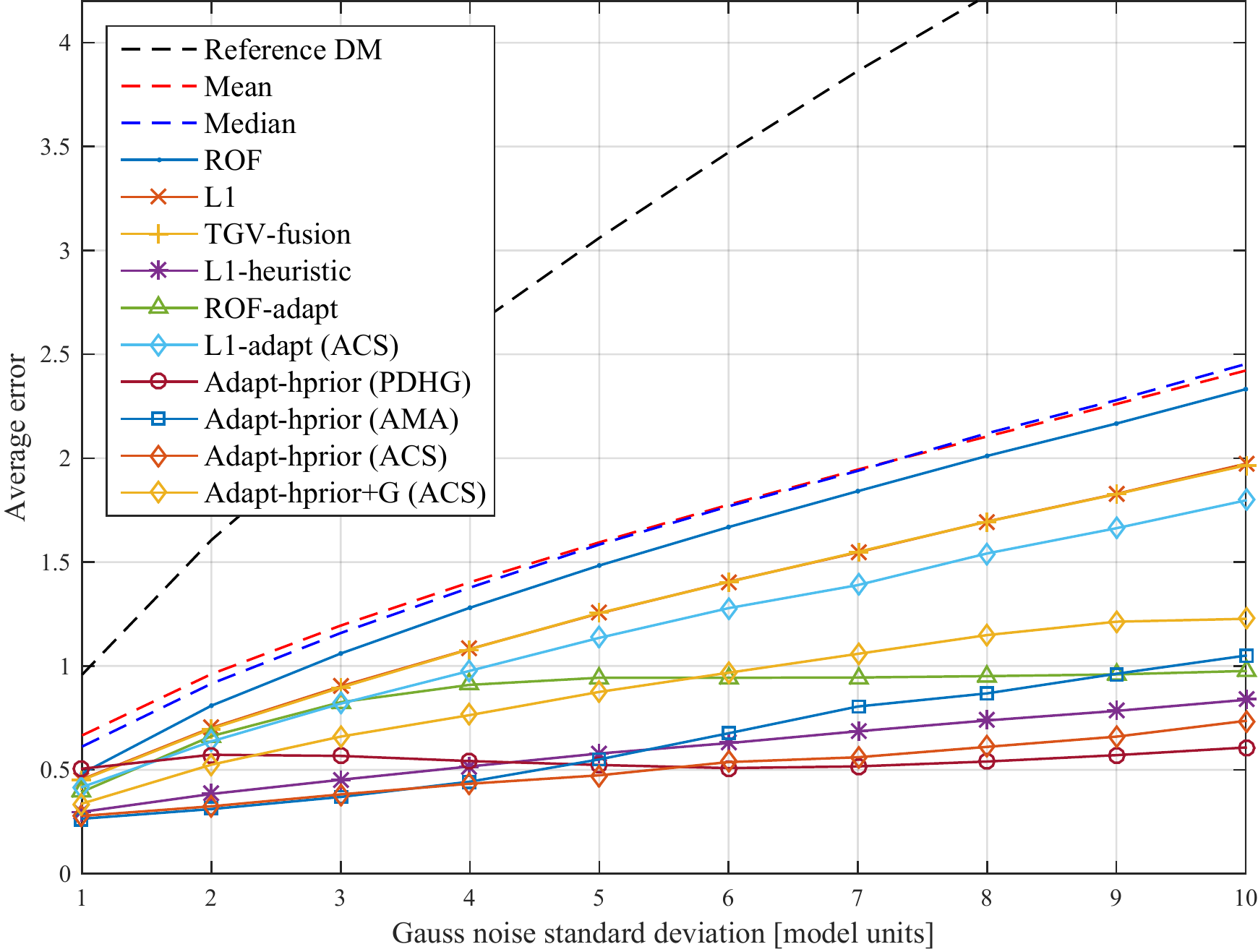}%
\hfil
\includegraphics[width=0.2\textwidth]{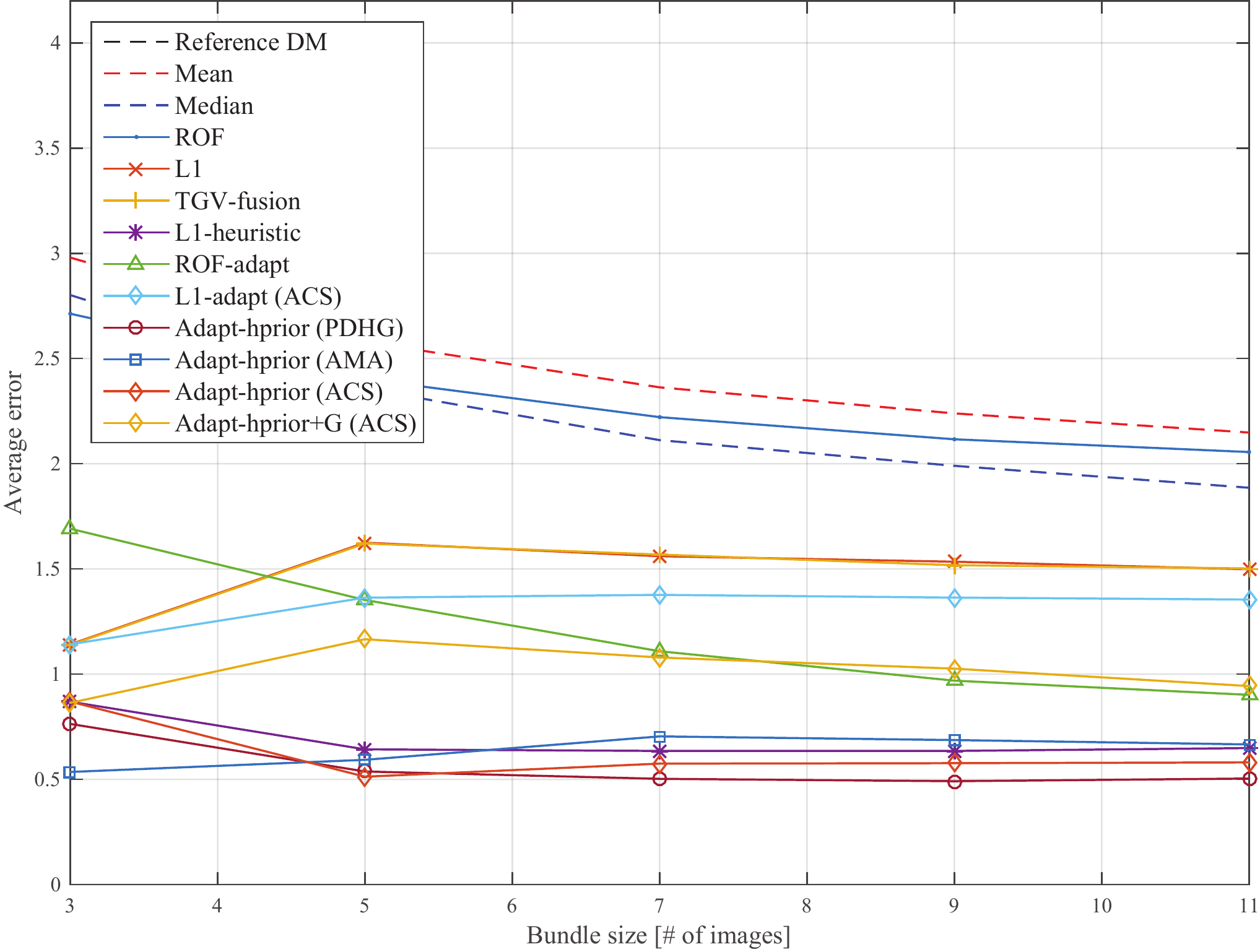}%
\hfil
\includegraphics[width=0.2\textwidth]{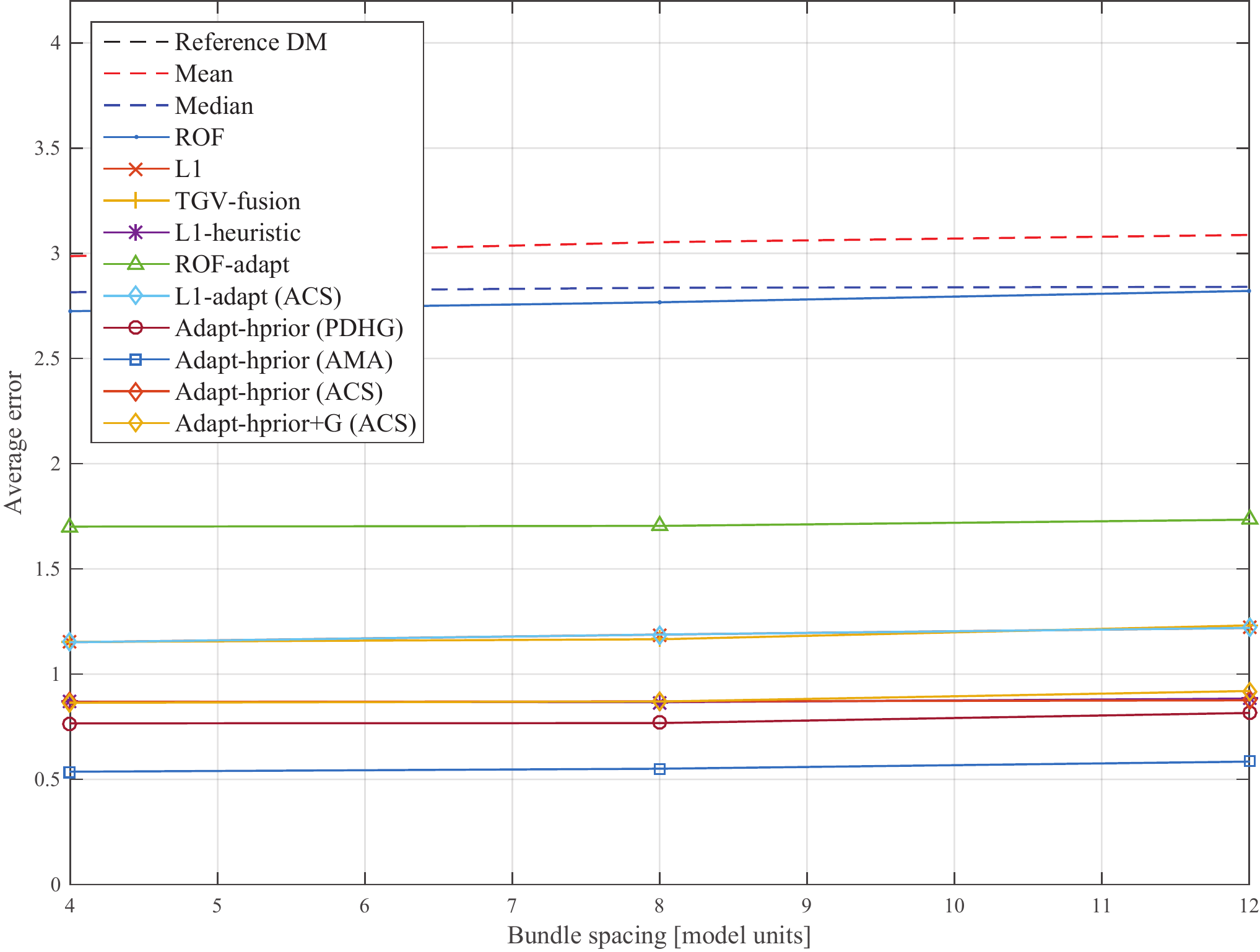}\\%
\includegraphics[width=0.2\textwidth]{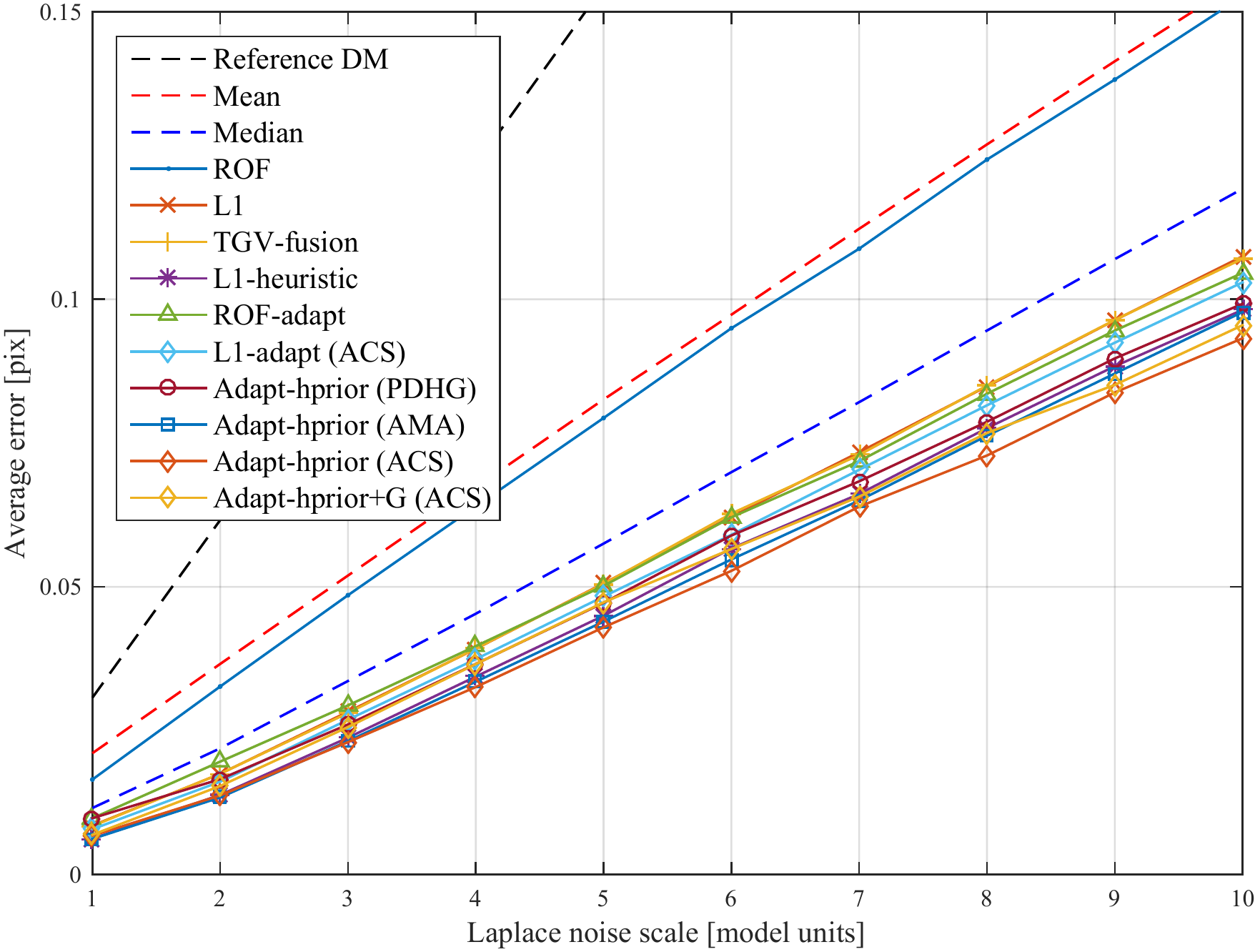}%
\hfil
\includegraphics[width=0.2\textwidth]{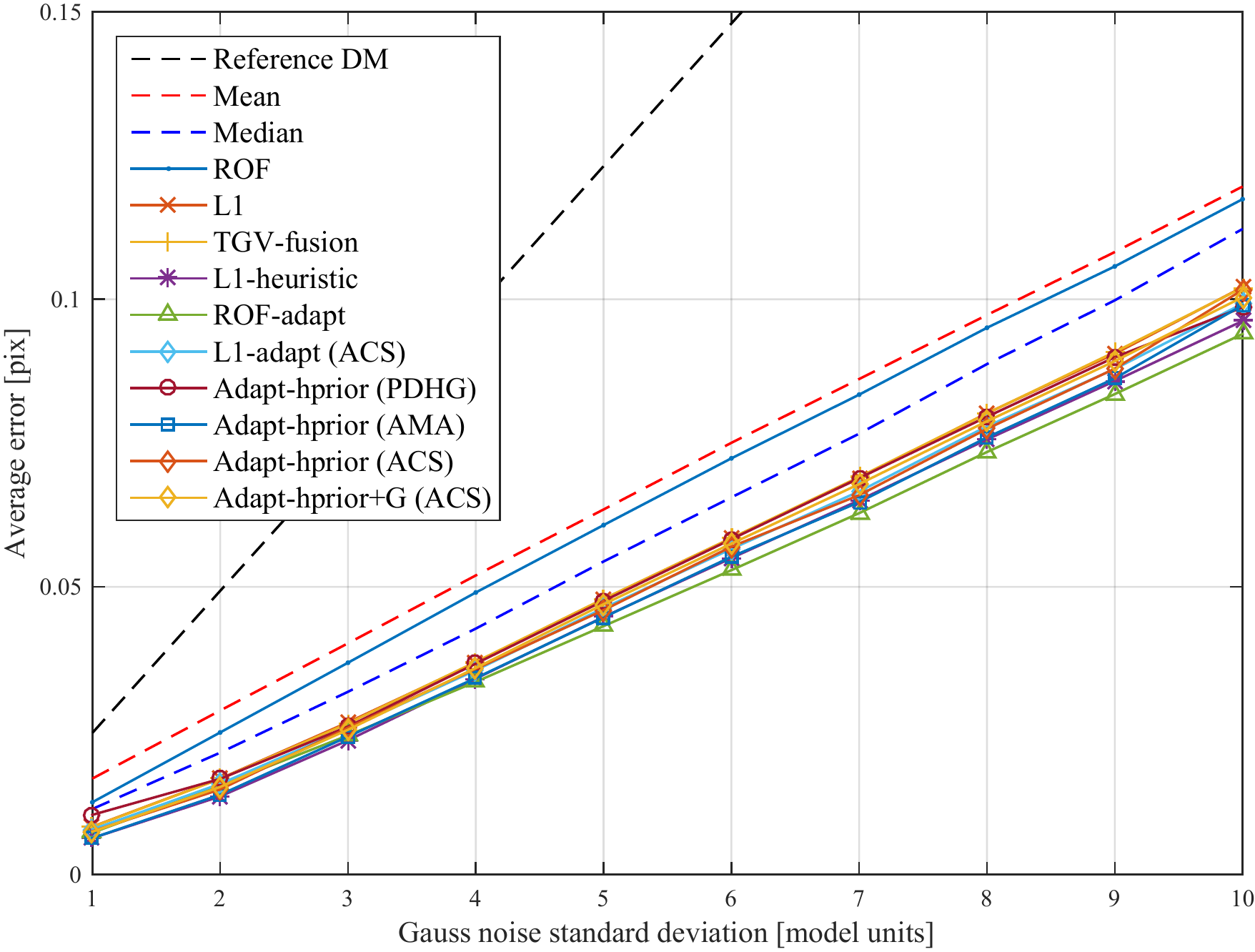}%
\hfil
\includegraphics[width=0.2\textwidth]{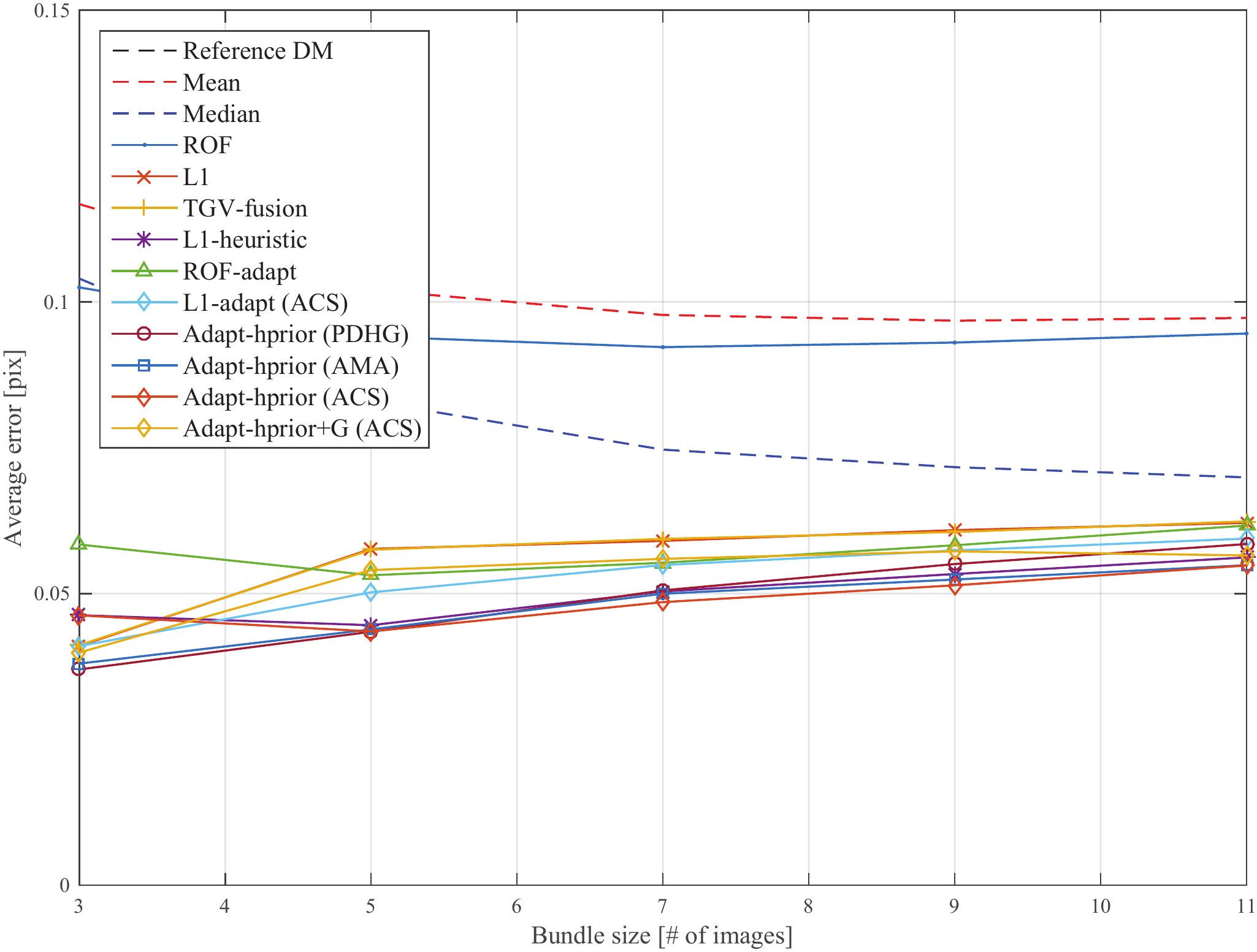}%
\hfil
\includegraphics[width=0.2\textwidth]{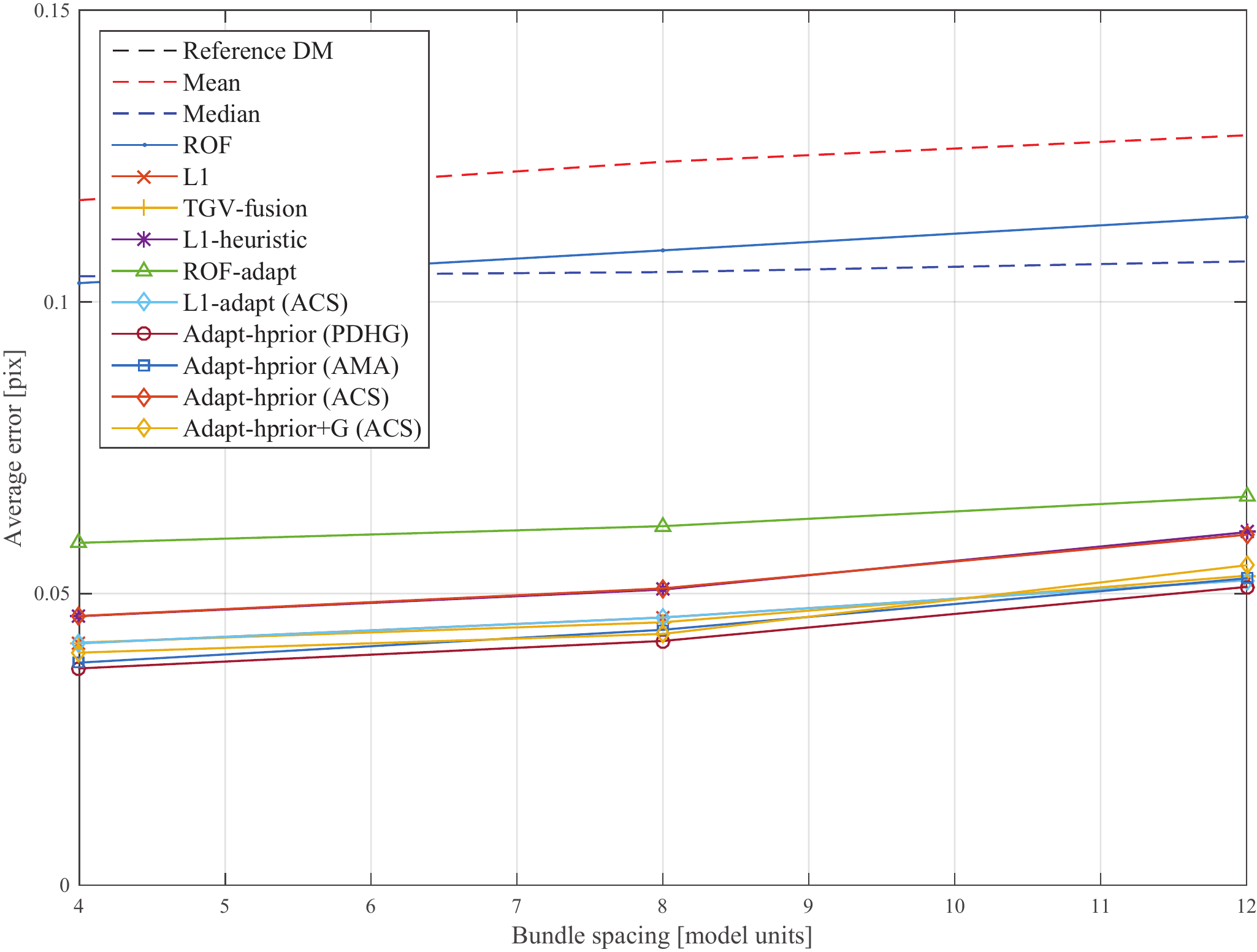}%
\caption{Average depth (top) and disparity (bottom) error in relation to (from left to right): a) Laplace noise scale; b) standard deviation of the Gauss noise; c) number of fused depth images; d) distance between the fused depth images. Laplace noise scale $b=6\,[m.u.]$.} \label{fig:urbgraphs}
\end{figure*}

\begin{table}[ht!]                                                 
\centering      
\caption{Results for the urban landscapes dataset for different versions and ablations of the proposed model for Laplace noise of scale $b=6\,[m.u.]$.  (Best values in bold)} \label{tab:resurb}
\begin{tabular}{c|cccc|cc}                                                                                                                                     
 & RMSE & ZMAE  & NMAE & Z-avg  & out-3[\%] & D-avg [px]\\                                   
\hline                                                                                   
Reference DM  & 8.4952 & 6.0155 & 1.4727 & 4.2221 & 0 & 0.1863 \\                                                                                         
Mean  & 3.4925 & 2.1655 & 1.3122 & 2.1490 & 0 & 0.0973 \\                                                                                           
Median  & 3.1017 & 1.7001 & 1.2664 & 1.8832 & 0 & 0.0699 \\               
\hline                                                                                   
ROF & 3.3454 & 2.0385 & 1.2821 & 2.0601 & 0 & 0.0950 \\                                                                                                  
L1 & 2.4370 & 1.2554 & 1.0936 & 1.4957 & 0 & 0.0622 \\                                                                                                   
TGV-fusion  \cite{Pock2011} & 2.4630 & 1.2654 & 1.0905 & 1.5035 & 0 & 0.0626 \\                                                                                          
L1-heuristic & \textbf{1.5670} & 0.5490 & 0.3168 & 0.6484 & 0 & 0.0565 \\         
\hline                                                                                   
ROF-adapt & 1.7434 & 0.7264 & 0.5809 & 0.9027 & 0 & 0.0619 \\                                                                                      
L1-adapt (PDHG) & 1.5768 & \textbf{0.4582} & \textbf{0.1647} & \textbf{0.4919} & 0 & 0.0562 \\                                                                                  
L1-adapt (AMA) & 2.3775 & 1.1960 & 1.0502 & 1.4401 & 0 & 0.0612 \\       
L1-adapt (ACS) & 1.7316 & 0.5385 & 0.2851 & 0.6430 & 0 & 0.0591 \\       
Adapt-hprior (PDHG)  & 1.5718 & 0.4874 & 0.1693 & 0.5062 & 0 & 0.0585 \\  
Adapt-hprior (AMA) & 1.7229 & 0.6022 & 0.3091 & 0.6845 & 0 & 0.0548 \\   
Adapt-hprior (ACS) & 1.7201 & 0.5694 & 0.2050 & 0.5855 & 0 & \textbf{0.0528} \\   
Adapt-hprior+G (PDHG)  & 1.9403 & 0.6254 & 0.3591 & 0.7582 & 0 & 0.0573 \\
Adapt-hprior+G (AMA)  & 2.3256 & 1.1123 & 0.9194 & 1.3348 & 0 & 0.0618 \\                                                                                    
Adapt-hprior+G (ACS) & 1.9872 & 0.7723 & 0.5470 & 0.9434 & 0 & 0.0565 \\ 
\end{tabular}
\end{table}

We observe also here that in the case of joint depth and confidence (biconvex) estimation problem, ACS gives better results with respect to AMA. 
In contrast to the previous dataset, we observe here that the PDHG versions of the adaptive methods always converged providing better results with respect to ACS and AMA methods. Nevertheless, ACS algorithm still gives results with similar errors. It is interesting to see that also for this dataset the {\em L1-heuristic} version give satisfactory results. 
Moreover, the {\em L1-adapt} version gives good results with respect to the methods which use prior confidence. This is important, especially considering that the heuristic priors explicitly use knowledge about the problem, and it highlights the power of the adaptive methods to infer suitable confidence values from the data. 

A visual comparison of the results is presented in Figure~\ref{fig:visresu3}. We see that the adaptive versions of the proposed model gives the best results. The results of this dataset better highlight the contribution of the automatically estimated confidence values. The difference with respect to the previous dataset, lies mostly in the ratio between the noise scale and the distance from the object.

Finally, the last two columns of Figure~\ref{fig:urbgraphs} show the effect of the bundle size and spacing on the fusion result for the urban scenes dataset. We see that the observations made for the objects dataset remain valid also here.

\subsubsection{Real data}
We evaluated the performance of our model for the depth fusion on real data using the KITTI dataset \cite{Geiger2012}. 
Here, ground truth of the disparity and calibration data of the cameras are provided, while ground truth localization data are not given. To estimate the camera motion, we considered two different stereo-camera localization methods in order to recover the relative transformations between the reference and the other views. The first is based on \cite{Geiger2011}, and the second is the one used in \cite{Ntouskos2013} for the localization of a head-mounted stereo-camera. 

The dataset contains stereo-pairs of images hence depth images from each of these stereo-pairs have to be computed. 
We have considered two methods for computing the depth images. The first is semi-global matching (SGM) algorithm \cite{Hirschmuller2005}, while the other is the ELAS method \cite{Geiger2010}. 
As our method assumes that the depth maps are given as-is, the quality of the result depends on the quality of the original depth images, hence the results presented here should be compared to the results of the stereo evaluation of the respective methods. 
The results regarding the non-occluded areas are presented in Table~\ref{tab:kittitrainresviso} and in Table~\ref{tab:kittitrainresgm} for the different choices of localization and disparity estimation algorithms evaluated for the training set of the KITTI stereo benchmark. The average values reported here are arithmetic averages in order to be consistent with the values reported on the website of the KITTI benchmark. One can notice that the proposed model performs better in all the combinations apart from the combination \cite{Ntouskos2013} \&\cite{Geiger2010}. This suggests that the proposed model is robust with respect to registration errors. The largest improvement in the out-3 metric with respect to the single view disparity estimation is equal to $4.25\,[\%]$ and it is observed  for the combination \cite{Geiger2011} \&\cite{Geiger2010}.

\begin{table}[th!]                                                 
\centering      
\caption{Results for KITTI stereo benchmark training set with localization according to \cite{Geiger2011}.}                         
\label{tab:kittitrainresviso}       
\begin{tabular}{c|ccc|ccc}     
&\multicolumn{3}{c}{SGM\cite{Hirschmuller2005}}&\multicolumn{3}{c}{VISO\cite{Geiger2010}}\\                                                       
 & density [\%] & out-3 [\%] & D-avg [px] & density [\%]& out-3 [\%]& D-avg [px] \\                       
\hline                                           
Reference & 84.6221 & 12.6218 & 3.0169 & 93.4506 & 11.5387 & 2.0531 \\                                             
Mean & 98.7363 & 12.7838 & 2.5826 & 99.6008 & 13.6108 & 2.0172\\                                                      
Median & 98.7361 & 9.0966 & 2.1139 & 99.6008 & 7.6852 & 1.4663 \\            
\hline                                           
TGV-fusion & 100 & 8.6929 & 2.0184 & 100 & 7.4690 & 1.4149 \\                                                  
L1-heuristic & 100 & 8.6780 & 1.9994 & 100 & 7.3058 & 1.3741 \\      
\hline                                           
Adapt-hprior (ACS) & 100 & \textbf{8.6466} &\textbf{1.9941} & 100 & \textbf{7.2947} & \textbf{1.3696}\\                                         
\end{tabular}                                               
\end{table}   

\begin{table}[th!]                                                 
\centering      
\caption{Results for KITTI stereo benchmark training set with localization according to \cite{Ntouskos2013}.}                         
\label{tab:kittitrainresgm}  
\begin{tabular}{c|ccc|ccc}     
&\multicolumn{3}{c}{SGM\cite{Hirschmuller2005}}&\multicolumn{3}{c}{VISO\cite{Geiger2010}}\\  
 & density [\%] & out-3 [\%] & D-avg [px] & density [\%]& out-3 [\%]& D-avg [px] \\                       
\hline                                  
Reference & 84.6186 & 12.6285 & 3.0334 & 93.4459 & 11.5412 & 2.0516 \\                                            
Mean & 98.7740 & 13.1046 & 2.6234 & 99.6015 & 13.9590 & 2.0603\\                                                      
Median & 98.7739 & 9.4853 & 2.1547 & 99.6015 & 7.9815 & 1.5191\\            
\hline                                           
TGV-fusion & 100 & 9.0621 & 2.0516 & 100 & 7.7581 & 1.4665 \\                                                  
L1-heuristic & 100 & 9.0451 & 2.0333 & 100 & \textbf{7.6057} & \textbf{1.4268} \\      
\hline                                           
Adapt-hprior (ACS) & 100 & \textbf{9.0162} & \textbf{2.0281} & 100 & 7.8962 & 1.4834 \\
\end{tabular}                                          
\end{table}

\begin{figure}[th!]%
\centering
\includegraphics[width=0.45\columnwidth]{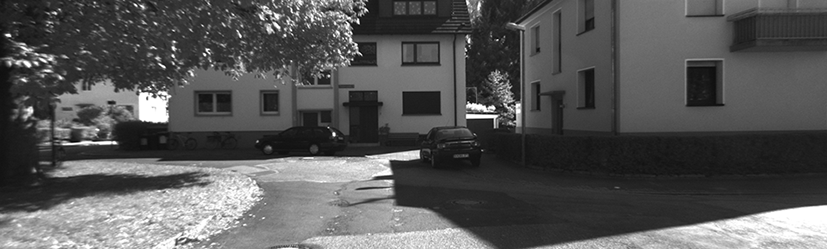}%
\hfil
\includegraphics[width=0.45\columnwidth]{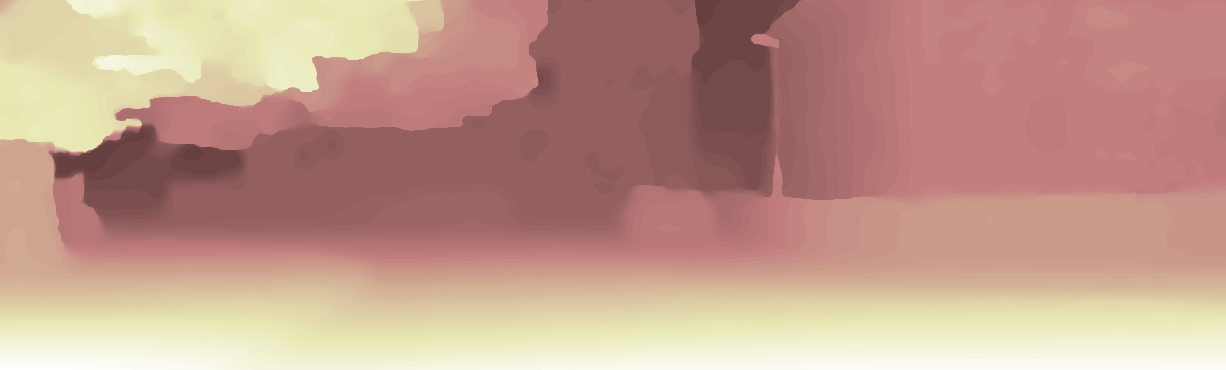}\\[3pt]%
\includegraphics[width=0.45\columnwidth]{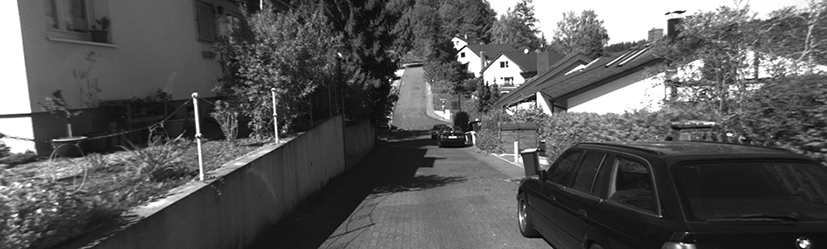}%
\hfil	
\includegraphics[width=0.45\columnwidth]{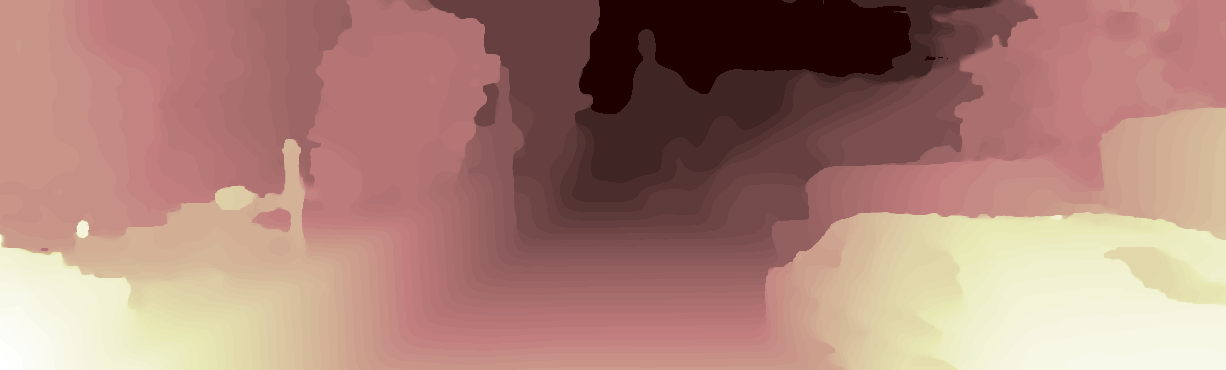}\\[3pt]%
\includegraphics[width=0.45\columnwidth]{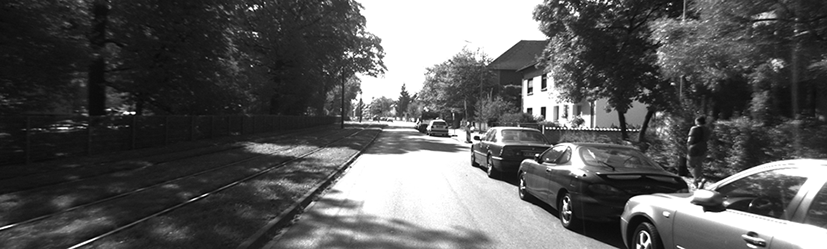}%
\hfil
\includegraphics[width=0.45\columnwidth]{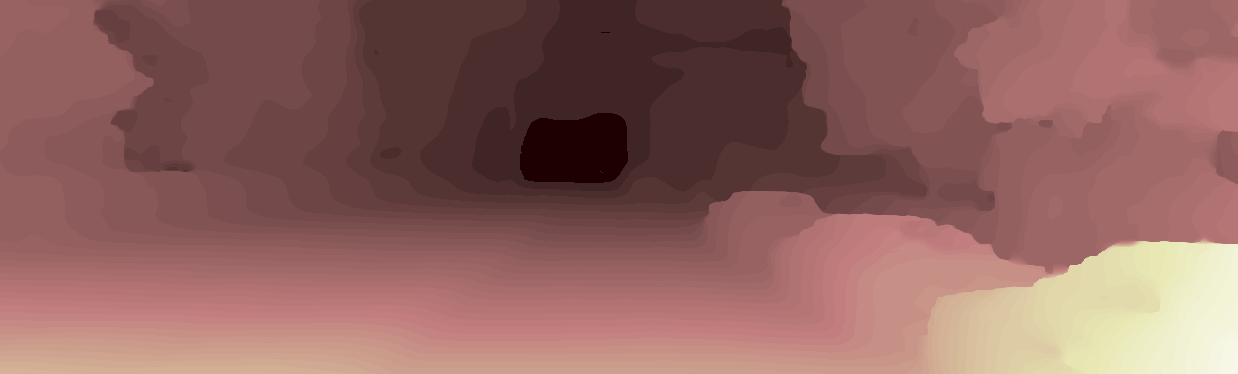}\\[3pt]%
\caption{Fused depth images for the KITTI dataset.} \label{fig:kittiresim}
\end{figure}

We used the {\em Adapt-hprior} versions, the best performing version of our method, to compute the disparity images of the testing set of the KITTI stereo benchmark, using the combination  \cite{Geiger2011} \&\cite{Geiger2010} for localization and single view disparity estimation respectively. The results obtained are presented in Table~\ref{tab:kittitestres}. We can see that the results improve by $1.78\%$ with respect to the single-pair disparity estimation algorithm in the out-noc-3 metric, and by $3.07\%$ with respect to out-all-3. 

\begin{table}[ht!]                                                 
\centering    
\caption{Results for KITTI stereo benchmark testing set with localization according to \cite{Geiger2011} and comparison to the single view results of \cite{Geiger2010}.}                         
\label{tab:kittitestres}    
\begin{tabular}{c|ccccc}      
 & density [\%] & out-noc-3 [\%] & out-all-3 [\%] & avg-noc [px] & avg-all [px]   \\                       
\hline                                  
Reference\cite{Geiger2010} & 94.55 & 8.24 & 9.96 & 1.4 & 1.6 \\                                                                                    
\hline
Adapt-hprior (ACS) & \textbf{99.70} & \textbf{6.46} & \textbf{6.89} & \textbf{1.2} & \textbf{1.3} \\                
\end{tabular}                                                    
\end{table}

Finally, we repeated the evaluation by computing individual disparity maps using \cite{Zbontar2015}, which corresponds to the current state of the art. The results are presented in Table \ref{tab:kittitestres2}, while the complete results are available under the short name {\em cfusion} on the KITTI benchmark's website. We note that the proposed fusion model is able to further increase the accuracy of the disparity maps. Considering also the occluded regions of the reference image, our model achieves better results with respect to all competing methods on the dataset, by the time of submission of this manuscript. The improvement on the reflective regions of the images is even more significant, where the accuracy improves by $3.14\%$ in the out-noc-3 metric, and by $5.76\%$ in the  out-all-3 metric, with respect to \cite{Zbontar2015}. Examples of fused depth images for this evaluation are presented in Figure~\ref{fig:kittiresim}.

\begin{table}[ht!]                                                 
\centering    
\caption{Results for KITTI stereo benchmark testing set with localization according to \cite{Geiger2011} and comparison to the single view results of \cite{Zbontar2015}.}                         
\label{tab:kittitestres2}    
\begin{tabular}{c|ccccc}      
 & density [\%] & out-noc-3 [\%] & out-all-3 [\%] & avg-noc [px] & avg-all [px]   \\                       
\hline                                  
Reference\cite{Zbontar2015} & \textbf{100} & 2.61 & 3.84 & \textbf{0.8} & 1.0 \\                                                                                    
Adapt-hprior (ACS) & 99.93 & \textbf{2.46} & \textbf{2.69} & \textbf{0.8} & \textbf{0.8} \\                
\hline
Reference - Reflective\cite{Zbontar2015} & - & 18.45 & 21.96 & 3.5 & 4.3 \\                                                                                    
Adapt-hprior (ACS) - Reflective &  - & \textbf{15.31} & \textbf{16.20} & \textbf{2.6} & \textbf{2.8}
\end{tabular}                                                      
\end{table}

\begin{figure}[t!]%
\centering
\begin{tabular}{ccccc}
\rot{ Ground Truth}&
\includegraphics[width=0.15\columnwidth]{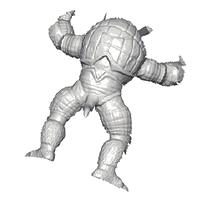}&
\includegraphics[width=0.15\columnwidth]{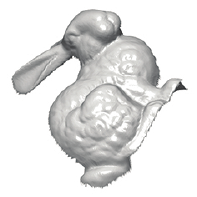}&
\includegraphics[width=0.15\columnwidth]{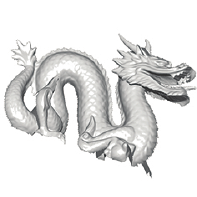}&
\includegraphics[width=0.15\columnwidth]{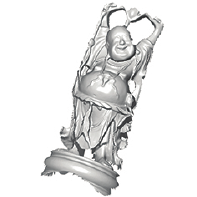}\\
%\rot{Ground Truth}&
%\includegraphics[width=0.15\columnwidth]{obj_2_ref}&
%\includegraphics[width=0.15\columnwidth]{obj_4_ref}&
%\includegraphics[width=0.15\columnwidth]{urb_1_ref}&
%\includegraphics[width=0.15\columnwidth]{urb_4_ref}\\
\rot{\hspace{1em} Median}&
\includegraphics[width=0.15\columnwidth]{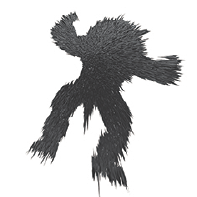}&
\includegraphics[width=0.15\columnwidth]{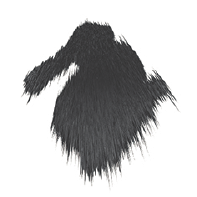}&
\includegraphics[width=0.15\columnwidth]{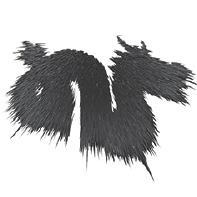}&
\includegraphics[width=0.15\columnwidth]{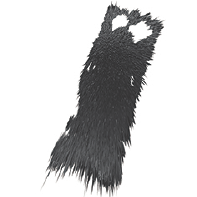}\\ %\hline
\rot{ ROF\cite{Rudin1992}}&
\includegraphics[width=0.15\columnwidth]{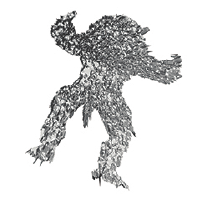}&
\includegraphics[width=0.15\columnwidth]{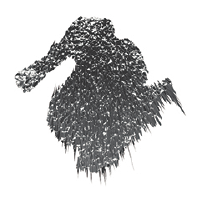}&
\includegraphics[width=0.15\columnwidth]{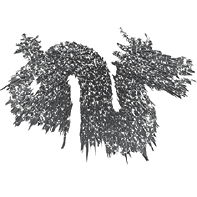}&
\includegraphics[width=0.15\columnwidth]{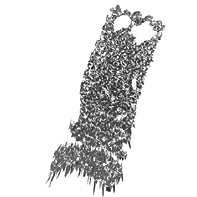}\\
\rot{\hspace{2em} L1}&
\includegraphics[width=0.15\columnwidth]{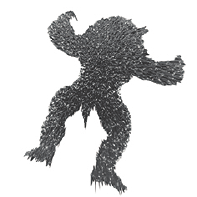}&
\includegraphics[width=0.15\columnwidth]{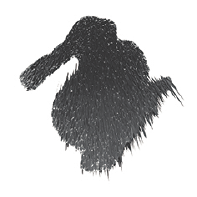}&
\includegraphics[width=0.15\columnwidth]{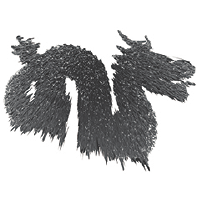}&
\includegraphics[width=0.15\columnwidth]{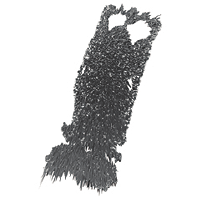}\\
\rot{ TGV-fusion\cite{Pock2011}}&
\includegraphics[width=0.15\columnwidth]{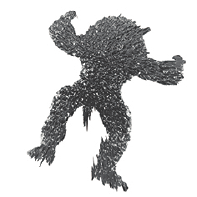}&
\includegraphics[width=0.15\columnwidth]{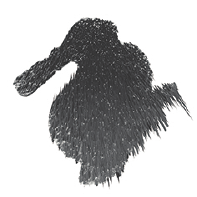}&
\includegraphics[width=0.15\columnwidth]{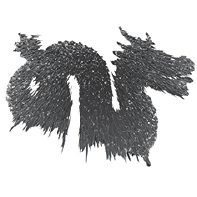}&
\includegraphics[width=0.15\columnwidth]{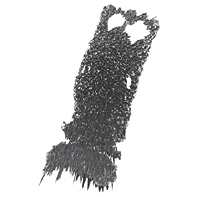}\\ %\hline
\rot{ L1-heuristic}&
\includegraphics[width=0.15\columnwidth]{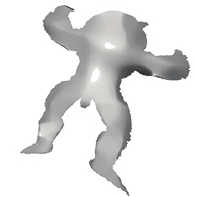}&
\includegraphics[width=0.15\columnwidth]{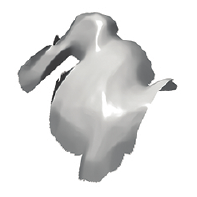}&
\includegraphics[width=0.15\columnwidth]{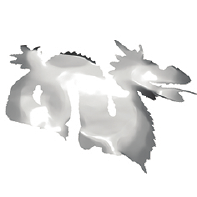}&
\includegraphics[width=0.15\columnwidth]{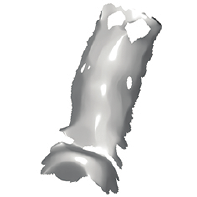}\\
\rot{ L1-adapt}&
\includegraphics[width=0.15\columnwidth]{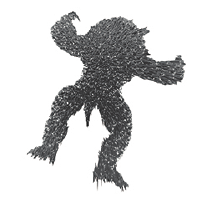}&
\includegraphics[width=0.15\columnwidth]{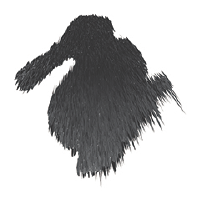}&
\includegraphics[width=0.15\columnwidth]{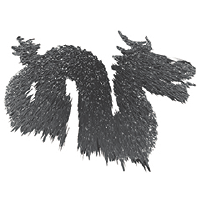}&
\includegraphics[width=0.15\columnwidth]{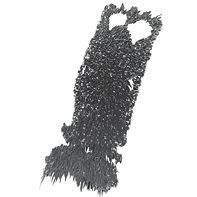}\\
\rot{ Adapt-hprior}&
\includegraphics[width=0.15\columnwidth]{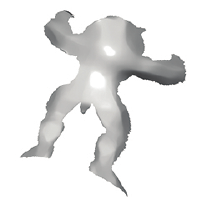}&
\includegraphics[width=0.15\columnwidth]{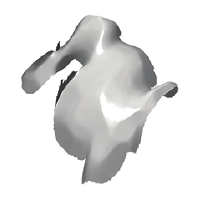}&
\includegraphics[width=0.15\columnwidth]{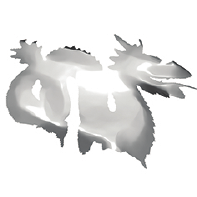}&
\includegraphics[width=0.15\columnwidth]{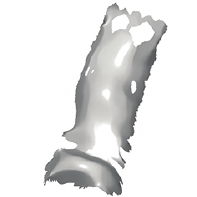}
\end{tabular}
\caption{Surfaces obtained by different methods for the Stanford 3D scanning dataset\cite{Turk94,Curless96,Krishnamurthy96}. %Top to bottom: i) GT; ii) Median; iii) ROF; iv) L1; v) TGV-fusion; vi) L1-heuristic; vii) L1-adapt (ACS); viii) Adapt-hprior (ACS).
}\label{fig:visresu1}
\end{figure}

\begin{figure}[t!]%
\centering
\begin{tabular}{ccccc}
\rot{ Ground Truth}&
\includegraphics[width=0.15\columnwidth]{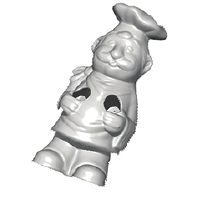}&
\includegraphics[width=0.15\columnwidth]{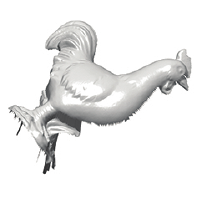}&
\includegraphics[width=0.15\columnwidth]{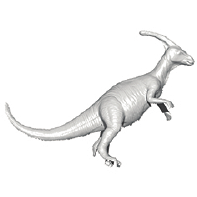}&
\includegraphics[width=0.15\columnwidth]{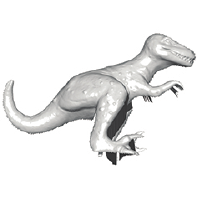}\\
%\rot{Ground Truth}&
%\includegraphics[width=0.15\columnwidth]{obj_2_ref}&
%\includegraphics[width=0.15\columnwidth]{obj_4_ref}&
%\includegraphics[width=0.15\columnwidth]{urb_1_ref}&
%\includegraphics[width=0.15\columnwidth]{urb_4_ref}\\
\rot{\hspace{1em} Median}&
\includegraphics[width=0.15\columnwidth]{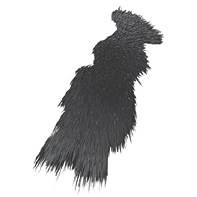}&
\includegraphics[width=0.15\columnwidth]{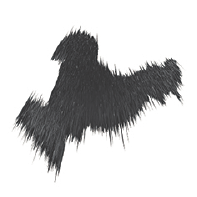}&
\includegraphics[width=0.15\columnwidth]{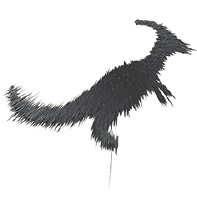}&
\includegraphics[width=0.15\columnwidth]{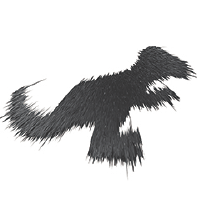}\\ %\hline
\rot{ ROF\cite{Rudin1992}}&
\includegraphics[width=0.15\columnwidth]{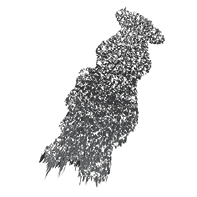}&
\includegraphics[width=0.15\columnwidth]{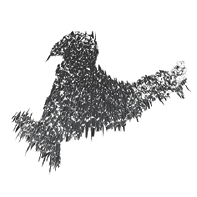}&
\includegraphics[width=0.15\columnwidth]{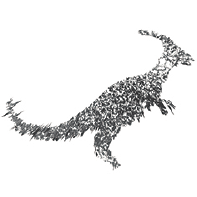}&
\includegraphics[width=0.15\columnwidth]{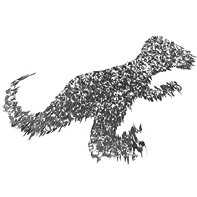}\\
\rot{\hspace{2em} L1}&
\includegraphics[width=0.15\columnwidth]{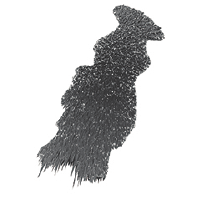}&
\includegraphics[width=0.15\columnwidth]{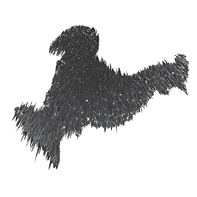}&
\includegraphics[width=0.15\columnwidth]{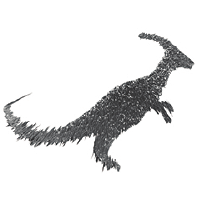}&
\includegraphics[width=0.15\columnwidth]{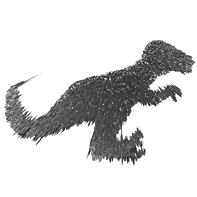}\\
\rot{ TGV-fusion\cite{Pock2011}}&
\includegraphics[width=0.15\columnwidth]{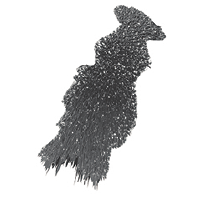}&
\includegraphics[width=0.15\columnwidth]{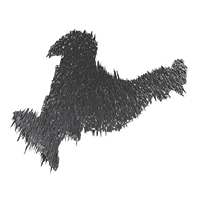}&
\includegraphics[width=0.15\columnwidth]{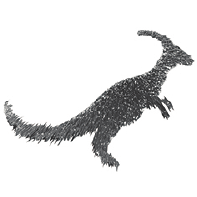}&
\includegraphics[width=0.15\columnwidth]{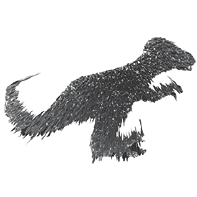}\\ %\hline
\rot{ L1-heuristic}&
\includegraphics[width=0.15\columnwidth]{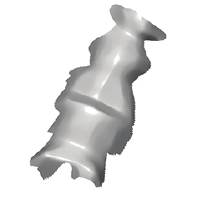}&
\includegraphics[width=0.15\columnwidth]{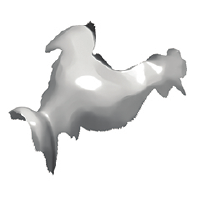}&
\includegraphics[width=0.15\columnwidth]{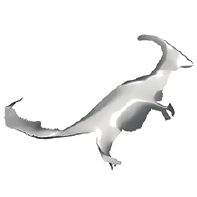}&
\includegraphics[width=0.15\columnwidth]{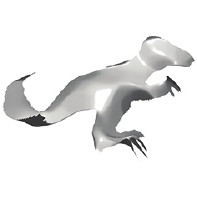}\\
\rot{ L1-adapt}&
\includegraphics[width=0.15\columnwidth]{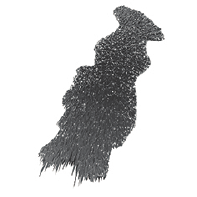}&
\includegraphics[width=0.15\columnwidth]{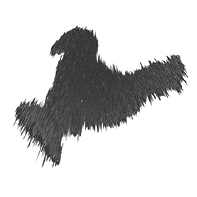}&
\includegraphics[width=0.15\columnwidth]{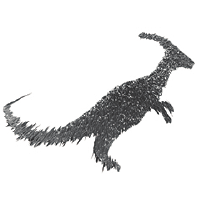}&
\includegraphics[width=0.15\columnwidth]{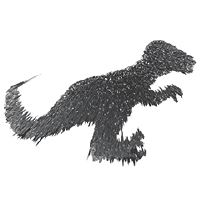}\\
\rot{ Adapt-hprior}&
\includegraphics[width=0.15\columnwidth]{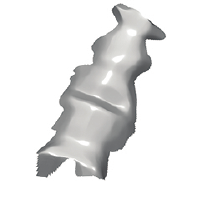}&
\includegraphics[width=0.15\columnwidth]{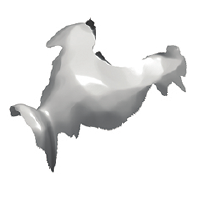}&
\includegraphics[width=0.15\columnwidth]{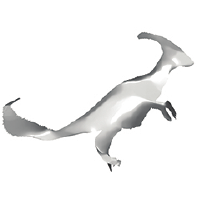}&
\includegraphics[width=0.15\columnwidth]{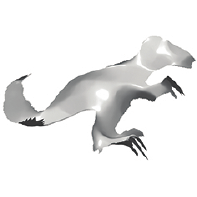}
\end{tabular}
\caption{Surfaces obtained by different methods for the dataset of \cite{Mian2006}. %Top to bottom: i) GT; ii) Median; iii) ROF; iv) L1; v) TGV-fusion; vi) L1-heuristic; vii) L1-adapt (ACS); viii) Adapt-hprior (ACS).
}\label{fig:visresu2}
\end{figure}

\begin{figure}[t!]%
\centering
\begin{tabular}{ccccc}
\rot{ Ground Truth}&
\includegraphics[width=.15\columnwidth]{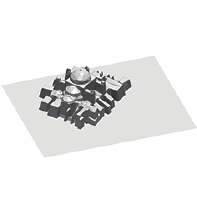}&
\includegraphics[width=0.15\columnwidth]{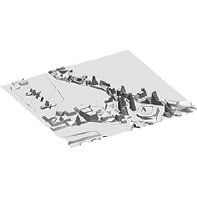}&
\includegraphics[width=0.15\columnwidth]{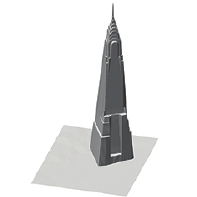}&
\includegraphics[width=0.15\columnwidth]{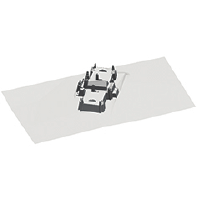}\\
%\rot{Ground Truth}&
%\includegraphics[width=0.15\columnwidth]{obj_2_ref}&
%\includegraphics[width=0.15\columnwidth]{obj_4_ref}&
%\includegraphics[width=0.15\columnwidth]{urb_1_ref}&
%\includegraphics[width=0.15\columnwidth]{urb_4_ref}\\
\rot{\hspace{1em} Median}&
\includegraphics[width=0.15\columnwidth]{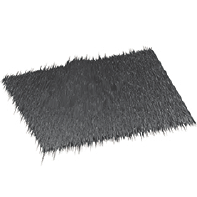}&
\includegraphics[width=0.15\columnwidth]{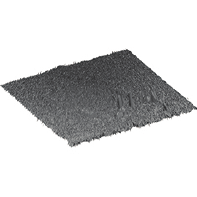}&
\includegraphics[width=0.15\columnwidth]{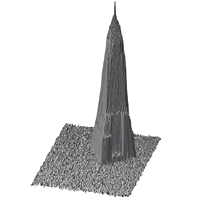}&
\includegraphics[width=0.15\columnwidth]{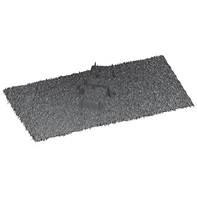}\\ %\hline
\rot{ ROF\cite{Rudin1992}}&
\includegraphics[width=0.15\columnwidth]{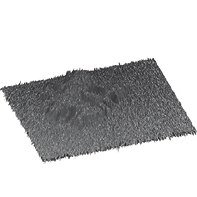}&
\includegraphics[width=0.15\columnwidth]{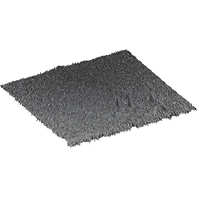}&
\includegraphics[width=0.15\columnwidth]{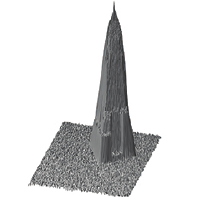}&
\includegraphics[width=0.15\columnwidth]{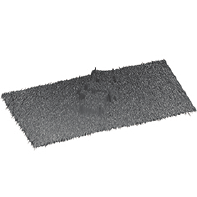}\\
\rot{\hspace{2em} L1}&
\includegraphics[width=0.15\columnwidth]{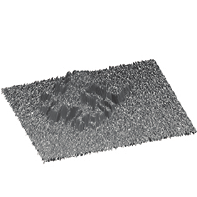}&
\includegraphics[width=0.15\columnwidth]{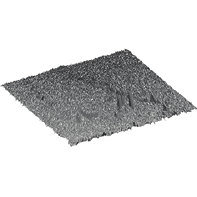}&
\includegraphics[width=0.15\columnwidth]{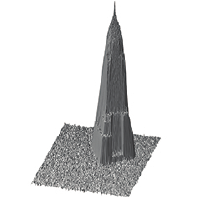}&
\includegraphics[width=0.15\columnwidth]{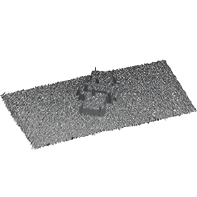}\\
\rot{ TGV-fusion\cite{Pock2011}}&
\includegraphics[width=0.15\columnwidth]{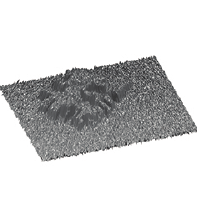}&
\includegraphics[width=0.15\columnwidth]{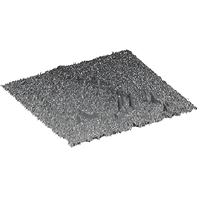}&
\includegraphics[width=0.15\columnwidth]{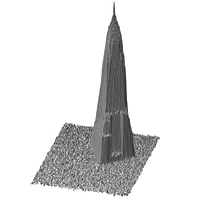}&
\includegraphics[width=0.15\columnwidth]{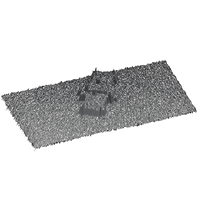}\\ %\hline
\rot{ L1-heuristic}&
\includegraphics[width=0.15\columnwidth]{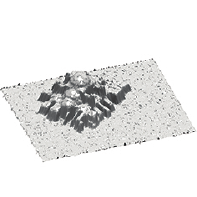}&
\includegraphics[width=0.15\columnwidth]{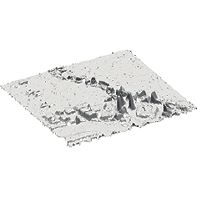}&
\includegraphics[width=0.15\columnwidth]{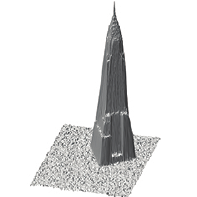}&
\includegraphics[width=0.15\columnwidth]{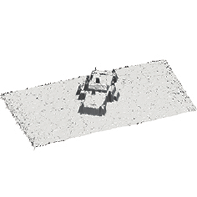}\\
\rot{ L1-adapt}&
\includegraphics[width=0.15\columnwidth]{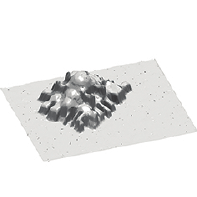}&
\includegraphics[width=0.15\columnwidth]{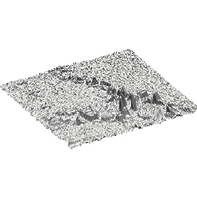}&
\includegraphics[width=0.15\columnwidth]{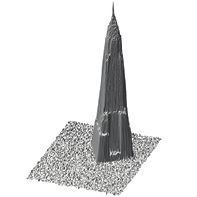}&
\includegraphics[width=0.15\columnwidth]{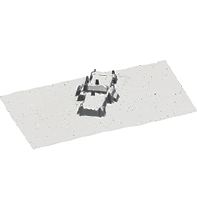}\\
\rot{ Adapt-hprior}&
\includegraphics[width=0.15\columnwidth]{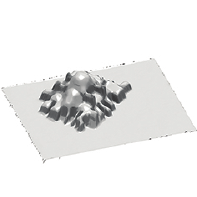}&
\includegraphics[width=0.15\columnwidth]{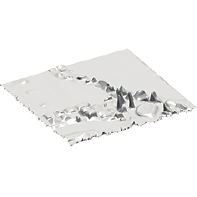}&
\includegraphics[width=0.15\columnwidth]{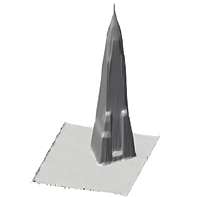}&
\includegraphics[width=0.15\columnwidth]{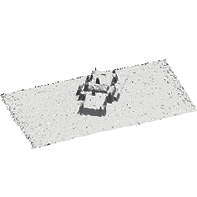}
\end{tabular}
\caption{Surfaces obtained by different methods for the Urban Landscapes dataset. %Top to bottom: i) GT; ii) Median; iii) ROF; iv) L1; v) TGV-fusion; vi) L1-heuristic; vii) L1-adapt (ACS); viii) Adapt-hprior (ACS).
}\label{fig:visresu3}
\end{figure}

\section{Conclusions} \label{sec:conc}
We introduce a novel model for data fusion with spatially varying confidence values. The proposed model directly estimates  the confidence values from the given data. We have proved the main properties of this model and also discussed methods to estimate  optimal solution. Indeed, an optimal solution for this family of models can be estimated by solving a biconvex non-smooth optimization problem. We presented two algorithms for solving the biconvex optimization problem, corresponding to the ACS, AMA, and PDHG  classes of algorithms, discussing their convergence to critical points. We also discuss possible ablations of the proposed model, and focus on the possibility to assign {\em a-priori} confidence values. 

We demonstrated numerically the behavior of the proposed  model for synthetic images and we evaluated its performance considering  the fusion of depth images as application. The results show that  model outperforms the considered baselines and state of the art algorithms for this problem. We also examined the performance of various ablations of the full model. Moreover, we have seen that for the case of depth image fusion, spatially varying confidence values estimated from the geometry of the scene can provide satisfactory results. 

As future work on the theoretical front we shall examine the PDHG algorithm for biconvex problems and its convergence. On the application side we shall examine closer TV regularization on manifolds for 3D modeling as in \cite{Ntouskos-2015ICCV} and \cite{Natola-2016CVPR} and study the consistency and coherence of surfaces generated from images.

% or


\appendices

%\section{}
%Appendix two text goes here.Appendix two text goes here.Appendix two text goes here.Appendix two text goes here.Appendix two text goes here.Appendix two text goes here.Appendix two text goes here.Appendix two text goes here.Appendix two text goes here.Appendix two text goes here.Appendix two text goes here.Appendix two text goes here.Appendix two text goes here.Appendix two text goes here.Appendix two text goes here.Appendix two text goes here.Appendix two text goes here.Appendix two text goes here.Appendix two text goes here.Appendix two text goes here.Appendix two text goes here.Appendix two text goes here.Appendix two text goes here.Appendix two text goes here.Appendix two text goes here.

\ifCLASSOPTIONcompsoc
  \section*{Acknowledgments}
\else
  \section*{Acknowledgment}
\fi

This work is supported by the EU FP7 TRADR (609763) and the EU H2020 SecondHands (643950) project.
We thank the authors of the 3D models used for freely providing them on the 3D Warehouse repository.
\

\ifCLASSOPTIONcaptionsoff
  \newpage
\fi



%\bibliographystyle{IEEEtran}
\bibliography{main_fusion}
\bibliographystyle{IEEEtranS}

\end{document}